\documentclass[letterpaper]{article} % DO NOT CHANGE THIS
\usepackage[submission]{aaai24}  % DO NOT CHANGE THIS
\usepackage{times}  % DO NOT CHANGE THIS
\usepackage{helvet}  % DO NOT CHANGE THIS
\usepackage{courier}  % DO NOT CHANGE THIS
\usepackage[hyphens]{url}  % DO NOT CHANGE THIS
\usepackage{graphicx} % DO NOT CHANGE THIS
\usepackage{natbib}  % DO NOT CHANGE THIS AND DO NOT ADD ANY OPTIONS TO IT
\usepackage{caption} % DO NOT CHANGE THIS AND DO NOT ADD ANY OPTIONS TO IT
\ifdefined\preamble\else\def\preamble{}
\usepackage[on]{editing}

\usepackage[utf8]{inputenc} % allow utf-8 input
\usepackage[T1]{fontenc}    % use 8-bit T1 fonts
\usepackage{url}            % simple URL typesetting
\usepackage{booktabs}       % professional-quality tables
\usepackage{amsfonts}       % blackboard math symbols
\usepackage{nicefrac}       % compact symbols for 1/2, etc.
\usepackage{microtype}      % microtypography
\usepackage{lipsum}
\usepackage{enumitem}
\usepackage{optidef}
\usepackage{subcaption}
\usepackage{bm}
\usepackage{upgreek}
\usepackage{multicol}
\usepackage{multirow}

\usepackage{tikz}
\usetikzlibrary{shapes,decorations,arrows,calc,arrows.meta,fit,positioning}
\tikzset{
    -Latex,auto,node distance =1 cm and 1 cm,semithick,
    state/.style ={ellipse, draw, minimum width = 0.7 cm},
    point/.style = {circle, draw, inner sep=0.04cm,fill,node contents={}},
    bidirected/.style={Latex-Latex,dashed},
    el/.style = {inner sep=2pt, align=left, sloped}
}

\usepackage{amsmath,amsthm,amsfonts,amssymb}
\usepackage[linesnumbered, ruled, vlined]{algorithm2e}

\usepackage{pifont}% http://ctan.org/pkg/pifont
\definecolor{myRed}{rgb}{0.75,0,0}
\definecolor{myGreen}{rgb}{0,0.75,0}
\usepackage{thmtools,thm-restate}

\declaretheorem{lemma}
\declaretheorem{fact}

\declaretheorem{corollary}
\declaretheorem{proposition}

\newtheorem{innercustomprop}{Proposition}
\newenvironment{customprop}[1]
  {\renewcommand\theinnercustomprop{#1}\innercustomprop}
  {\endinnercustomprop}

\theoremstyle{definition}
\declaretheorem{definition}
\declaretheorem{example}

\newcommand{\Parents}{Pa}

\DeclareSymbolFont{symbolsC}{U}{txsyc}{m}{n}
\DeclareMathSymbol{\boxright}{\mathrel}{symbolsC}{128}

\newcommand{\norm}[1]{\left\lVert#1\right\rVert}

\def\*#1{\mathbf{#1}}

\newcommand{\tildepai}[1]{\widetilde{\*{pa}}_{#1}}

\newcommand{\pai}[1]{\*{pa}_{#1}}
\newcommand{\Pai}[1]{\*{Pa}_{#1}}

\newcommand{\ui}[1]{\*{u}_{#1}}
\newcommand{\Ui}[1]{\*{U}_{#1}}

\theoremstyle{definition}

\DeclareMathOperator{\rst}{Rst}

\DeclareMathOperator{\unif}{Unif}
\DeclareMathOperator{\bern}{Bernoulli}
\DeclareMathOperator{\dirich}{Dirichlet}

\newcommand{\xdasharrow}[2][->]{
\tikz[baseline=-\the\dimexpr\fontdimen22\textfont2\relax]{
\node[anchor=south,font=\scriptsize, inner ysep=1.5pt,outer xsep=8pt](x){#2};
\draw[shorten <=3.4pt,shorten >=3.4pt,dashed,#1](x.south west)--(x.south east);
}
}

\def\ddefloop#1{\ifx\ddefloop#1\else\ddef{#1}\expandafter\ddefloop\fi}
\def\ddef#1{\expandafter\def\csname bb#1\endcsname{\ensuremath{\mathbb{#1}}}}
\ddefloop ABCDEFGHIJKLMNOPQRSTUVWXYZ\ddefloop
\def\ddef#1{\expandafter\def\csname c#1\endcsname{\ensuremath{\mathcal{#1}}}}
\ddefloop ABCDEFGHIJKLMNOPQRSTUVWXYZ\ddefloop
\def\ddef#1{\expandafter\def\csname h#1\endcsname{\ensuremath{\widehat{#1}}}}
\ddefloop ABCDEFGHIJKLMNOPQRSTUVWXYZ\ddefloop
\def\ddef#1{\expandafter\def\csname v#1\endcsname{\ensuremath{\boldsymbol{#1}}}}
\ddefloop ABCDEFGHIJKLMNOPQRSTUVWXYZabcdefghijklmnopqrstuvwxyz\ddefloop
\def\ddef#1{\expandafter\def\csname v#1\endcsname{\ensuremath{\boldsymbol{\csname #1\endcsname}}}}
\ddefloop {alpha}{beta}{gamma}{delta}{epsilon}{varepsilon}{zeta}{eta}{theta}{vartheta}{iota}{kappa}{lambda}{mu}{nu}{xi}{pi}{varpi}{rho}{varrho}{sigma}{varsigma}{tau}{upsilon}{phi}{varphi}{chi}{psi}{omega}{Gamma}{Delta}{Theta}{Lambda}{Xi}{Pi}{Sigma}{varSigma}{Upsilon}{Phi}{Psi}{Omega}{ell}\ddefloop

\makeatletter
\newcommand*{\indep}{%
  \mathbin{%
    \mathpalette{\@indep}{}%
  }%
}
\newcommand*{\nindep}{%
  \mathbin{%                   % The final symbol is a binary math operator
    \mathpalette{\@indep}{\not}% \mathpalette helps for the adaptation
  }%
}
\newcommand*{\@indep}[2]{%
  \sbox0{$#1\perp\m@th$}%        box 0 contains \perp symbol
  \sbox2{$#1=$}%                 box 2 for the height of =
  \sbox4{$#1\vcenter{}$}%        box 4 for the height of the math axis
  \rlap{\copy0}%                 first \perp
  \dimen@=\dimexpr\ht2-\ht4-.2pt\relax
  \kern\dimen@
  {#2}%
  \kern\dimen@
  \copy0 %                       second \perp
} 
\makeatother
\fi

\PassOptionsToPackage{square, numbers}{natbib}

\title{Neural Causal Abstractions}

\author{
    Kevin Xia {\normalfont and} Elias Bareinboim
}
\affiliations {
    Causal Artificial Intelligence Lab \\
    Columbia University \\
    \{kevinmxia, eb\}@cs.columbia.edu
}

\begin{document}

\maketitle

\begin{abstract}
The abilities of humans to understand the world in terms of cause and effect relationships, as well as to compress information into abstract concepts, are two hallmark features of human intelligence. 
These two topics have been studied in tandem in the literature under the rubric of causal abstractions theory. 
In practice, it remains an open problem how to best leverage abstraction theory in real-world causal inference tasks, where the true mechanisms are unknown and only limited data is available. 
In this paper, we develop a new family of causal abstractions by clustering variables and their domains. This approach refines and generalizes previous notions of abstractions to better accommodate individual causal distributions that are spawned by Pearl's causal hierarchy. 
We show that such abstractions are learnable in practical settings through Neural Causal Models \citep{xia:etal21}, enabling the use of the deep learning toolkit to solve  various challenging causal inference tasks -- identification, estimation, sampling --  at different levels of granularity. 
Finally, we integrate these results with representation learning to create more flexible abstractions, moving these results closer to practical applications.
Our experiments support the theory and illustrate how to scale causal inferences to high-dimensional settings involving image data.
\end{abstract}

\section{Introduction} \label{sec:intro}

Humans understand the world around them through the use of abstract notions. Biologists can study the function of the liver without understanding the interactions between its subatomic particles studied by physicists. Economists find it more practical to consider macro-level behavior through concepts like aggregate supply and demand rather than studying the purchasing behavior of individuals. At home, we choose to interpret the object in the television as a dog or a car as opposed to a collection of photons or pixels. Humans are highly capable of learning through interacting with the environment and understanding cause and effect between different concepts. Understanding causality is considered a hallmark of human intelligence and allows humans to plan a course of action, determine blame and responsibility, and generalize across environments. It follows that the ability to abstract concepts and study them causally is a key ability expected from modern intelligent systems.

AI systems are built on a foundation of generative models, which are representations of the underlying processes from which data is collected. Standard generative models simply model some joint density of a set of variables of interest, while \emph{causal} generative models further model distributions involving causal interventions and counterfactual relations. In this paper, we study the problem of learning a causal generative model from data, which can be useful for many purposes such as sampling novel causally-consistent data points (i.e.~from interventional or counterfactual distributions). One major challenge is that data is often provided in complex low level forms (e.g., pixels), while it would be more useful in applications to focus on higher level concepts (e.g., dog or car). We would therefore like to learn a more abstract causal generative model at a higher level of granularity, while guaranteeing that the queries from the coarser model match the ground truth.

To formalize this problem, we build on the semantics of a class of generative models called structural causal models (SCMs) \citep{pearl:2k}. An SCM $\cM^*$ describes a collection of mechanisms and distribution over unobserved factors. Each SCM induces three qualitatively different sets of distributions related to the human concepts of ``seeing'' (called observational), ``doing'' (interventional), and ``imagining'' (counterfactual), collectively known as the Ladder of Causation or the Pearl Causal Hierarchy (PCH) \citep{pearl:mackenzie2018,bareinboim:etal20}. The PCH is a containment hierarchy in which each of these distribution sets can be put into increasingly refined layers, where observational distributions go in layer 1 ($\cL_1$), interventional in layer 2 ($\cL_2$), and counterfactual in layer 3 ($\cL_3$). In typical tasks of causal inference, the goal is to obtain a quantity from a higher layer when given data only from lower layers (e.g. inferring interventional quantities from observational data). Still, it is understood that this is generally impossible without additional assumptions since higher layers are underdetermined by lower layers \citep{bareinboim:etal20, ibeling2020probabilistic}.

Generative models can often be implemented in practice as neural networks. Deep learning models have achieved promising success in a variety of applications such as computer vision \citep{NIPS2012_c399862d}, speech recognition \citep{pmlr-v32-graves14}, and game playing \citep{volodymr:etal13}. Many of these successes are attributed to \emph{representation learning} \citep{10.1109/TPAMI.2013.50}, in which the learned representation can be thought of as an abstraction of the data. Further, there has also been growing interest in the idea of incorporating causality into deep models\footnote{Many successful approaches have been developed to estimate causal effects from observational data under the backdoor or conditional  ignorability conditions \citep{shalit2017estimating, louizos2017causal,  NIPS2017_b2eeb736, johansson2018learning, NEURIPS2018_a50abba8, yoon2018ganite,  kallus2018deepmatch, shi2019adapting,  du2020adversarial, Guo_2020}, and also to answer causal queries through neural-parameterized SCMs \citep{kocaoglu2018causalgan, goudet2018learning}.}. Prior work introduced one such model, the Neural Causal Model (NCM), which incorporates the same causal assumptions encoded in a causal diagram to identify and estimate interventional and counterfactual distributions \citep{xia:etal21, xia:etal23}.
\footnote{The literature also includes non-neural approaches for such problems, including estimators with stronger statistical properties such as double robustness and convergence guarantees, for example, \citep{jung2020estimating,jung2020werm,jung2021dml}.  }
Despite the soundness of this approach in theory, current NCM-based methods face challenges when applied to complex real-world settings for various reasons: (1) optimization is difficult when scaled to high dimensions, (2) unprocessed data can come in complicated forms (e.g. images, text, etc.), and (3) the causal diagram is difficult to fully specify in some high-dimensional settings. In this paper, we address these challenges by studying how representation learning and causal reasoning are related to each other and by building on this understanding to develop a neural framework for causal abstraction learning.
%In this work, we address these challenges by incorporating concepts of representation learning and abstractions within NCMs.

Existing works that study causal abstractions set a solid foundation by defining various mathematical notions of abstractions \citep{rubenstein:etal17-causalsem, beckers2019abstracting, Beckers2019-BECACA-8, geiger2023causal, pmlr-v213-massidda23a}. In App.~\ref{sec:related-work}, we explain some of the foundational results and discuss their drawbacks. In particular, we note that existing definitions are declarative; that is, if the lower and higher level models are given, one can use the definition to decide whether the higher level model is indeed an abstraction of the lower level one. However, neither models are available in practice, and one would want to use limited lower level data to learn a higher level causal abstraction. We will expand on the current generation of causal abstractions in two ways. First, given that the true SCM is almost never available in practice, nor entirely learnable from data, we introduce a relaxed notion of abstractions that applies on the layers of the PCH. Second, we develop algorithms to systematically learn abstractions in practice given some structural information about the data, which can then be used for downstream inferential tasks such as causal identification, estimation, or sampling.

%\begin{wrapfigure}{r}{0.4\textwidth}
\begin{figure}
%\vspace{-0.1in}
%\hspace{-0.05cm}
\centering
\includegraphics[width=\linewidth]{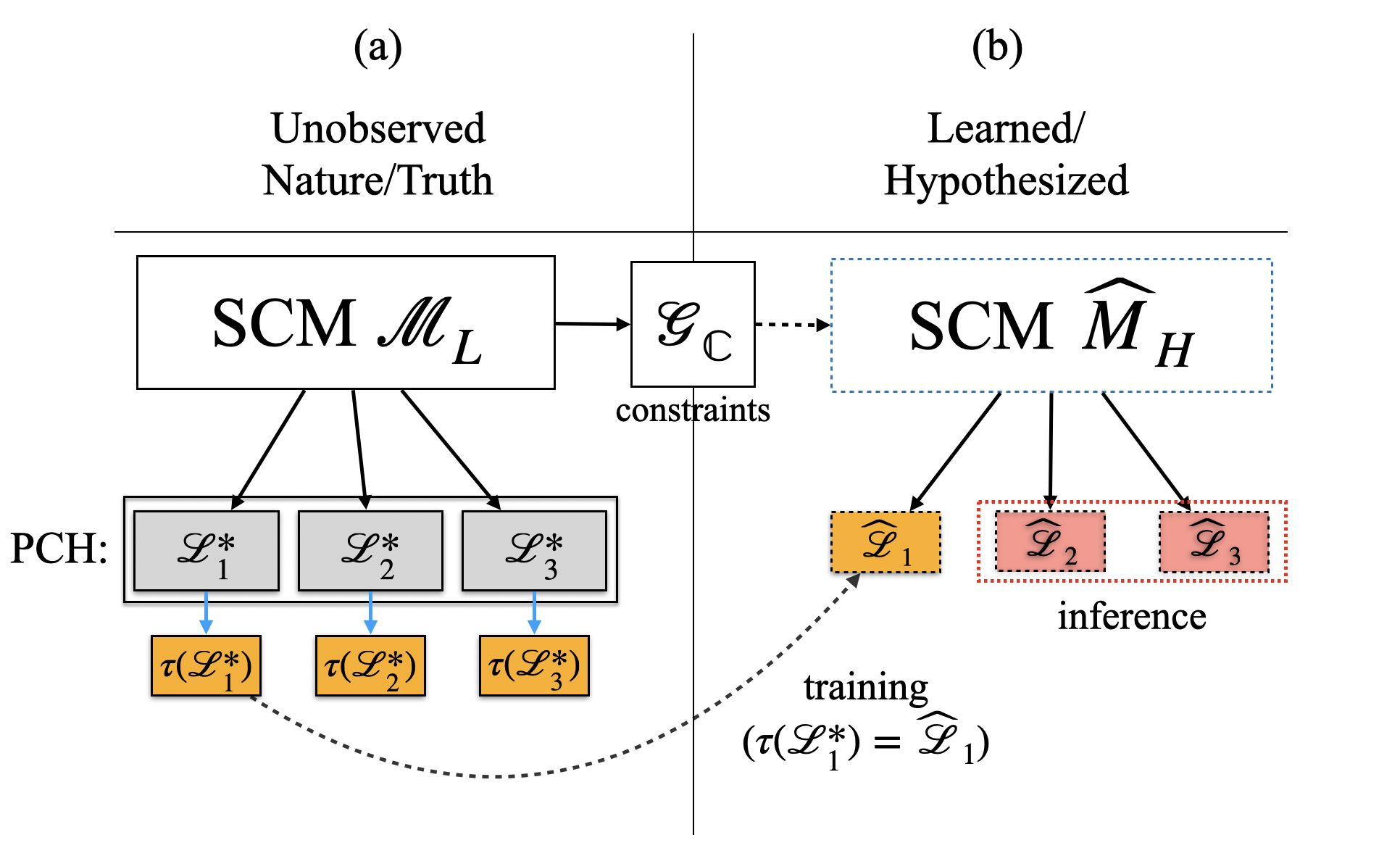}
\caption{Overview of this paper. High-level SCM $\widehat{M}_H$ (right) is trained on available data to serve as an abstract proxy of the true, unobserved, low-level SCM $\cM_L$ (left).}
\label{fig:summary}
\end{figure}
%\vspace{-0.4in}
%\end{wrapfigure}

The general problem tackled by this paper is summarized in Fig.~\ref{fig:summary}. The ground truth model $\cM_L$ (left) is defined over low level variables $\*V_L$ (e.g., pixels), while it may be practical to work in their high level abstract counterparts $\*V_H$ (e.g., dog or car). $\cM_L$ induces distributions from the three layers of the PCH (i.e.~$\cL^*_1$, $\cL^*_2$, $\cL^*_3$), defined over $\*V_L$. In this work, we introduce a new type of abstraction function $\tau$ that maps distributions over $\*V_L$ to ones over $\*V_H$ (i.e.~$\tau(\cL_1^*), \tau(\cL_2^*), \tau(\cL_3^*)$). Furthermore, $\cM_L$ is unobserved, and only limited data is given (e.g., observational data from $\cL^*_1$). The goal is to learn a high-level SCM $\hM_H$ (right) over the high-level variables $\*V_H$ that encodes the given causal constraints ($\cG_{\bbC}$ in the figure) and matches $\cM_L$ on the available data across $\tau$ (e.g.~$\widehat{\mathcal{L}}_1 = \tau(\cL_1^*)$). Then, we investigate when and how the resulting model $\hM_H$ can be used as a surrogate, allowing one to make interventional and counterfactual inferences about the higher layers of $\cM_L$ through the higher layers of $\hM_H$.

As an example, suppose an economist is studying the effects of spending trends of various countries on their average income $Y$. In addition to $Y$, she has collected observational data on several variables of spending trends, such as consumer spending $C$, investments $I$, government spending $G$, imports $M$, and exports $E$. She wants to understand the causal effect of increasing general spending on average income of the population, and one way to do this is to study the causal effect of collectively increasing $C$, $I$, $G$, $M$, and $E$ on $Y$. However, the economist notes that $C$, $I$, $G$, $M$, and $E$ can be aggregated together into a single abstract variable called gross domestic product (GDP). The tools that we introduce in this paper allow her to proceed by constructing a high-level model $\widehat{M}_H$ over the variables GDP and $Y$, encoding the required causal assumptions, and training the model over the given observational data. Despite the high-level variables not matching the original low-level variables (i.e., $C$, $I$, $G$, $M$, $E$), the causal effect of GDP on $Y$ can be queried from the model $\widehat{M}_H$ to solve the problem.

More specifically, our contributions are as follows: In Sec.~\ref{sec:abstract-ncm}, we define a new class of abstractions based on clusters of variables (intervariable) and their domains (intravariable). Building on this new class, we define a notion of abstraction consistency on the layers of the PCH. We then show how to systematically construct an abstraction consistent with all three layers of the PCH and then relate these abstractions to existing definitions. In Sec.~\ref{sec:learning-abs}, we show how to leverage NCM machinery to perform interventional (layer 2) and counterfactual (layer 3) inferences across these abstractions when the true SCM is unavailable.  In Sec.~\ref{sec:applications}, we introduce a variant of the NCM that learns representations of each variable and encodes causal assumptions on the representation level, allowing us to learn abstractions even in settings where the assumption of the availability of clusters is relaxed. Experiments in Sec.~\ref{sec:experiments} corroborate with the theory.

\subsection{Preliminaries}

We now introduce the notation and definitions used throughout the paper. We use uppercase letters ($X$) to denote random variables and lowercase letters ($x$) to denote corresponding values. Similarly, bold uppercase ($\*X$) and lowercase ($\*x$) letters denote sets of random variables and values respectively. We use $\cD_{X}$ to denote the domain of $X$ and $\cD_{\mathbf{X}} = \cD_{X_1} \times \dots \times \cD_{X_k}$ for the domain of $\mathbf{X} = \{X_1, \dots, X_k\}$. We denote $P(\*X = \*x)$ (often shortened to $P(\*x)$) as the probability of  $\*X$ taking the values $\*x$ under the distribution $P(\*X)$.

We utilize the basic semantic framework of structural causal models (SCMs) \citep{pearl:2k}, following the presentation in \citet{bareinboim:etal20}. %as defined in \citep[Ch.~7]{pearl:2k}, and causal diagrams \citep{pearl:95a}.
\begin{definition}[Structural Causal Model (SCM)]
    \label{def:scm}
    A structural causal model $\cM$ is a 4-tuple $\langle \*U, \*V, \cF, P(\*U) \rangle$, where
    \begin{itemize}
        \item $\*U$ is a set of background (exogenous) variables that are determined by factors outside the model;
        \item $\*V$ is a set $\{V_1, V_2, \dots, V_n\}$ of variables, called endogenous, that are determined by other variables in the model -- that is, variables in $\*U \cup \*V$;
        \item $\cF$ is a set of functions $\{f_{V_1}, f_{V_2}, \dots, f_{V_n}\}$ such that each $f_{V_i}$ is a mapping from exogenous parents $\Ui{V_i} \subseteq \*U$ and endogenous parents $\Pai{V_i} \subseteq \*V \setminus V_i$ to $V_i$;
        \item $P(\*U)$ is a probability function defined over $\cD_{\*U}$.
        \hfill $\blacksquare$
    \end{itemize}
\end{definition}

\begin{definition}[Causal Diagram {\citep[Def.~13]{bareinboim:etal20}}]
    \label{def:cg}
    Each SCM $\cM$ induces a causal diagram $\cG$, constructed as follows:
    \begin{enumerate}
        \item add a vertex for each $V_i \in \*V$;
        \item add a directed arrow $(V_j \rightarrow V_i)$ for every $V_i \in \*V$ and $V_j \in \Pai{V_i}$; and
        \item add a dashed-bidirected arrow $(V_j  \dashleftarrow \dashrightarrow V_i)$ for every pair $V_i, V_j \in \*V$ such that $\Ui{V_i}$ and $\Ui{V_j}$ are not independent (Markovianity is not assumed).
        \hfill $\blacksquare$
    \end{enumerate}
\end{definition}

Our treatment is constrained to \emph{recursive} SCMs, which implies acyclic causal diagrams, with finite discrete domains over endogenous variables $\mathbf{V}$. 

Counterfactual (and also interventional and observational) quantities can be computed from SCM $\cM$ as follows: % , as shown next; for further discussion on the semantics, please refer to \citep[Sec.~1.2]{bareinboim:etal20}.
%Superscripts are omitted when unambiguous.
\begin{definition}[Layer 3 Valuation {\citep[Def.~7]{bareinboim:etal20}}] 
\label{def:l3-semantics}
An SCM $\cM$ induces layer $\cL_3(\cM)$, a set of distributions over $\*V$, each with the form $P(\*Y_*) = P(\*Y_{1[\*x_1]}, \*Y_{2[\*x_2], \dots})$ such that
\begin{align}
    \label{eq:def:l3-semantics}
    & P^{\cM}(\*y_{1[\*x_1]}, \*y_{2[\*x_2]}, \dots) = \nonumber \\
    & \int_{\cD_{\mathbf{U}}} \mathbf{1}\left[\*Y_{1[\*x_1]}(\*u)=\*y_1, \*Y_{2[\*x_2]}(\*u) = \*y_2, \dots \right] dP(\*u)
\end{align}
where ${\*Y}_{i[\*x_i]}(\*u)$ is evaluated under 
 $\mathcal{F}_{\*x_i}\! :=\! \{f_{V_j}\! :\! V_j \in \*V \setminus \*X_i\} \cup \{f_X \leftarrow x\! :\! X \in \*X_i\}$. $\cL_2$ is the subset of $\cL_3$ for which all $\*x_i$ are equal, and $\cL_1$ is the subset for which all $\*X_i = \emptyset$.
 \hfill $\blacksquare$
\end{definition}
Each $\*Y_i$ corresponds to a set of variables in a world where the original mechanisms $f_X$ are replaced with constants $\*x_i$ for each $X \in \*X_i$; this is also known as the mutilation procedure. This procedure corresponds to interventions, and we use subscripts to denote the intervening variables (e.g. $\*Y_{\*x}$) or subscripts with brackets when the variables are indexed (e.g. $\*Y_{1[\*x_1]}$). For instance, $P(y_x, y'_{x'})$ is the probability of the joint counterfactual event $Y=y$ had $X$ been $x$ and $Y=y'$ had $X$ been $x'$. 

We use the notation $\cL_i(\cM)$ to denote the set of $\cL_i$ distributions from $\cM$. We use $\bbZ$ to denote a set of quantities from Layer 2 (i.e. $\bbZ = \{P(\*V_{\*z_k})\}_{k=1}^{\ell}$), and $\bbZ(\cM)$ denotes those same quantities induced by SCM $\cM$ (i.e. $\bbZ(\cM) = \{P^{\cM}(\*V_{\*z_k})\}_{k=1}^{\ell}$).

We also build on Neural Causal Models (NCMs), in particular for performing causal inferences:
\begin{definition}[$\cG$-Constrained Neural Causal Model ($\cG$-NCM) {\citep[Def.~7]{xia:etal21}}]
    \label{def:gncm}
    Given a causal diagram $\cG$, a $\cG$-constrained Neural Causal Model (for short, $\cG$-NCM) $\widehat{M}(\bm{\theta})$ over variables $\*V$ with parameters $\bm{\theta} = \{\theta_{V_i} : V_i \in \*V\}$ is an SCM $\langle \widehat{\*U}, \*V, \widehat{\cF}, P(\widehat{\*U}) \rangle$ such that
    \begin{itemize}
        \item $\widehat{\*U} = \{\widehat{U}_{\*C} : \*C \in \bbC(\cG)\}$, where $\bbC(\cG)$ is the set of all maximal cliques over bidirected edges of $\cG$;
        
        \item $\widehat{\cF} = \{\hat{f}_{V_i} : V_i \in \*V\}$, where each $\hat{f}_{V_i}$ is a feedforward neural network parameterized by $\theta_{V_i} \in \bm{\theta}$ mapping values of $\Ui{V_i} \cup \Pai{V_i}$ to values of $V_i$ for $\Ui{V_i} = \{\widehat{U}_{\*C} : \widehat{U}_{\*C} \in \widehat{\*U} \text{ s.t. } V_i \in \*C\}$ and $\Pai{V_i} = \Parents_{\cG}(V_i)$;
        
        \item $P(\widehat{\*U})$ is defined s.t.\ $\widehat{U} \sim \unif(0, 1)$ for each $\widehat{U} \in \widehat{\*U}$.
        \hfill $\blacksquare$
    \end{itemize}
\end{definition}

\section{Abstractions of the Pearl Causal Hierarchy} \label{sec:abstract-ncm}

The discussion of abstractions begins with defining causal variables. In many established causal inference tasks, it is typically assumed that there is a well-specified and known set of endogenous variables of interest $\*V$, and nature is modeled by a collection of mechanisms that assign values to each of these variables. However, in practice, the definition of $\*V$ may not always be clear. In particular, the variables of interest may not align with the features of the data. For example, in an economic system, perhaps data on each individual consumer is collected, but the variable of interest is an aggregate measure like gross domestic product (GDP). In image data, perhaps the pixel values are collected, but the variables of interest are related to the objects of the image, not the individual pixels.

Acknowledging that the data is not always provided in the best choice of granularity, the causal abstraction literature typically defines two sets of variables, $\*V_L$ and $\*V_H$, which describe the lower level and higher level settings, respectively. For example, $\*V_L$ might describe the pixels of an image, while $\*V_H$ might describe its structural content. They are typically modeled by corresponding causal models $\cM_L$ and $\cM_H$, respectively.

In this section, we study on the distinction between low level variables $\*V_L$ (e.g.~pixels) and their higher level counterparts $\*V_H$ (e.g.~image) from the perspective of individual distributions of the PCH. We consider nature's underlying SCM $\cM_L$ defined over low level variables, $\*V_L$, and the goal is to reason about the higher level variables $\*V_H$ given data on $\*V_L$\footnote{For concreteness, we assume that $\cM_L$ is an SCM, but the underlying generative model can be left implicit as explained in Appendix \ref{app:cons-hierarchy}.}.

\subsection{Constructive Abstraction Functions}

The connection between $\*V_H$ and $\*V_L$ can be described through a mapping between their domains, $\tau: \cD_{\*V_L} \rightarrow \cD_{\*V_H}$. Here, we consider a family of abstraction functions where $\tau$ is based on clusters of the variables and values of $\*V_L$:

\begin{definition}[Inter/Intravariable Clusterings]
    \label{def:var-clusterings}
    Let $\cM$ be an SCM over variables $\*V$.
    \begin{enumerate}
        \item A set $\bbC$ is said to be an intervariable clustering of $\*V$ if $\bbC = \{\*C_1, \*C_2, \dots \*C_n\}$ is a partition of a subset of $\*V$. $\bbC$ is further considered admissible w.r.t.~$\cM$ if for any $\*C_i \in \bbC$ and any $V \in \*C_i$, no descendent of $V$ outside of $\*C_i$ is an ancestor of any variable in $\*C_i$. That is, there exists a topological ordering of the clusters of $\bbC$ relative to the functions of $\cM$.
        \item A set $\bbD$ is said to be an intravariable clustering of variables $\*V$ w.r.t.~$\bbC$ if $\bbD = \{\bbD_{\*C_i} : \*C_i \in \bbC\}$, where $\bbD_{\*C_i} = \{\cD^{1}_{\*C_i}, \cD^{2}_{\*C_i}, \dots, \cD^{m_i}_{\*C_i}\}$ is a partition (of size $m_i$) of the domains of the variables in $\*C_i$, $\cD_{\*C_i}$ (recall that $\cD_{\*C_i}$ is the Cartesian product $\cD_{V_1} \times \cD_{V_2} \times \dots \times \cD_{V_k}$ for $\*C_i = \{V_1, V_2, \dots, V_k\}$, so elements of $\cD_{\*C_i}^j$ take the form of tuples of the value settings of $\*C_i$).
        \hfill $\blacksquare$
    \end{enumerate}
\end{definition}

In words, $\*V$ is divided into $n$ subsets or clusters $\*C_1, \dots, \*C_n$ (variables that are not put into one of the clusters are projected away), and they are called \emph{intervariable} clusters because the variables themselves are divided apart. Admissibility implies that the recursivity assumption of SCMs is retained through the intervariable clusters. Then, the joint \emph{domains} of each of these $n$ clusters are further partitioned. For example, for a specific intervariable cluster $\*C_i$, the domain $\cD_{\*C_i}$ contains the set of all tuples of values of $\*C_i$, and $\bbD_{\*C_i}$ describes a partition $\cD_{\*C_i}^1, \dots, \cD_{\*C_i}^{m_i}$ of size $m_i$ over this set of values (i.e.\ each $\cD_{\*C_i}^{j} \subseteq \cD_{\*C_i}$). The \emph{intravariable} clusters are the set of the value partitions over each intervariable cluster, and the term ``intravariable'' denotes that the clustering is within the variable domains. Intuitively speaking, intervariable clusters partition the low level variables to describe each high level variable as a collection of low level variables. Intravariable clusters then describe the domains of these high level variables by partitioning the corresponding value spaces of these intervariable clusters.

\begin{example}
\label{ex:bmi}
Consider a study on the effects of certain food dishes on body mass index (BMI), inspired by nutrition studies like \citet{gamba_schuchter_rutt_seto_2014}. Data is collected on individuals eating at restaurants, including the restaurant ($R$), dish ordered ($D$), the amount of carbohydrates ($C$), fat ($F$), and protein ($P$) in the dish, and the BMI of the customer ($B$). That is, $\*V_L = \{R, D, C, F, P, B\}$. One food scientist argues that any nutritional impact of the food on BMI could be abstracted based on how many calories are in each dish. One may then be tempted to cluster the variables $C$, $F$, and $P$ together into one variable, named calories, labeled $Z$. This is an example of intervariable clustering. 

To denote this formally, we may choose $\bbC = \{\*C_1 = \{B\}, \*C_2 = \{C, F, P\}, \*C_3 = \{D\}\}$ as the intervariable clusters. In this case, $B$ and $D$ are placed in their own clusters, $\*C_1$ and $\*C_3$, respectively. $C$, $F$, and $P$ are all clustered together into $\*C_2$. $R$ is not included and is abstracted away, which may be desirable if $R$ is not relevant to the study. Collectively, $\*C_1$, $\*C_2$, and $\*C_3$ form a partition of the subset of $\*V_L$ without $R$. Each of the clusters of $\bbC$ will correspond to a high level variable of $\*V_H$. In this case, for example, let $Z$ denote the high level variable corresponding to cluster $\*C_2$, interpreted as calories. This is shown at the top of Fig.~\ref{fig:cluster-example} (red).

The domain of $\*C_2$ contains every tuple of $C$, $F$, and $P$, but the domain of $Z$ can be simplified. After all, the computation of calories can be specified as $Z = 4C + 9F + 4P$, which means that two sets of values, $(c_1, f_1, p_1), (c_2, f_2, p_2)$ are considered equivalent if $4c_1 + 9f_1 + 4p_1 = 4c_2 + 9f_2 + 4p_2$. This clustering of domain values is an example of intravariable clustering, shown at the bottom of Fig.~\ref{fig:cluster-example} (blue). More formally, the intervariable clusters would be denoted $\bbD = \{\bbD_{\*C_1}, \bbD_{\*C_2}, \bbD_{\*C_3}\}$, where each $\bbD_{\*C_i}$ is a partition of $\cD_{\*C_i}$. In the case of $\bbD_{\*C_2}$, we may define $\bbD_{\*C_2} = \{\cD_{\*C_2}^1, \cD_{\*C_2}^2, \dots\}$, where each $\cD_{\*C_2}^j$ is a collection of tuples $(c, f, p) \in \cD_{\*C_2}$ corresponding to some specific value $4c + 9f + 4p$. In Fig.~\ref{fig:cluster-example} for example, $\cD_{\*C_2}^1 = \{(c, f, p) : 4c + 9f + 4p = 200, (c, f, p) \in \cD_{\*C_2}\}$. Each of the intravariable clusters correspond to a domain value of the high level variable. For example, $\cD_{\*C_2}^1$ corresponds to a value of $Z = 200$.
\hfill $\blacksquare$
\end{example}

\begin{figure}
%\begin{wrapfigure}{r}{0.5\textwidth}
%\vspace{-0.2in}
%\hspace{-0.05cm}
\centering
\includegraphics[width=1.0\linewidth]{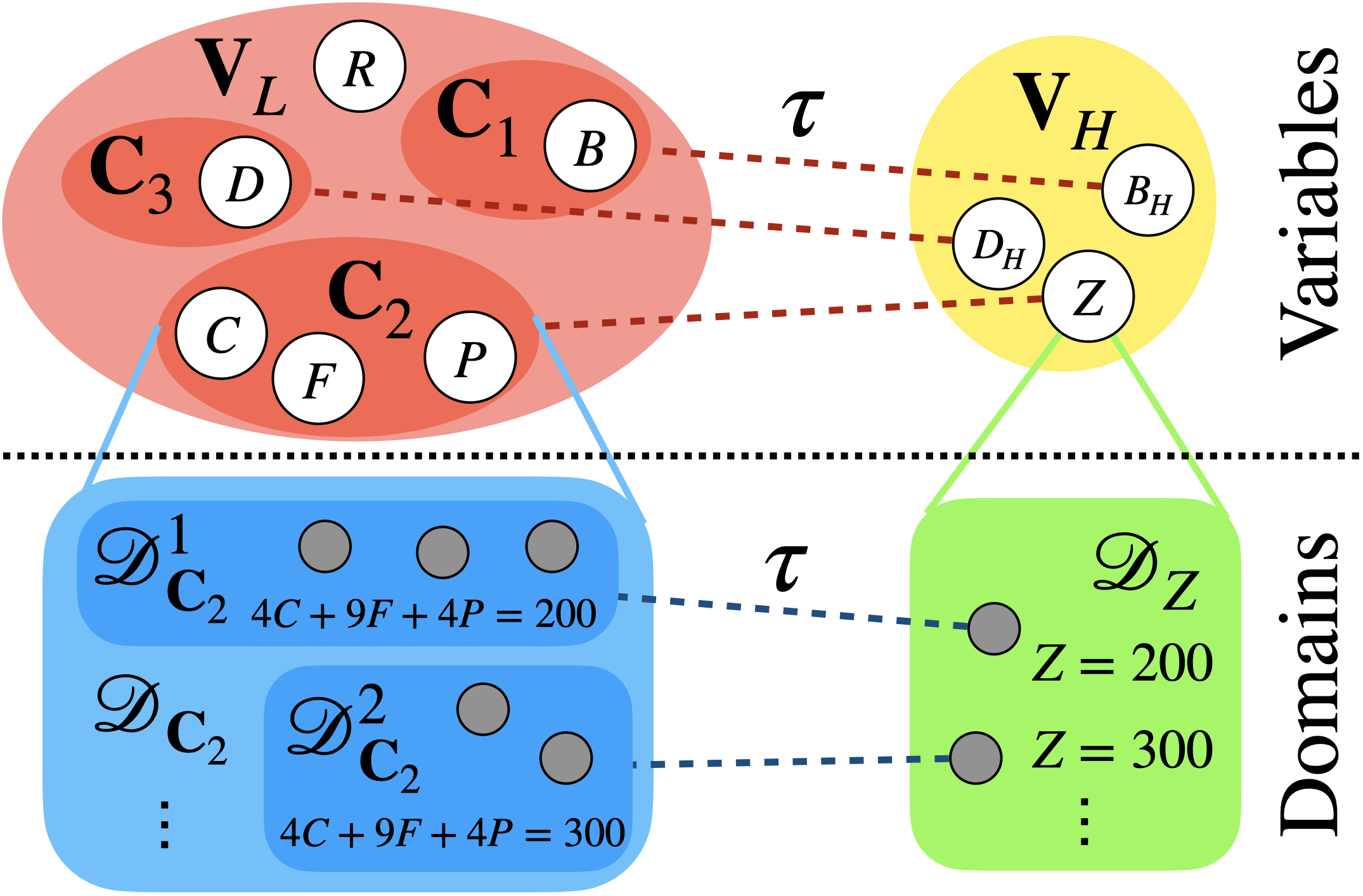}
\caption{Example of a constructive abstraction function $\tau$ w.r.t.\ corresponding inter/intravariable clusters. \textbf{Top (intervariable):} The low-level variables, dish ($D$) and BMI ($B$), are in their own clusters while restaurant ($R$) is abstracted away. Carbohydrates ($C$), fat ($F$), and protein ($P$) are clustered together and are mapped to a single variable, calories ($Z$). \textbf{Bottom (intravariable):} The intravariable clustering for $\*C_2 = \{C, F, P\}$ is shown. Calories $Z$ can be computed from $C, F, P$ using the formula $Z = 4C + 9F + 4P$. This means that the domain is partitioned such that two different values, $(c_1, f_1, p_1), (c_2, f_2, p_2)$ are in the same intravariable cluster if $4c_1 + 9f_1 + 4p_1 = 4c_2 + 9f_2 + 4p_2$.}
\label{fig:cluster-example}
%\vspace{-0.3in}
%\end{wrapfigure}
\end{figure}

For the remainder of this paper, we consider settings where the intervariable clusters are admissible. Collectively, given an intervariable clustering $\bbC$ and intravariable clustering $\bbD$ of $\*V_L$, an abstraction function $\tau$ can be defined as follows.

\begin{definition}[Constructive Abstraction Function]
    \label{def:tau}
    A function $\tau: \cD_{\*V_L} \rightarrow \cD_{\*V_H}$ is said to be a constructive abstraction function w.r.t.~inter/intravariable clusters $\bbC$ and $\bbD$ iff
    \begin{enumerate}
        \item There exists a bijective mapping between $\*V_H$ and $\bbC$ such that each $V_{H, i} \in \*V_H$ corresponds to $\*C_i \in \bbC$;
        \item For each $V_{H, i} \in \*V_H$, there exists a bijective mapping between $\cD_{V_{H, i}}$ and $\bbD_{\*C_i}$ such that each $v_{H, i}^j \in \cD_{V_{H, i}}$ corresponds to $\cD^j_{\*C_i} \in \bbD_{\*C_i}$; and 
        \item $\tau$ is composed of subfunctions $\tau_{\*C_i}$ for each $\*C_i \in \bbC$ such that $\*v_H = \tau(\*v_L) = (\tau_{\*C_i}(\*c_i) : \*C_i \in \bbC)$, where $\tau_{\*C_i}(\*c_i) = v^j_{H,i}$ if and only if $\*c_i \in \cD^{j}_{\*C_i}$. We also apply the same notation for any $\*W_L \subseteq \*V_L$ such that $\*W_L$ is a union of clusters in $\bbC$ (i.e. $\tau(\*w_L) = (\tau_{\*C_i}(\*c_i) : \*C_i \in \bbC, \*C_i \subseteq \*W_L)$).
    \hfill $\blacksquare$
    \end{enumerate}
\end{definition}

In words, through the subfunction $\tau_{\*C_i}$, each low level cluster $\*C_i \in \bbC$ maps to a single high level variable $V_{H, i} \in \*V_H$, and the value $\*c_i \in \cD_{\*C_i}$ maps to a corresponding high level value $v_{H, i}^j \in \cD_{V_{H, i}}$. Specifically, $\tau_{\*C_i}(\*c_i)$ maps to $v_{H, i}^j$ if $\*c_i$ is in the intravariable cluster $\cD_{\*C_i}^j$. Then, the overall function $\tau$ is simply composed of the subfunctions $\tau_{\*C_i}$. Intuitively, $\tau$ is a constructive abstraction function if it maps $\*V_L$ to $\*V_H$ by first grouping the variables w.r.t.~their corresponding intervariable cluster in $\bbC$ (red maps to yellow in Fig.~\ref{fig:cluster-example} (top)), followed by assigning each cluster a value based on which intravariable cluster they belong in $\bbD$ (blue maps to green in Fig.~\ref{fig:cluster-example} (bottom)). As a result, $\*V_H$ can be interpreted such that $\*V_H = \bbC$ and $\cD_{V_{H, i}} = \bbD_{\*C_i}$ for each $V_{H, i} \in \*V_H$\footnote{For another example of abstractions constructed from clusters, see App.~\ref{app:examples} Ex.~\ref{ex:clusters}. For examples of abstraction functions that are not constructive, see Ex.~\ref{ex:non-constructive-tau}.}. This construction of $\tau$ means that $\tau$ is unique given the clusters $\bbC$ and $\bbD$ (up to a renaming of the variables $\*V_H$ and its values $\cD_{\*V_H}$).

\begin{example}[Example \ref{ex:bmi} continued]
    Suppose the high level variables are denoted as $\*V_H = \{D_H, Z, B_H\}$, where $D_H$ and $B_H$ represent the high level counterparts of $D$ and $B$ that remain unchanged across the abstraction. Each high-level variable (i.e.~$D_H$, $Z$, $B_H$) corresponds to an intervariable cluster (i.e., $\*C_1$, $\*C_2$, $\*C_3$, respectively), establishing a bijective connection between $\*V_H$ and $\bbC$. Each of their domains also correspond to an intravariable cluster in $\bbD$. For example, each value of $Z = z$ corresponds to the choice of $\cD_{\*C_2}^j$ such that $4c + 9f + 4p = z$. The constructive abstraction function $\tau$ constructed from the clusters $\bbC$ and $\bbD$ would map $(D_H, Z, B_H) \gets \tau(R, D, C, F, P, D)$, which can be decomposed as
    \begin{align}
        & \tau(R, D, C, F, P, D) \nonumber \\
        &= (\tau_{\*C_1}(D), \tau_{\*C_2}(C, F, P), \tau_{\*C_3}(B)) \\
        &= (D, 4C + 9F + 4P, B).
    \end{align}
    Observe that $\tau$ is broken down into $\tau_{\*C_1}$, $\tau_{\*C_2}$, and $\tau_{\*C_3}$, which maps the variables of each intervariable cluster to their corresponding high level variable $D_H$, $Z$, and $B_H$, respectively. $D$ and $B$ are not affected by the abstraction in this example, so $\tau_{\*C_1}$ and $\tau_{\*C_3}$ are the identity function, directly setting $D_H \gets D$ and $B_H \gets B$. However, the calories, $Z$, is computed through $\tau_{\*C_2}(C, F, P) = 4C + 9F + 4P$. This ensures that all values from $\cD_{\*C_2} = \cD_{C} \times \cD_{F} \times \cD_{P}$ that are in the same intravariable cluster are mapped to the same value of $Z$. This mapping is illustrated in Fig.~\ref{fig:cluster-example}.
    \hfill $\blacksquare$
\end{example}

Note that the relationship between $\*V_L$ and $\*V_H$ modeled by $\tau$ is not causal. Rather, the contents of $\*V_L$ \emph{constitute} $\*V_H$\footnote{The distinction between causal and constitutional relationships is important and is explained in detail in Appendix \ref{app:cons-hierarchy}.}. Intuitively, two variables of $\*V_L$ are mapped to the same intervariable cluster if they constitute the same high level variable (e.g. two pixels of the same dog), and two values are mapped to the same intravariable cluster if, from a higher level perspective, they are functionally identical (e.g.~same image of the dog but rotated or cropped). In this sense, intravariable clustering can be thought of as invariances in the data, since downstream functions are \emph{invariant} to values that are in the same intravariable cluster\footnote{This analogy is explored further in Apps.~\ref{app:invariance-condition} and \ref{app:rep-learning}.}.

This paper will focus on abstractions based on constructive abstraction functions $\tau$ created from intervariable and intravariable clusters. This is in contrast with the previous works on causal abstractions discussed in App.~\ref{sec:related-work}, which leave the functional form of $\tau$ implicit.\footnote{In a related setting, there may be cases where information about $\cM_H$ is provided while $\tau$ is unknown. % (i.e.~it is not clear what kind of abstraction relationship is shared between the low and high-level spaces). 
In such cases, there are works in the literature that aim to learn $\tau$ \citep{DBLP:conf/clear2/ZennaroDAWD23, felekis:etal24} such that desired abstraction properties hold. Such functions may or may not be constructive, refer to App.~\ref{app:learn-tau-given-mh} for further details.} One benefit of making $\tau$ concrete is that it allows for a rigorous definition of equivalence between the distributions of a low level model and that of a high level model, as will be elaborated next.

\subsection{Layer-Specific Abstractions}

Ultimately, we would like to study causal properties of $\*V_L$ through their higher level counterparts $\*V_H$. A sensible goal is, therefore, to learn an SCM $\cM_H$ over $\*V_H$, which can then be queried for causal inference tasks. Still, even if $\*V_H$ and $\*V_L$ are connected through some function $\tau$, this alone does not imply that $\cM_H$ is an abstraction of $\cM_L$. This is the case since the distributions over $\*V_H$ induced by $\cM_H$ may not have any clear connection with the distributions over $\*V_L$. To explain this point with a simple example, suppose $\*V_L = \{X, Y\}$ and $\*V_H = \*V_L$ (i.e., $\tau$ is the identity function). However, in $\cM_L$, $f_X$ takes $Y$ as an argument, while in $\cM_H$, $f_Y$ takes $X$ as an argument. These two models are obviously unrelated despite sharing the same set of variables.

When two SCMs are defined over the same space of variables, one can verify that they are similar if they induce the same distributions. For example, an SCM $\cM'$ is $\cL_2$-consistent with $\cM$ if $\cL_2(\cM') = \cL_2(\cM)$, that is, $\cM$ and $\cM'$ match in every interventional distribution \citep{bareinboim:etal20, xia:etal21}. However, when two SCMs are defined over different variable spaces, comparing their distributions is no longer well-defined. Hence, a different notion of consistency is needed to compare an SCM over $\*V_L$ with another over $\*V_H$ through $\tau$.

We first note that not all low-level quantities have a clear interpretation in the high-level setting. For instance, in Example \ref{ex:bmi}, the low-level quantity $P(C=5)$ does not have a counterpart in $\*V_H$ because $C$ is clustered together with $F$ and $P$ in the intervariable clusters. As another example, $P(C = 5, F = 0, P = 0)$ also does not have a counterpart in $\*V_H$ because $\tau$ maps $(C = 5, F = 0, P = 0)$ to the same value of $Z$ as other quantities such as $(C = 0, F = 0, P = 5)$. In both cases, $Z = 4C + 9F + 4P = 20$. Hence, $P(C = 5, F = 0, P = 0)$ is not equivalent to $P(Z = 20)$, which, in a sense, represents all values of $(C, F, P)$ such that $4C + 9F + 4P = 20$.

To define the low-level counterfactual quantities that have high-level counterparts through $\tau$, first denote $\*Y_{L,*}$ as a set of counterfactual variables over $\*V_L$. That is,
\begin{equation}
    \label{eq:valid-low-ctf}
    \*Y_{L, *} = \left( \*Y_{L, 1[\*x_{L, 1}]}, \*Y_{L, 2[\*x_{L, 2}]}, \dots\right),
\end{equation}
where each $\*Y_{L, i[\*x_{L, i}]}$ corresponds to the potential outcomes of the variables $\*Y_{L, i}$ under the intervention $\*X_{L, i} = \*x_{L, i}$. Each $\*Y_{L, i}$ and $\*X_{L, i}$ must be unions of clusters from $\bbC$ (i.e.~$\*Y_{L, i} = \bigcup_{\*C \in \bbC'} \*C$ for some $\bbC' \subseteq \bbC$) such that $\tau(\*Y_{L, i})$ and $\tau(\*X_{L, i})$ are well-defined (i.e.~$\tau(\*Y_{L, i}) = \left(\bigwedge_{\*C \in \bbC'} \tau_{\*C}(\*C) \right)$). For instance, from Ex.~\ref{ex:bmi}, one term could be $\*Y_{L, i} = \{D, C, F, P\} = \*C_1 \cup \*C_2$, but $\*Y_{L, i} = \{D, C, F\}$ would be invalid since $C$ and $F$ are clustered together with $P$. For the high-level counterpart, denote
\begin{align}
    & \*Y_{H, *} = \tau(\*Y_{L, *}) \\
    &= \left(\tau(\*Y_{L, 1[\tau(\*x_{L, 1})]}), \tau(\*Y_{L, 2[\tau(\*x_{L, 2})]}), \dots\right). \label{eq:valid-high-ctf}
\end{align}
For any value $\*y_{H, *} \in \cD_{\*Y_{H, *}}$, denote
\begin{equation}
    \cD_{\*Y_{L, *}}(\*y_{H, *}) = \{\*y_{L, *} : \*y_{L, *} \in \cD_{\*Y_{L, *}}, \tau(\*y_{L, *}) = \*y_{H, *}\},
\end{equation}
that is, the set of all values $\*y_{L, *}$ such that $\tau(\*y_{L, *}) = \*y_{H, *}$. Considering again Ex.~\ref{ex:bmi}, if $\*y_{H, *}$ refers to the value of $Z = 20$, then $\cD_{\*Y_{L, *}}(\*y_{H, *})$ is the set of all tuples $(c, f, p) \in \cD_{\*C_2}$ such that $\tau(C = c, F = f, P = p) = (Z = 20)$ (i.e., $4c + 9f + 4p = 20$).

We can now define a notion of consistency relating low level counterfactual quantities to high level counterparts.
\begin{definition}[$Q$-$\tau$ Consistency]
    \label{def:q-tau-consistency}
    Let $\cM_L$ and $\cM_H$ be SCMs defined over variables $\*V_L$ and $\*V_H$, respectively. Let $\tau: \cD_{\*V_L} \rightarrow \cD_{\*V_H}$ be a constructive abstraction function w.r.t.~clusters $\bbC$ and $\bbD$. Let
    \begin{equation}
        \label{eq:q-valid}
        Q = \sum_{\*y_{L, *} \in \cD_{\*Y_{L, *}}(\*y_{H, *})} P(\*Y_{L, *} = \*y_{L, *})
    \end{equation}
    be a low-level Layer 3 quantity of interest (for some $\*y_{H, *} \in \cD_{\*Y_{H, *}}$), as expressed in Eq.~\ref{eq:valid-low-ctf}, and let
    \begin{equation}
        \label{eq:tauq-valid}
        \tau(Q) = P(\*Y_{H, *} = \*y_{H, *})
    \end{equation}
    be its high level counterpart, as expressed in Eq.~\ref{eq:valid-high-ctf}.
    We say that $\cM_H$ is $Q$-$\tau$ consistent with $\cM_L$ if
    \begin{equation}
        \label{eq:q-tau-consistency}
        \begin{split}
            & \sum_{\*y_{L, *} \in \cD_{\*Y_{L, *}}(\*y_{H, *})} P^{\cM_L}(\*Y_{L, *} = \*y_{L, *}) \\
            &= P^{\cM_H}(\*Y_{H, *} = \*y_{H, *}),
        \end{split}
    \end{equation}
    that is, the value of $Q$ induced by $\cM_L$ is equal to the value of $\tau(Q)$ induced by $\cM_H$\footnote{Note that the equality in Eq.~\ref{eq:q-tau-consistency} is consistent with the push-forward measure through $\tau$.}. Furthermore, if $\cM_H$ is $Q$-$\tau$ consistent with $\cM_L$ for all $Q \in \cL_i(\cM_L)$ of the form of Eq.~\ref{eq:q-valid}, then $\cM_H$ is said to be $\cL_i$-$\tau$ consistent with $\cM_L$.
    \hfill $\blacksquare$
\end{definition}

In words, suppose $Q$ is a quantity from $\*V_L$ in the form of Eq.~\ref{eq:q-valid}. That is, it is a counterfactual quantity such that the variables of each term $\*Y_{L, i}$ and each intervention $\*X_{L, i}$ are unions of clusters in $\bbC$, and it is summed over values of $\*y_{L, i}$ that map to one specific set of high level variables $\*y_{H, i}$. Then, a query of this form has a counterpart $\tau(Q)$, obtained by applying $\tau$ on each term, shown in Eq.~\ref{eq:tauq-valid}. We say that $\cM_H$ is $Q$-$\tau$ consistent with $\cM_L$ if the value of $\tau(Q)$, computed from $\cM_H$, is equal to the value of $Q$ computed from $\cM_L$ (i.e.~Eq.~\ref{eq:q-tau-consistency} holds). Note that Def.~\ref{def:q-tau-consistency} naturally applies to the $\cL_2$ case (i.e.~all $\*x_{L, i}$ are identical) and the $\cL_1$ case (i.e.~all $\*X_{L, i} = \emptyset$).

Def.~\ref{def:q-tau-consistency} delineates the formal connection between quantities of $\cM_L$ and $\cM_H$. Intuitively, $\cM_H$ can only be viewed as an abstraction of $\cM_L$ for the quantities in which they are $\tau$-consistent. Consider the following example to ground the discussion.

\begin{figure*}
    \centering
    \begin{tabular}{l|llll|llll|ll|l}
    \hline \hline
       & $U_{RY}$ & $U_A$ & $U_B$ & $U_Y$ & $R$ & $A$ & $B$ & $Y$ & $Y_{A=0, B=0}$ & $Y_{A=1, B=1}$ & $P$   \\ \hline \hline
    0  & 0        & 0     & 0     & 0     & 0   & 0   & 0   & 0   & 0              & 0              & $p_0 = 0.288$ \\ \hline
    1  & 0        & 0     & 0     & 1     & 0   & 0   & 0   & 1   & 1              & 1              & $p_1 = 0.032$ \\ \hline
    2  & 0        & 0     & 1     & 0     & 0   & 0   & 1   & 0   & 0              & 0              & $p_2 = 0.072$ \\ \hline
    3  & 0        & 0     & 1     & 1     & 0   & 0   & 1   & 1   & 1              & 1              & $p_3 = 0.008$ \\ \hline
    4  & 0        & 1     & 0     & 0     & 0   & 1   & 0   & 0   & 0              & 0              & $p_4 = 0.072$ \\ \hline
    5  & 0        & 1     & 0     & 1     & 0   & 1   & 0   & 1   & 1              & 1              & $p_5 = 0.008$ \\ \hline
    6  & 0        & 1     & 1     & 0     & 0   & 1   & 1   & 0   & 0              & 0              & $p_6 = 0.018$ \\ \hline
    7  & 0        & 1     & 1     & 1     & 0   & 1   & 1   & 1   & 1              & 1              & $p_7 = 0.002$ \\ \hline
    8  & 1        & 0     & 0     & 0     & 1   & 1   & 1   & 1   & 0              & 1              & $p_8 = 0.288$ \\ \hline
    9  & 1        & 0     & 0     & 1     & 1   & 1   & 1   & 0   & 1              & 0              & $p_9 = 0.032$ \\ \hline
    10 & 1        & 0     & 1     & 0     & 1   & 1   & 0   & 0   & 0              & 1              & $p_{10} = 0.072$ \\ \hline
    11 & 1        & 0     & 1     & 1     & 1   & 1   & 0   & 1   & 1              & 0              & $p_{11} = 0.008$ \\ \hline
    12 & 1        & 1     & 0     & 0     & 1   & 0   & 0   & 0   & 0              & 1              & $p_{12} = 0.072$ \\ \hline
    13 & 1        & 1     & 0     & 1     & 1   & 0   & 0   & 1   & 1              & 0              & $p_{13} = 0.008$ \\ \hline
    14 & 1        & 1     & 1     & 0     & 1   & 0   & 1   & 0   & 0              & 1              & $p_{14} = 0.018$ \\ \hline
    15 & 1        & 1     & 1     & 1     & 1   & 0   & 1   & 1   & 1              & 0              & $p_{15} = 0.002$ \\ \hline
    \end{tabular}
    \caption{Values computed from $\cM_L$ in Example \ref{ex:drug-tau}.}
    \label{fig:drug-ex-u-tab}
\end{figure*}

\begin{example}
    \label{ex:drug-tau}
    \allowdisplaybreaks
    
    Consider a study on a new cancer drug. The drug is given in two doses, and patients take the second dose a month after the first dose. The variables observed are whether the individual takes the first dose ($A$), whether they take the second dose ($B$), whether they recover ($Y$), and whether they come from a wealthy background ($R$) and therefore have better nutrition and medical care. The SCM $\cM_L$ is as follows:
    \begin{align}
        \*U_L &= \{U_{RY}, U_A, U_B, U_{Y}\} \\
        \*V_L &= \{R, A, B, Y\} \\
        \cF_L &= 
        \begin{cases}
            R \gets f^L_R(u_{RY}) = u_{RY} \\
            A \gets f^L_A(r, u_{A}) = r \oplus u_A \\
            B \gets f^L_B(r, a, u_{B}) = (r \wedge a) \oplus u_B \\
            Y \gets f^L_Y(a, b, u_{RY}, u_Y) \\
            \quad = ((a \wedge b) \wedge u_{RY}) \oplus u_Y
        \end{cases} \\
        P(\*U_L) &= 
        \begin{cases}
            P(U_{RY} = 1) = 0.5 \\
            P(U_A = 1) = P(U_B = 1) = 0.2 \\
            P(U_Y = 1) = 0.1
        \end{cases}
    \end{align}
    In words, people are more likely to take both doses if they are rich, and most people will only take the second dose if they have already taken the first dose. Also, people who take both doses are more likely to recover, but only if they came from a high socioeconomic background ($U_{RY} = 1$). The values computed from $\cM_L$ are shown in Figure \ref{fig:drug-ex-u-tab}.

    The rows of the figure can be used to compute quantities of the PCH from $\cM_L$. Denote $p_{i}$ as the probability of the $i$th row of the table. Then, for example, the quantity $P(Y = 1 \mid A = 1, B = 1)$, or the probability that someone recovers given that they took both doses of the drug is
    \begin{align*}
            &P^{\cM_L}(Y = 1 \mid A = 1, B = 1) \\
            &= \frac{P^{\cM_L}(Y = 1, A = 1, B = 1)}{P^{\cM_L}(A = 1, B = 1)} \\
            &= \frac{p_7 + p_8}{p_6 + p_7 + p_8 + p_9} \\
            &= \frac{0.002 + 0.288}{0.018 + 0.002 + 0.288 + 0.032} \\
            &\approx 0.853. \stepcounter{equation}\tag{\theequation}\label{eq:ex-q-tau-cons-l1q}
    \end{align*}

    The causal quantity $P(Y_{A = 1, B = 1} = 1)$, or the probability that someone recovers when forced to take both doses of the drug, can be computed as
    \begin{align*}
        &P^{\cM_L}(Y_{A = 1, B = 1} = 1) \\
        &= p_1 + p_3 + p_5 + p_7 + p_8 + p_{10} + p_{12} + p_{14}\\
        &= 0.032 + 0.008 + 0.008 + 0.002 \\
        &+ 0.288 + 0.072 + 0.072 + 0.018 \\
        &= 0.5. \stepcounter{equation}\tag{\theequation}\label{eq:ex-q-tau-cons-l2q}
    \end{align*}
    Indeed, one may be misled to think that the drug is extremely effective when only looking at the conditional quantity in Eq.~\ref{eq:ex-q-tau-cons-l1q}, as opposed to the causal effect, as in Eq.~\ref{eq:ex-q-tau-cons-l2q}. In reality, the causal effect of the drug is not as high.

    Suppose the researchers decide that this much detail in the study is unnecessary, and they consider working in a more abstract model. One way to simplify the model is to reduce the amount of variables. Perhaps they decide that wealth ($R$) is irrelevant and can be abstracted away, and the two doses ($A$ and $B$) can simply be abstracted into one variable, treatment ($X$). This can be represented using the intervariable clusters $\bbC = \{\*C_1 = \{A, B\},\*C_2 = \{Y\}\}$, where $A$ and $B$ are in the same cluster, $Y$ is in a separate cluster, and $R$ is not included.

    Further, the treatment $X$ is only considered complete if both doses are taken, so we can further perform an intravariable clustering, where the domains are:
    \begin{equation}
        \bbD_{\*C_1} =
        \begin{cases}
            x_0 =& \{(A = 0, B = 0), (A = 0, B = 1), \\
            & (A = 1, B = 0)\} \\
            x_1 =& \{(A = 1, B = 1)\}
        \end{cases}
    \end{equation}
    We denote $X$ as the higher level variable corresponding to $\*C_1$, and we define its domain $\cD_{X} = \{0, 1\}$ to be binary, where $x_0$ corresponds to $0$ and $x_1$ corresponds to $1$. We leave $Y$ as is in the lower level space.

    We can then define the constructive abstraction function $\tau$ based on $\bbC$ and $\bbD$, where $\*V_H = \{X, Y\}$. For example,
    \begin{equation}
    \tau(R = 1, A = 0, B = 1, Y = 1) = (X = 0, Y = 1),
    \end{equation}
    and
    \begin{equation}
    \tau(R = 1, A = 1, B = 1, Y = 0) = (X = 1, Y = 0).
    \end{equation}
    Now define $\cM_H$ over the newly defined $\*V_H = \tau(\*V_L)$ as follows.
    \begin{align}
        \*U_H &= \{U_X, U_{Y0}, U_{Y1}\} \\
        \*V_H &= \{X, Y\} \\
        \cF_H &=
        \begin{cases}
            X \gets f^H_X(u_X) = u_X \\
            Y \gets f^H_Y(x, u_{Y0}, u_{Y1}) = 
            \begin{cases}
                u_{Y0} & x = 0 \\
                u_{Y1} & x = 1
            \end{cases}
        \end{cases} \\
        P(\*U_H) &=
        \begin{cases}
            P(U_X = 1) = 0.34 \\
            P(U_{Y0} = 1) = 0.1 \\
            P(U_{Y1} = 1) = 0.852941
        \end{cases}
    \end{align}

    Interestingly, note that $P^{\cM_H}(Y = 1 \mid X = 1) = P(U_{Y1} = 1) \approx 0.853$, which is equal to $P^{\cM_L}(Y = 1 \mid A = 1, B = 1)$ computed in Eq.~\ref{eq:ex-q-tau-cons-l1q}. In fact, if $Q = P(Y = 1, A = 1, B = 1)$, then the corresponding $\tau(Q)$ from Def.~\ref{def:q-tau-consistency} is $P(Y = 1, X = 1)$, since $\tau(Y = 1, A = 1, B = 1) = (Y = 1, X = 1)$. Since they are equal, we would say that $\cM_H$ is $Q$-$\tau$ consistent with $\cM_L$.

    Now suppose $Q' = P(Y_{A = 1, B = 1} = 1)$. The corresponding $\tau(Q')$ would be $P(Y_{X = 1} = 1)$. However, note that $P^{\cM_H}(Y_{X = 1} = 1) = P(U_{Y1} = 1) \approx 0.853$, which is not equal to $P^{\cM_L}(Y_{A = 1, B = 1} = 1) = 0.5$ computed from Eq.~\ref{eq:ex-q-tau-cons-l2q}. Then, $\cM_H$ is not $Q'$-$\tau$ consistent with $\cM_L$.
    
    It turns out that $\cM_H$ is $Q$-$\tau$ consistent with $\cM_L$ for every $Q \in \cL_1$, making $\cM_H$ $\cL_1$-$\tau$ consistent with $\cM_L$. On the other hand, this is not the case for $\cL_2$, the interventional layer. In fact, it seems that $\cM_H$ is equating correlation with causation and fails to capture the nuances of interventions in $\cM_L$. Still, such a model could be useful if the queries of interest are on $\cL_1$. One could argue that $\cM_H$ is a suitable abstraction of $\cM_L$ on Layer 1, but not on Layer 2. The concept of $Q$-$\tau$ consistency allows us to define ``partial'' abstractions based on the specific quantities of the PCH that match.
    \hfill $\blacksquare$
\end{example}

It turns out that when $\cM_H$ is $Q$-$\tau$ consistent with $\cM_L$ on all three layers of the PCH (i.e.~$\cL_3$-$\tau$ consistent), then $\cM_H$ can be considered an abstraction of $\cM_L$ on the SCM-level, which coincides with the definition of constructive $\tau$-abstractions (Def.~\ref{def:cons-tau-abs} from App.~\ref{sec:related-work}) from \citet{beckers2019abstracting}, shown below.

\begin{restatable}[Abstraction Connection]{proposition}{absconnect}
    \label{prop:abs-connect}
    Let $\tau: \cD_{\*V_L} \rightarrow \cD_{\*V_H}$  be a constructive abstraction function (Def.~\ref{def:tau}). $\cM_H$ is $\cL_3$-$\tau$ consistent (Def.~\ref{def:q-tau-consistency}) with $\cM_L$ if and only if there exists SCMs $\cM_L'$ and $\cM_H'$ s.t.~$\cL_3(\cM_L') = \cL_3(\cM_L)$, $\cL_3(\cM_H') = \cL_3(\cM_H)$, and $\cM_H'$ is a constructive $\tau$-abstraction of $\cM'_L$.
    
    \hfill $\blacksquare$
\end{restatable}

All proofs are provided in Appendix \ref{app:proofs}. This proposition provides the connection between the abstractions defined in this work and established definitions from previous works\footnote{Note that one subtlety of this result is that it is not $\cM_H$ that is directly a constructive $\tau$-abstraction of $\cM_L$, but rather their $\cL_3$-equivalent counterparts, $\cM_H'$ and $\cM_L'$. Indeed, the definition of constructive $\tau$-abstractions is stronger than $\cL_3$-$\tau$ consistency (see proof for more details), but in tasks where we are only concerned with the layers of the PCH, this distinction is inconsequential.}.

\subsection{Algorithmic Abstraction Construction}

With the abstraction function $\tau$ defined, the notion of $Q$-$\tau$ consistency allows for comparisons of distributions between the low level model $\cM_L$ and the abstraction $\cM_H$. Still, it would be desirable to be able to systematically construct $\cM_H$ given $\cM_L$ and $\tau$ such that $\cM_H$ is $Q$-$\tau$ consistent with $\cM_L$ for as many queries $Q$ as possible. Moving in this direction, we first note that as a subtlety, for some cases of $\cM_L$, there are certain choices of $\bbC$ and $\bbD$ (and corresponding $\tau$) for which $Q$-$\tau$ consistency (for some queries $Q$) is impossible to achieve in any choice of $\cM_H$. This impossibility is illustrated in the following example, inspired by \citet{spirtes_cholesterol}.

\begin{example}
    \label{ex:cholesterol-inv-cond}
    Consider a study that aims to understand the effects of diet on heart disease. Having a poor diet ($X$) is known to cause heart disease ($Y$) depending on its cholesterol content. Cholesterol comes in two forms, called high-density and low-density lipoproteins (HDL and LDL, respectively). The HDL is believed to lower heart disease risk while LDL increases it \citep{steinberg2007, truswell2010}. Suppose the study is simplified to binary variables, and the true model $\cM_L$ is:
    \begin{align}
        \*U_L &= \{U_X, U_{C1}, U_{C2}, U_Y\} \\
        \*V_L &= \{X, HDL, LDL, Y\} \\
        \cF_L &=
        \begin{cases}
            X \gets f^L_X(u_X) = u_X \\
            HDL \gets f^L_{HDL}(x, u_{C1}) = x \oplus u_{C1} \\
            LDL \gets f^L_{LDL}(x, u_{C2}) = x \oplus u_{C2} \\
            Y \gets f^L_Y(hdl, ldl, u_Y) = (ldl \wedge \neg hdl) \oplus u_Y 
        \end{cases} \label{eq:ex-cholesterol-F} \\
        P(\*U_L) &=
        \begin{cases}
            P(U_X = 1) = 0.5 \\
            P(U_{C1} = 1) = 0.1 \\
            P(U_{C2} = 1) = 0.1 \\
            P(U_Y = 1) = 0.1
        \end{cases}
    \end{align}
    As $\cM_L$ indicates, a person is more likely to get heart disease if their diet consists of high LDL levels but low HDL levels. For example, note that $P^{\cM_L}(Y_{LDL = 1, HDL = 0} = 1) = 0.9$ while $P^{\cM_L}(Y_{LDL = 0, HDL = 1} = 1) = 0.1$.
    
    Now, suppose a data scientist decides to abstract HDL and LDL together into a variable called ``total cholesterol'' (TC). Say that TC is defined as
    \begin{equation}
        TC = HDL + LDL.
    \end{equation}
    In fact, this leads to a choice of intervariable clusters
    \begin{equation}
        \bbC = \{\*C_1 = \{X\}, \*C_2 = \{HDL, LDL\}, \*C_3 = \{Y\}\},
    \end{equation}
    and then for intravariable clusters, they would choose
    \begin{equation}
    \bbD_{\*C_2} = 
    \begin{cases}
        tc_0 &= \{(HDL = 0, LDL = 0)\} \\
        tc_1 &= \{(HDL = 0, LDL = 1), \\
        & (HDL = 1, LDL = 0)\} \\
        tc_2 &= \{(HDL = 1, LDL = 1)\}.
    \end{cases}
    \end{equation}For the other clusters, simply use the same variables. Let $\tau$ be the constructive abstraction function defined with this choice of $\bbC$ and $\bbD$ (i.e.~$\tau_{\*C_2}(hdl, ldl) = hdl + ldl$).

    An issue arises due to the grouping of values $(HDL = 0, LDL = 1)$ and $(HDL = 1, LDL = 0)$ into the same intravariable cluster. To witness, note that $\tau_{\*C_1}(HDL = 0, LDL = 1) = \tau_{\*C_2}(HDL = 1, LDL = 0) = (TC = 1)$. Now, consider two queries $Q_1 = P(Y_{HDL = 0, LDL = 1} = 1)$ and $Q_2 = P(Y_{HDL = 1, LDL = 0} = 1)$, and observe that
    \begin{align}
        P^{\cM_L}(Y_{HDL = 0, LDL = 1} = 1) &= P(U_Y = 0) = 0.9, \\
        P^{\cM_L}(Y_{HDL = 1, LDL = 0} = 1) &= P(U_Y = 1) = 0.1.
    \end{align}
    However, since $\tau_{\*C_1}(HDL = 0, LDL = 1) = \tau_{\*C_2}(HDL = 1, LDL = 0) = (TC = 1)$, both $Q_1$ and $Q_2$ have the same high-level counterpart. That is, $\tau(Q_1) = \tau(Q_2) = P(Y_{TC = 1} = 1)$. No choice of $\cM_H$ over $\*V_H$ can be both $Q_1$-$\tau$ consistent and $Q_2$-$\tau$ consistent with $\cM_L$ because $P^{\cM_H}(Y_{TC = 1} = 1)$ cannot both be equal to $0.9$ and $0.1$.
    \hfill $\blacksquare$
\end{example}

Intuitively, Ex.~\ref{ex:cholesterol-inv-cond} shows two values that cannot be grouped into the same intravariable cluster because the function $f^L_Y$ (from Eq.~\ref{eq:ex-cholesterol-F}) produces different results depending on which value is used. Grouping the two values in the same cluster would imply that the two values are ``equivalent'' and hence, there would be an inevitable loss of information. Indeed, real-world studies that consider total cholesterol instead of separating it into LDL and HDL often have conflicting results, indicating an invalid abstraction. This phenomenon can be described formally through the following condition.

\begin{restatable}[Abstract Invariance Condition (AIC)]{definition}{aicdef}
%\begin{definition}[Abstract Invariance Condition (AIC)]
    \label{def:invariance-condition}
    Let $\cM_L = \langle \*U_L, \*V_L, \cF_L, P(\*U_L) \rangle$ be an SCM and $\tau: \cD_{\*V_L} \rightarrow \cD_{\*V_H}$ be a constructive abstraction function relative to $\bbC$ and $\bbD$. The SCM $\cM_L$ is said to satisfy the abstract invariance condition (AIC, for short) with respect to $\tau$ if, for all $\*v_1, \*v_2 \in \cD_{\*V_L}$ such that $\tau(\*v_1) = \tau(\*v_2)$, $\forall \*u \in \cD_{\*U_L}, \*C_i \in \bbC$, the following holds:
    % \begin{equation}
    %         \tau \left(\bigcup_{V \in \*C_i} \hspace{-1.5mm} f^L_V(\pai{V}^{(1)}, \*u_V)\right) = \tau \left(\bigcup_{V \in \*C_i} \hspace{-1.5mm} f^L_V(\pai{V}^{(2)}, \*u_V)\right)
    % \end{equation}
    \begin{equation}
        \label{eq:aic}
        \begin{split}
            & \tau_{\*C_i} \left( \left ( f^L_V(\pai{V}^{(1)}, \*u_V): V \in \*C_i \right ) \right) \\
            &= \tau_{\*C_i} \left( \left ( f^L_V(\pai{V}^{(2)}, \*u_V): V \in \*C_i \right ) \right),
        \end{split}
    \end{equation}
    where $\pai{V}^{(1)}$ and $\pai{V}^{(2)}$ are the values corresponding to $\*v_1$ and $\*v_2$. Then, $\tildepai{V}$ is used to denote any arbitrary value s.t.~$\tau(\tildepai{V}) = \tau(\pai{V}^{(1)}) = \tau(\pai{V}^{(2)})$.
    \hfill $\blacksquare$
%\end{definition}
\end{restatable}

In words, the AIC enforces that if two low level values $\*v_1, \*v_2 \in \cD_{\*V_L}$ map to the same high level value (i.e.~$\tau(\*v_1) = \tau(\*v_2)$), then for each cluster $\*C_i \in \bbC$, the functions of those clusters should map to the same value regardless of $\*U_L$ (i.e.~the outputs of $f_{V}^L(\pai{V}^{(1)}, \*u_V)$ for each $V \in \*C_i$ should map to the same result as the outputs of $f_{V}^L(\pai{V}^{(2)}, \*u_V)$ when passed through $\tau_{\*C_i}$). Intuitively, this implies that two values in the same intravariable cluster have the same functional effect in the higher level setting.

In Ex.~\ref{ex:cholesterol-inv-cond}, the AIC is not satisfied since $(HDL = 0, LDL = 1)$ and $(HDL = 1, LDL = 0)$ cannot be grouped into the same intravariable cluster. As established, $\tau_{\*C_1}(HDL = 0, LDL = 1) = \tau_{\*C_2}(HDL = 1, LDL = 0) = 1$. However, observing $f^L_Y$ (from Eq.~\ref{eq:ex-cholesterol-F}) given these inputs, note that $f^L_Y(HDL = 0, LDL = 1, u_Y) = \neg u_Y$, while $f^L_Y(HDL = 1, LDL = 0, u_Y) = u_Y$, which give opposite results for either choice of $u_Y$.

In contrast, we consider a different abstraction that does satisfy the AIC next.

\begin{example}[Example \ref{ex:cholesterol-inv-cond} continued]
    \label{ex:inv-cond-satisfied}
    Consider a different choice of clusters that does not violate the AIC. Although somewhat unintuitive, we can actually abstract the cholesterol values in the given $\cM_L$ by taking their difference instead of their sum!\footnote{This is true given the simple definition of $\cM_L$ in this example. In more complex descriptions, such as in cases where the variables are continuous, the clusters would have to be chosen more carefully to avoid violations of the AIC.} For example, define $Z = LDL - HDL$. In terms of clusters we would keep the same intervariable clusters $\bbC = \{\*C_1 = \{X\}, \*C_2 = \{HDL, LDL\}, \*C_3 = \{Y\}\}$, but for intravariable clusters, we choose $\bbD_{\*C_2} = \{z_{-1} = \{(HDL = 1, LDL = 0)\}, z_0 = \{(HDL = 0, LDL = 0), (HDL = 1, LDL = 1)\}, z_1 = \{(HDL = 0, LDL = 1)\}\}$. Define $\tau$ as the constructive abstraction function defined over $\bbC$ and $\bbD$ (i.e.~$\tau_{\*C_2}(hdl, ldl) = ldl - hdl$).

    The difference now is that instead of clustering together the values $(HDL = 1, LDL = 0)$ and $(HDL = 0, LDL = 1)$, it is the values $(HDL = 0, LDL = 0)$ and $(HDL = 1, LDL = 1)$ that are clustered together. When looking at $f^L_Y$ (Eq.~\ref{eq:ex-cholesterol-F}) under these two values, we see that $f^L_Y(HDL = 0, LDL = 0, u_Y) = f^L_Y(HDL = 1, LDL = 1, u_Y) = u_Y$ for any choice of $u_Y$, satisfying the AIC.

    Intuitively, $f_Y$ no longer changes behavior between these two values, or in other words, $f_Y$ is \emph{invariant} between these two values. With no other downstream variables to consider, this implies that these two values are functionally identical in the model and can be abstracted together into a single value without loss of information.
    \hfill $\blacksquare$
\end{example}

It turns out that the AIC describes precisely when an appropriate $\cM_H$ exists as an abstraction of the low level model $\cM_L$, as shown by the following result.

\begin{restatable}[Abstraction Conditions]{proposition}{absconditions}
    \label{prop:inv-cond-completeness}
    For any SCM $\cM_L$ and constructive abstraction function $\tau$ relative to $\bbC$ and $\bbD$, there exists an SCM $\cM_H$ over variables $\*V_H = \tau(\*V_L)$ such that $\cM_H$ is $\cL_3$-$\tau$ consistent with $\cM_L$ if and only if there exists $\cM_L'$ such that $\cL_3(\cM_L) = \cL_3(\cM_L')$ and $\cM_L'$ satisfies the abstract invariance condition with respect to $\tau$.
    \hfill $\blacksquare$
\end{restatable}

This critical property guarantees the existence of a higher level SCM $\cM_H$ such that $\cL_3$-$\tau$ consistency holds, so we will assume that the AIC holds for the rest of this work. Still, see App.~\ref{app:invariance-condition} for further discussion on its implications and for possible relaxations in cases where $\cL_3$-$\tau$ consistency is not required.

%\begin{wrapfigure}{r}{0.5\textwidth}
%\vspace{-0.2in}
%\hspace{-0.05cm}
%\IncMargin{1em}
%\begin{algorithm}[H]
\begin{algorithm}[t]
    %\scriptsize
    %\setstretch{0.9}
    %\renewcommand{\AlCapSty}[1]{\normalfont\scriptsize{\textbf{#1}}\unskip}
    
    \DontPrintSemicolon
    \SetKwFunction{absfunc}{AbsFunc}
    \SetKwInOut{Input}{Input}
    \SetKwInOut{Output}{Output}
    
    \Input{ SCM $\cM_L = \langle \*U_L, \*V_L, \cF_L, P(\*U_L) \rangle$, admissible inter/intravariable clusters $\bbC$ and $\bbD$ satisfying abstract invariance condition}
    \Output{ SCM $\cM_H$ and $\tau: \cD_{\*V_H} \rightarrow \cD_{\*V_L}$ s.t.~$\cM_H$ is $\cL_3$-$\tau$ consistent with $\cM_L$}
    \BlankLine
    \textls[-20]{
    $\*U_H \gets \*U_L, P(\*U_H) \gets P(\*U_L)$\;
    $\*V_H \gets \bbC, \cD_{\*V_H} \gets \bbD$\;
    $\tau \gets \absfunc(\bbC, \bbD)$ \tcp*{from Def.\ \ref{def:tau}}
    \For{$\*C_i \in \bbC$}{
        $f^H_i \gets \tau \left( f^L_V(\tildepai{V}, \*u_V) : V \in \*C_i \right)$\;
    }
    $\cF_H \gets \{f^H_i : \*C_i \in \bbC\}$\;
    \Return $\tau$, $\cM_H = \langle \*U_H, \*V_H, \cF_H, P(\*U_H) \rangle$\;
    }
    %\caption{\scriptsize Obtaining an abstraction from a low-level SCM.}
    \caption{Constructing $\cM_H$ from $\cM_L$.}
    \label{alg:map-abstraction}
\end{algorithm}
%\DecMargin{1em}
%\vspace{-0.3in}
%\end{wrapfigure}

With the notion of abstractions well-defined, we study how $\cM_H$ can be obtained from $\cM_L$. Interestingly, when given the admissible clusterings $\bbC$ and $\bbD$, the procedure for recovering $\tau$ and converting $\cM_L$ to $\cM_H$ can be done as shown in Alg.~\ref{alg:map-abstraction}. Intuitively, one can obtain an abstraction $\cM_H$ of $\cM_L$ by first constructing the abstraction function $\tau$ using the clusterings $\bbC$ and $\bbD$ (lines 2-3), followed by designing the functions of $\cM_H$ to wrap the original functions of $\cM_L$ with $\tau$ (lines 4-6). This can be verified using the following result.

\begin{restatable}{proposition}{mapabstractioncorrect}
\label{prop:map-abstraction}
Let $\tau$ and $\cM_H$ be the function and SCM obtained from running Alg.~\ref{alg:map-abstraction} on inputs $\cM_L$, $\bbC$, and $\bbD$. Then, $\cM_H$ is $\cL_3$-$\tau$ consistent with $\cM_L$.
\hfill $\blacksquare$
\end{restatable}

See below for an example of running Alg.~\ref{alg:map-abstraction}

\begin{example}[Example \ref{ex:drug-tau} continued]
    \label{ex:drug-alg-abstraction}
    We will run Alg.~\ref{alg:map-abstraction} on $\cM_L$ and clusters $\bbC$ and $\bbD$ described earlier in Ex.~\ref{ex:drug-tau}. Following the algorithm, we first set $\*U_H = \*U_L$ and $P(\*U_H) = P(\*U_L)$. We construct $\tau$ via Def.~\ref{def:tau}, as shown in the earlier example. Then, we can compute the function $f^H_X$ as follows.
    \begin{align*}
        & f^H_X(u_{RY}, u_A, u_B) \\
        &= \tau_{\*C_1}(f^L_A(r, u_A), f^L_B(r, a, u_B)) \\
        &= \tau_{\*C_1}(r \oplus u_A, (r \wedge (r \oplus u_A)) \oplus u_B) \\
        &= \tau_{\*C_1}(u_{RY} \oplus u_A, (u_{RY} \wedge (u_{RY} \oplus u_A)) \oplus u_B) \stepcounter{equation}\tag{\theequation}
    \end{align*}
    For $f^H_Y$, denote $\tilde{a}$ and $\tilde{b}$ as an arbitrary setting of $A$ and $B$ such that $\tau(\tilde{a}, \tilde{b}) = x$, as indicated in line 5 of the algorithm.
    \begin{align*}
        & f^H_Y(x, u_{RY}, u_Y) \\
        &= \tau_{\*C_2}(f^L_Y(\tilde{a}, \tilde{b}, u_{RY}, u_Y)) \\
        &= \tau_{\*C_2}(((\tilde{a} \wedge \tilde{b}) \wedge u_{RY}) \oplus u_Y) \\
        &= \tau_{\*C_2}((x \wedge u_{RY}) \oplus u_Y) \stepcounter{equation}\tag{\theequation}
    \end{align*}
    
    Putting everything together, we obtain $\cM_H$ as follows.
    \begin{align}
        \*U_H &= \*U_L = \{U_{RY}, U_A, U_B, U_{Y}\} \\
        \*V_H &= \{X, Y\} \\
        \cF_H &= \{ \\
            & X \gets f^H_X(u_{RY}, u_A, u_B) \nonumber \\
            &= \tau_{\*C_1}(u_{RY} \oplus u_A, (u_{RY} \wedge (u_{RY} \oplus u_A)) \oplus u_B) \nonumber \\
            & Y \gets f^H_Y(x, u_{RY}, u_Y) = \tau_{\*C_2}((x \wedge u_{RY}) \oplus u_Y) \nonumber \\
        P(\*U_H) &= P(\*U_L)
    \end{align}

    It is not difficult to see that $\cM_H$ is $\cL_3$-$\tau$ consistent with $\cM_L$. As an example, note that $P^{\cM_H}(Y_{X = 1} = 1) = P(U_{RY} = 1, U_Y = 0) + P(U_{RY} = 0, U_Y = 1) = (0.5)(0.9) + (0.5)(0.1) = 0.5$, which matches $P^{\cM_L}(Y_{A = 1, B = 1} = 1)$ from Eq.~\ref{eq:ex-q-tau-cons-l2q}.
    \hfill $\blacksquare$
\end{example}

Ex.~\ref{ex:drug-alg-abstraction} shows how Alg.~\ref{alg:map-abstraction} can be used to systematically obtain an abstraction $\cM_H$ of the low-level model $\cM_L$, so long as $\cM_L$ is provided alongside the clusters $\bbC$ and $\bbD$. Since $\cM_L$ is almost never available in practice, the following sections show how this requirement can be relaxed.

\section{Inferences Across Abstractions}
\label{sec:learning-abs}

As demonstrated by Alg.~\ref{alg:map-abstraction}, converting a low level model $\cM_L$ to a high level model $\cM_H$ is somewhat immediate when given full observability of the underlying SCM $\cM_L$. However, in real applications, it is rarely the case that the full specification of $\cM_L$ is known. Typically, one will only be given partial information of $\cM_L$ in the form of data, such as samples of the observational distribution $P(\*V_L)$. The question we investigate in this section is: is it still possible to ``learn'' some $\cM_H$ given the observed data?

%\begin{wrapfigure}{r}{0.4\textwidth}
\begin{figure}
%\vspace{-0.1in}
%\hspace{-0.05cm}
\centering
\includegraphics[width=\linewidth]{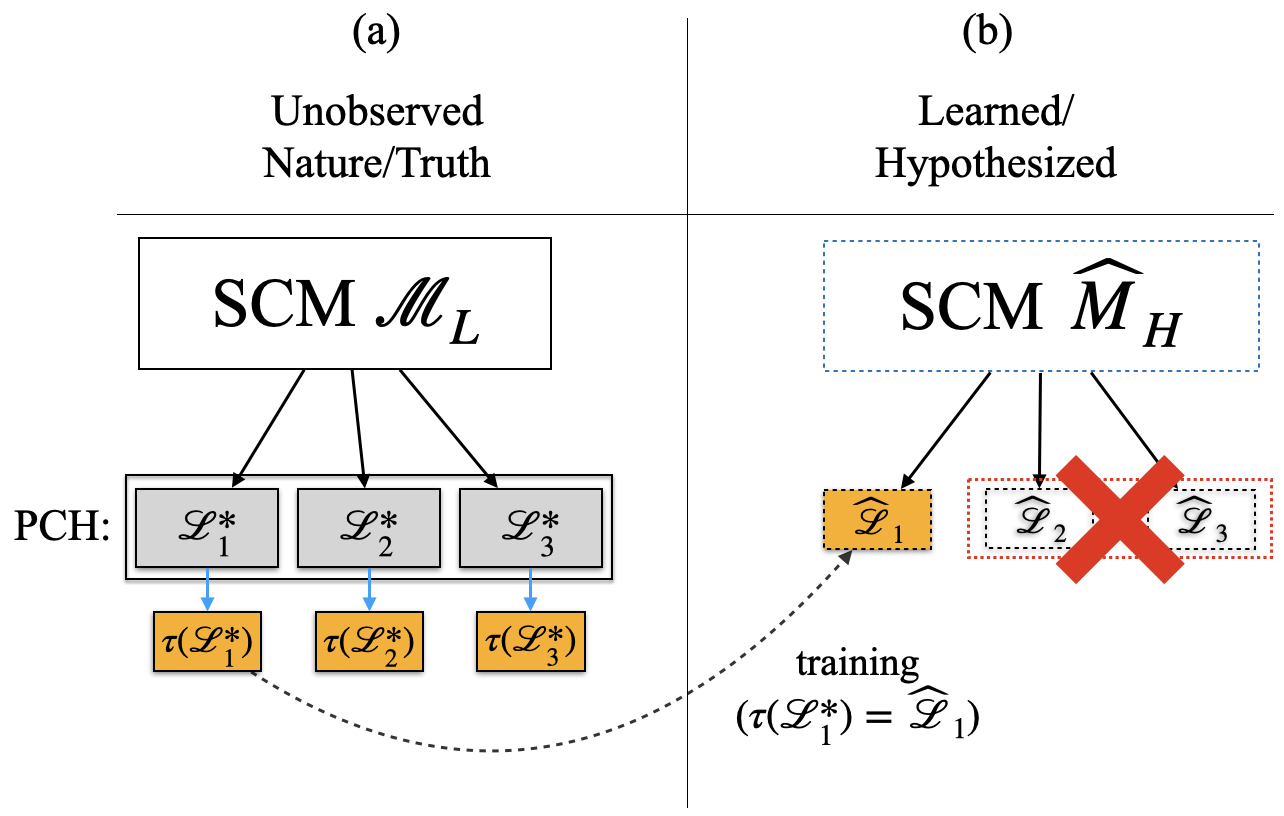}
\caption{Illustration of the Abstract CHT. Without additional information, a high-level model $\hM_H$ trained to be $\cL_1$-$\tau$ consistent with $\cM_L$ is not guaranteed to be $\cL_2$ or $\cL_3$-$\tau$ consistent.}
\label{fig:abstract-cht}
\end{figure}
%\vspace{-0.4in}
%\end{wrapfigure}

We first note the impossibility result described by the Causal Hierarchy Theorem (CHT) \citep[Thm.~1]{bareinboim:etal20}, which states that a model trained to match another SCM on lower layers of the causal hierarchy (e.g. $\cL_1$) will likely not match on higher layers (e.g. $\cL_2$ or $\cL_3$). Naturally, the same is true when it comes to inferring causal quantities across abstractions. One may be tempted to believe that $\cM_H$ can be learned given $\cL_1$ data from $\cM_L$ by instantiating some expressive parametric model $\hM_H$ on $\*V_H$, and then training $\widehat{M}_H$ on $P(\*V_H) = P(\tau(\*V_L))$ such that $\widehat{M}_H$ is $\cL_1$-$\tau$ consistent with $\cM_L$. Unfortunately, this strategy will fail in general since even under perfect training, $\hM_H$ is not guaranteed to be $\cL_2$-$\tau$ (or $\cL_3$-$\tau$) consistent with $\cM_L$. This means that any causal quantities induced by $\hM_H$ will likely bear no relationship with causal quantities induced by $\cM_L$. This phenomenon is described by the following proposition.

\begin{proposition}[Abstract Causal Hierarchy Theorem (Informal)]
    \label{prop:abs-cht}
    Given constructive abstraction function $\tau: \cD_{\*V_H} \rightarrow \cD_{\*V_L}$, even if $\cM_H$ is $\cL_i$-$\tau$ consistent with $\cM_L$, $\cM_H$ will almost never be $\cL_j$-$\tau$ consistent with $\cM_L$ for $j > i$.
    \hfill $\blacksquare$
\end{proposition}

In words, matching across abstractions on lower layers does not guarantee the same will hold for higher layers. This idea is illustrated in Fig.~\ref{fig:abstract-cht}. The left of the figure shows the unobserved true SCM $\cM_L$, which induces distributions from the three layers of the PCH. Observational data from $\cL_1$ is provided, and one may train a high-level model $\hM_H$ (right) such that it is $\cL_1$-$\tau$ consistent with $\cM_L$. However, even if training is perfect, it is not guaranteed that $\hM_H$ is $\cL_2$ or $\cL_3$-$\tau$ consistent with $\cM_L$. That is, queries from the $\cL_2$ or $\cL_3$ distributions of $\hM_H$ are not expected to match the equivalent queries in $\cM_L$ in general. (See further details in App.~\ref{app:proofs}.) 

The consequence of this result is that causal assumptions will be necessary to make progress. In particular, the class of plausible models must be constrained through assumptions about the generating model. Given this necessity, one type of assumption prevalent throughout causal inference literature is the availability of a causal diagram \citep{pearl:95a}, a graphical structure that qualitatively describes the functional relationships between variables in a non-parametric manner. This assumption is weaker than having the entire generating SCM, since it only encodes qualitative information of the functional dependences between exogenous and endogenous variables (as in Def.~\ref{def:cg}) rather than full detail of the generating mechanisms and exogenous distributions. Still, it has been shown that having the causal diagram allows certain inferences across layers, determined through the causal identification problem \citep{pearl:2k, bareinboim:pea16}.

In the context of abstractions however, specifying the causal diagram for the true model $\cM_L$ requires describing the relationships between every low-level variable in $\*V_L$. This may be unrealistic in many practical settings since there are typically too many low-level variables (e.g.~$128 \times 128$ pixels in an image) to expect a description of the relationship between every pair, and many of these relationships may not be well-defined in a causal manner. Instead, it may be more reasonable to specify a causal diagram over $\*V_H$ (or intervariable clusters $\bbC$). The amount of information required is reduced when $|\*V_H| \ll |\*V_L|$, and the causal relationships between variables may be more clear given that the higher-level variables tend to be more explainable. The causal diagram over $\*V_H$ can be viewed as a graphical abstraction of the causal diagram over $\*V_L$. The relationship can be formalized through the concept of cluster causal diagrams (C-DAGs) \citep{anand:etal23}, as described next.

\begin{figure}
    \centering
    \begin{subfigure}[c]{0.49\linewidth}
        \centering
        \begin{tikzpicture}[xscale=1.4, yscale=0.7]
            \tiny
            \node[draw, circle] (R) at (0, 1) {$R$};
            \node[draw, circle] (F) at (-1, 0) {$D$};
            \node[draw, circle] (N1) at (0, 0.2) {$C$};
            \node[draw, circle] (N2) at (0, -0.5) {$F$};
            \node[draw, circle] (N3) at (0, -1.2) {$P$};
            \node[draw, circle] (B) at (1, 0) {$B$};

            \node[draw,densely dotted,color=blue,rounded corners=1mm, fit=(F), line width=0.8pt, inner sep=1.2mm] {};
            \node[draw,densely dotted,color=blue,rounded corners=1mm, fit=(B), line width=0.8pt, inner sep=1.2mm] {};
            \node[draw,densely dotted,color=blue,rounded corners=1mm, fit=(N1) (N2) (N3), line width=0.8pt, inner sep=1.2mm] {};
            
            \path [-{Latex[length=1mm]}] (R) edge (F);
            \path [-{Latex[length=1mm]}] (F) edge (N1);
            \path [-{Latex[length=1mm]}] (F) edge (N2);
            \path [-{Latex[length=1mm]}] (F) edge (N3);
            \path [-{Latex[length=1mm]}] (N1) edge (B);
            \path [-{Latex[length=1mm]}] (N2) edge (B);
            \path [-{Latex[length=1mm]}] (N3) edge (B);
            \path [{Latex[length=1mm]}-{Latex[length=1mm]}, dashed, bend left] (R) edge (B);
            \path [{Latex[length=1mm]}-{Latex[length=1mm]}, dashed, bend left, out=45, in=135] (N1) edge (N2);
            \path [{Latex[length=1mm]}-{Latex[length=1mm]}, dashed, bend left, out=45, in=135] (N1) edge (N3);
            \path [{Latex[length=1mm]}-{Latex[length=1mm]}, dashed, bend left, out=45, in=135] (N2) edge (N3);
        \end{tikzpicture}
        %\caption{$\cG$ for Nutrition. Clusters $\bbC$ outlined in blue.}
        %\label{fig:exp1-graph}
    \end{subfigure}%
    \hfill
    \begin{subfigure}[c]{0.49\linewidth}
        \centering
        \begin{tikzpicture}[xscale=1, yscale=1.4]
            \tiny
            \node[draw, circle] (F) at (-1, 0) {$D_H$};
            \node[draw, circle] (N) at (0, 0) {$Z$};
            \node[draw, circle] (B) at (1, 0) {$B_H$};
            
            \path [-{Latex[length=1mm]}] (F) edge (N);
            \path [-{Latex[length=1mm]}] (N) edge (B);
            \path [{Latex[length=1mm]}-{Latex[length=1mm]}, dashed, bend left] (F) edge (B);
        \end{tikzpicture}
        %\caption{$\cG_{\bbC}$ for Nutrition.}
        %\label{fig:exp1-graph-tau}
    \end{subfigure}
    \caption{The causal diagram $\cG$ over variables $\*V_L$ for the nutrition study in Ex.~\ref{ex:bmi} is on the left. Clusters $\bbC = \{D_H = \{D\}, Z = \{C, F, P\}, B_H = \{B\}\}$ are outlined in blue. The corresponding C-DAG $\cG_{\bbC}$ is on the right.}
    \label{fig:cdag-examples}
\end{figure}
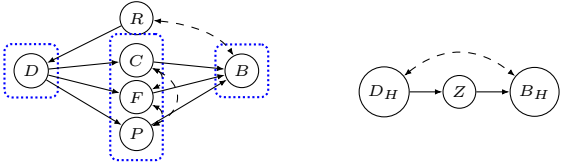

\begin{definition}[Cluster Causal Diagram (C-DAG) {\citep[Def.~1]{anand:etal23}}]
    \label{def:cdag}
    Given a causal diagram $\cG = \langle \*V, \*E \rangle$ and an admissible clustering $\bbC = \{\*C_1, \dots, \*C_k\}$ of $\*V$, construct a graph $\cG_{\bbC} = \langle \bbC, \*E_{\bbC}\rangle$ over $\bbC$ with a set of edges $\*E_{\bbC}$ defined as follows:
    \begin{enumerate}
        \item A directed edge $\*C_i \rightarrow \*C_j$ is in $\*E_{\bbC}$ if there exists some $V_i \in \*C_i$ and $V_j \in \*C_j$ such that $V_i \rightarrow V_j$ is an edge in $\*E$.

        \item A dashed bidirected edge $\*C_i \leftrightarrow \*C_j$ is in $\*E_{\bbC}$ if there exists some $V_i \in \*C_i$ and $V_j \in \*C_j$ such that $V_i \leftrightarrow V_j$ is an edge in $\*E$.
        \hfill $\blacksquare$
    \end{enumerate}
\end{definition}

In words, the nodes of the C-DAG $\cG_{\bbC}$ simply correspond to the clusters of $\bbC$, and edges connect clusters $\*C_i$ and $\*C_j$ if they connect some $V_i \in \*C_i$ and $V_j \in \*C_j$ in the original causal diagram $\cG$. Interestingly, the C-DAG definition aligns with the concept of intervariable clusters, providing a way for encoding constraints in the smaller space of $\*V_H$. Revisiting the nutrition study in Ex.~\ref{ex:bmi}, Fig.~\ref{fig:cdag-examples} shows the corresponding causal diagram $\cG$ (left) and the simpler C-DAG $\cG_{\bbC}$ (right). With the constraints of $\cG_{\bbC}$, we now introduce a notion of identification across abstractions to determine precisely which queries can be inferred.

\begin{definition}[Abstract Identification]
    \label{def:abs-id}
    Let $\tau: \cD_{\*V_H} \rightarrow \cD_{\*V_L}$ be a constructive abstraction function. Consider C-DAG $\cG_{\bbC}$, and let $\bbZ = \{P(\*V_{L[\*z_k]})\}_{k=1}^{\ell}$ be a collection of available interventional (or observational if $\*Z_k = \emptyset$) distributions over $\*V_L$. Let $\Omega_L$ and $\Omega_H$ be the space of SCMs defined over $\*V_L$ and $\*V_H$, respectively, and let $\Omega_L(\cG_{\bbC})$ and $\Omega_H(\cG_{\bbC})$ be their corresponding subsets that induce C-DAG $\cG_{\bbC}$. A query $Q$ is said to be $\tau$-ID from $\cG_{\bbC}$ and $\bbZ$ iff for every $\cM_L \in \Omega_L(\cG_{\bbC}), \cM_H \in \Omega_H(\cG_{\bbC})$ such that $\cM_H$ is $\bbZ$-$\tau$ consistent with $\cM_L$, $\cM_H$ is also $Q$-$\tau$ consistent with $\cM_L$.
    \hfill $\blacksquare$
\end{definition}

\begin{figure}
    \centering
    \begin{subfigure}{\linewidth}
        \centering
        \includegraphics[width=0.68\linewidth]{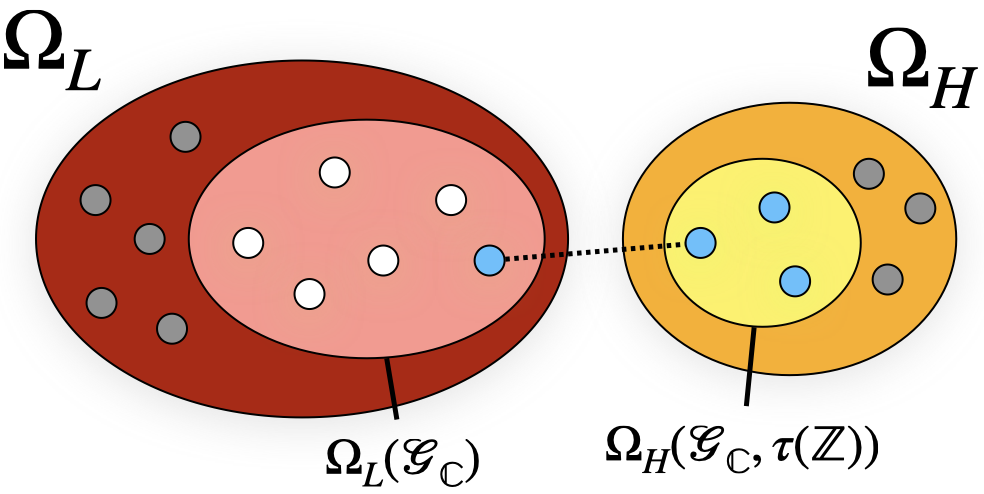}
        \caption{When $Q$ is $\tau$-ID from $\cG_{\bbC}$ and $\bbZ$, all SCMs over $\*V_H$ that induce $\cG_{\bbC}$ and is $\bbZ$-$\tau$ consistent with $\cM_L$ are also $Q$-$\tau$ consistent with $\cM_L$ for any choice of $\cM_L$.}
        \label{fig:tau-id-id}
    \end{subfigure}
    \begin{subfigure}{\linewidth}
        \centering
        \includegraphics[width=0.68\linewidth]{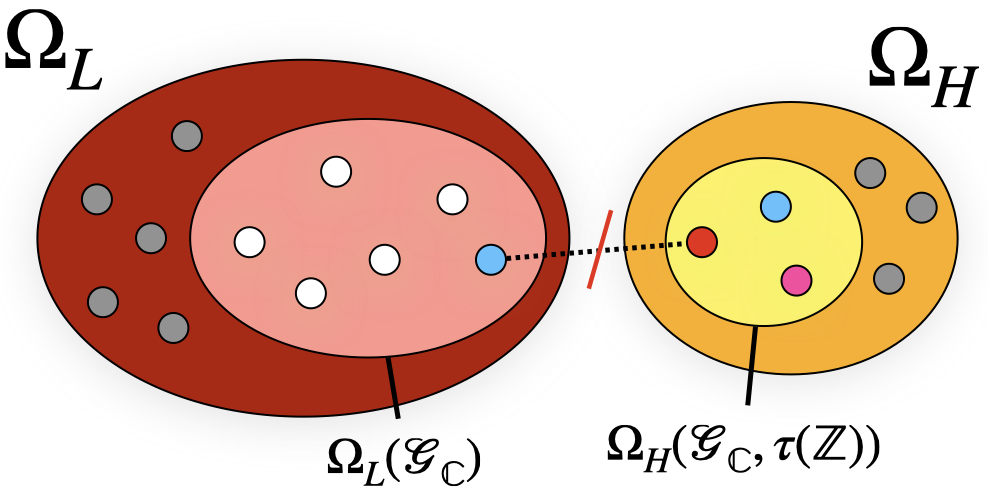}
        \caption{When $Q$ is not $\tau$-ID from $\cG_{\bbC}$ and $\bbZ$, then there exist choices of $\cM_L$ such that some SCM $\cM_H$ over $\*V_H$ that induces $\cG_{\bbC}$ and is $\bbZ$-$\tau$ consistent with $\cM_L$ is not $Q$-$\tau$ consistent with $\cM_L$.}
        \label{fig:tau-id-nonid}
    \end{subfigure}
    \caption{Examples of $\tau$-ID and $\tau$-nonID cases. The space of models over $\*V_L$, $\Omega_L$, is shown in dark red (left), and the subspace that induces C-DAG $\cG_{\bbC}$ is shown in light red. The blue dot within this space is an arbitrary choice of $\cM_L$. The space of models over $\*V_H$, $\Omega_H$, is shown in dark yellow (right), and the subspace that induces the causal graph $\cG_{\bbC}$ and is also $\bbZ$-$\tau$ consistent with $\cM_L$ is shown in light yellow.
    %\vspace{-0.5cm}
    }
    \label{fig:tau-id}
\end{figure}

\begin{figure*}
    \begin{minipage}[c]{0.28\linewidth}
        \centering
        \begin{tabular}{llll|l}
        \hline \hline
        $R$ & $A$ & $B$ & $Y$ & $P$   \\ \hline \hline
        0   & 0   & 0   & 0   & 0.288 \\ \hline
        0   & 0   & 0   & 1   & 0.032 \\ \hline
        0   & 0   & 1   & 0   & 0.072 \\ \hline
        0   & 0   & 1   & 1   & 0.008 \\ \hline
        0   & 1   & 0   & 0   & 0.072 \\ \hline
        0   & 1   & 0   & 1   & 0.008 \\ \hline
        0   & 1   & 1   & 0   & 0.018 \\ \hline
        0   & 1   & 1   & 1   & 0.002 \\ \hline
        1   & 0   & 0   & 0   & 0.072 \\ \hline
        1   & 0   & 0   & 1   & 0.008 \\ \hline
        1   & 0   & 1   & 0   & 0.018 \\ \hline
        1   & 0   & 1   & 1   & 0.002 \\ \hline
        1   & 1   & 0   & 0   & 0.072 \\ \hline
        1   & 1   & 0   & 1   & 0.008 \\ \hline
        1   & 1   & 1   & 0   & 0.032 \\ \hline
        1   & 1   & 1   & 1   & 0.288 \\ \hline
        \end{tabular}
        \subcaption{$P^{\cM_L}(\*V_L)$.}
        \label{fig:drug-ex-pv}
    \end{minipage}
    \hfill
    \begin{minipage}[c]{0.71\linewidth}
        \begin{subfigure}{.6\linewidth}
            \centering
            \begin{tikzpicture}[xscale=1.5, yscale=1]
                \node[draw, circle] (R) at (0, 1.5) {$R$};
                \node[draw, circle] (A) at (-1, 0) {$A$};
                \node[draw, circle] (B) at (-1, -1) {$B$};
                \node[draw, circle] (Y) at (1, 0) {$Y$};
    
                \node[draw,densely dotted,color=blue,rounded corners=2mm, fit=(A) (B), line width=1pt, inner sep=1.2mm] {};
                \node[draw,densely dotted,color=blue,rounded corners=2mm, fit=(Y), line width=1pt, inner sep=1.2mm] {};
                
                \path [-{Latex[length=2mm]}] (R) edge (A);
                \path [-{Latex[length=2mm]}] (R) edge (B);
                \path [-{Latex[length=2mm]}] (A) edge (B);
                \path [-{Latex[length=2mm]}] (A) edge (Y);
                \path [-{Latex[length=2mm]}] (B) edge (Y);
                \path [{Latex[length=2mm]}-{Latex[length=2mm]}, dashed, bend left] (R) edge (Y);
            \end{tikzpicture}
            \subcaption{Causal diagram $\cG$. Intervariable clusters $\bbC$ are grouped by the dashed-blue lines.}
            \label{fig:ex-drug-dag}
        \end{subfigure}
        \hfill
        \begin{subfigure}{.39\linewidth}
            \centering
            \begin{tikzpicture}[xscale=1, yscale=1.5]
                \node[draw, circle] (X) at (-1, 0) {$X$};
                \node[draw, circle] (Y) at (1, 0) {$Y$};
                
                \path [-{Latex[length=2mm]}] (X) edge (Y);
                \path [{Latex[length=2mm]}-{Latex[length=2mm]}, dashed, bend left] (X) edge (Y);
            \end{tikzpicture}
            \subcaption{C-DAG $\cG_{\bbC}$.}
            \label{fig:ex-drug-cdag}
        \end{subfigure}
        %\hfill
        %\vspace{1cm}
        \begin{subfigure}{\linewidth}
            \centering
            \vspace{0.4cm}
            \begin{tabular}{l|l|l}
            \hline \hline
            $\*V_H$ term      & Corresponding $\*V_L$ term                                               & $P$   \\ \hline \hline
            $P(X = 0, Y = 0)$ & $\sum_{r} \sum_{a, b : (a, b) \neq (1, 1)}P(R = r, A = a, B = b, Y = 0)$ & 0.594 \\ \hline
            $P(X = 0, Y = 1)$ & $\sum_{r} \sum_{a, b : (a, b) \neq (1, 1)}P(R = r, A = a, B = b, Y = 1)$ & 0.066 \\ \hline
            $P(X = 1, Y = 0)$ & $\sum_{r}P(R = r, A = 1, B = 1, Y = 0)$                                  & 0.05               \\ \hline
            $P(X = 1, Y = 1)$ & $\sum_{r}P(R = r, A = 1, B = 1, Y = 1)$                                  & 0.29  \\ \hline
            \end{tabular}
            \caption{$\tau(P^{\cM_L}(\*V_L))$.}
            \label{fig:drug-ex-taupv}
        \end{subfigure}
    \end{minipage}
    \caption{Dataset quantities computed from $\cM_L$ and graphs for Example \ref{ex:drug-nonid}.}
\end{figure*}

This definition establishes a notion of identification between two different spaces of SCMs, $\Omega_L$ and $\Omega_H$, that are connected through $\tau$. The concept of $\tau$-ID is illustrated in Fig.~\ref{fig:tau-id}. In words, $\tau$-identifiability implies that in every pair of SCMs $\cM_L$ over $\*V_L$ and $\cM_H$ over $\*V_H$, ``matching'' in graph $\cG_{\bbC}$ and data $\bbZ$ implies a match in query $Q$. Since $\cM_L$ and $\cM_H$ are defined over different spaces of variables, the term ``match'' has some nuance. Specifically, ``matching'' in $\cG_{\bbC}$ implies that $\cG_{\bbC}$ is a C-DAG for $\cM_L$ and is a causal diagram for $\cM_H$. ``Matching'' in $\bbZ$ (resp.~$Q$) implies that $\cM_H$ is $\bbZ$-$\tau$ consistent (resp.~$Q$-$\tau$ consistent) with $\cM_L$. As shown in Fig.~\ref{fig:tau-id-id}, $\tau$-ID implies that for any choice of $\cM_L$ over $\*V_L$ that induces C-DAG $\cG_{\bbC}$ (blue dot in light red space), all SCMs $\cM_H$ over $\*V_H$ that induce $\cG_{\bbC}$ and are $\bbZ$-$\tau$ consistent with $\cM_L$ (light yellow space) are also $Q$-$\tau$ consistent with $\cM_L$. As a consequence, $\tau(Q)$ can be evaluated from any of these choices of $\cM_H$ to compute $Q$.

On the other hand, $\tau$-nonidentifiability implies that there exist a pair of models $\cM_L$ over $\*V_L$ and $\cM_H$ over $\*V_H$ such that $\cM_L$ and $\cM_H$ match in both $\cG_{\bbC}$ and $\bbZ$ yet still do not match in $Q$. As shown in Fig.~\ref{fig:tau-id-nonid}, $\tau$-non-ID implies that there exists some $\cM_L$ over $\*V_L$ that induces C-DAG $\cG_{\bbC}$ (blue dot in light red space) such that there is some $\cM_H$ over $\*V_H$ that induces $\cG_{\bbC}$ and is $\bbZ$-$\tau$ consistent with $\cM_L$ (light yellow space) but is not $Q$-$\tau$ consistent with $\cM_L$. This means that despite the constraints added through the C-DAG $\cG_{\bbC}$, there are still queries that cannot be inferred across $\tau$ due to nonidentifiability. This is more acute when there is a large amount of unobserved confounding. Consider the following example where this is the case.

\begin{figure*}
    \begin{subfigure}{.25\linewidth}
        \centering
        \begin{tikzpicture}[xscale=1, yscale=1]
            \node[draw, circle] (R) at (0, 1) {$R$};
            \node[draw, circle] (X) at (-1, 0) {$X$};
            \node[draw, circle] (Y) at (1, 0) {$Y$};
            
            \path [-{Latex[length=2mm]}] (R) edge (X);
            \path [-{Latex[length=2mm]}] (X) edge (Y);
            \path [{Latex[length=2mm]}-{Latex[length=2mm]}, dashed, bend left] (R) edge (Y);
        \end{tikzpicture}
        \caption{C-DAG $\cG_{\bbC}$.}
        \label{fig:ex-drug-newcdag}
    \end{subfigure}
    \begin{subfigure}{.74\linewidth}
        \begin{tabular}{l|l|l}
        \hline \hline
        $\*V_H$ term             & Corresponding $\*V_L$ term                                      & $P$   \\ \hline \hline
        $P(R = 0, X = 0, Y = 0)$ & $\sum_{a, b : (a, b) \neq (1, 1)}P(R = 0, A = a, B = b, Y = 0)$ & 0.432 \\ \hline
        $P(R = 0, X = 0, Y = 1)$ & $\sum_{a, b : (a, b) \neq (1, 1)}P(R = 0, A = a, B = b, Y = 1)$ & 0.048 \\ \hline
        $P(R = 0, X = 1, Y = 0)$ & $P(R = 0, A = 1, B = 1, Y = 0)$                                 & 0.018 \\ \hline
        $P(R = 0, X = 1, Y = 1)$ & $P(R = 0, A = 1, B = 1, Y = 1)$                                 & 0.002 \\ \hline
        $P(R = 1, X = 0, Y = 0)$ & $\sum_{a, b : (a, b) \neq (1, 1)}P(R = 1, A = a, B = b, Y = 0)$ & 0.162 \\ \hline
        $P(R = 1, X = 0, Y = 1)$ & $\sum_{a, b : (a, b) \neq (1, 1)}P(R = 1, A = a, B = b, Y = 1)$ & 0.018 \\ \hline
        $P(R = 1, X = 1, Y = 0)$ & $P(R = 1, A = 1, B = 1, Y = 0)$                                 & 0.032 \\ \hline
        $P(R = 1, X = 1, Y = 1)$ & $P(R = 1, A = 1, B = 1, Y = 1)$                                 & 0.288 \\ \hline
        \end{tabular}
        \subcaption{$\tau(P^{\cM_L}(\*V_L))$.}
        \label{fig:ex-drug-newtaupv}
    \end{subfigure}
    \caption{Updated items given the new choice of $\bbC$ in Example \ref{ex:drug-id}}
\end{figure*}

\begin{example}[Example \ref{ex:drug-tau} continued]
    \label{ex:drug-nonid}
    The data scientist team may be interested in computing the causal effect of taking both stages of the drug on the recovery rate of the disease, $Q = P(Y_{A = 1, B = 1} = 1)$. However, $\cM_L$ is not observed, and instead, the observational data from $P(\*V_L)$ ($\bbZ = \{P(\*V_L)\}$) and the C-DAG $\cG_{\bbC}$ from Fig.~\ref{fig:ex-drug-cdag} are given. Is $Q$ $\tau$-ID from $\bbZ$ and $\cG_{\bbC}$? It turns out the answer is no. To witness, consider the following pair of models.
    {
    \allowdisplaybreaks
    \begin{equation}
    \cM_1 = \!
    \begin{cases}
        \*U_1 = \{U_{XY}, U_{Y0}, U_{Y1}\} \\
        \*V_H = \{X, Y\} \\
        \cF_1 =
        \begin{cases}
            & X \gets f^1_X(u_{XY}) = u_{XY} \\
            & Y \gets f^1_Y(x, u_{XY}, u_{Y0}, u_{Y1}) \\
            &= 
            \begin{cases}
                x \vee u_{Y0} & u_{XY} = 0 \\
                u_{Y1} & u_{XY} = 1
            \end{cases}
        \end{cases} \\
        P(\*U_1) =
        \begin{cases}
            P(U_{XY} = 1) = 0.34 \\
            P(U_{Y0} = 1) = 0.1 \\
            P(U_{Y1} = 1) = 0.852941
        \end{cases}
    \end{cases}
    \end{equation}

    \begin{equation}
    \cM_2 = \!
    \begin{cases}
        \*U_2 = \{U_{XY}, U_{Y0}, U_{Y1}\} \\
        \*V_H = \{X, Y\} \\
        \cF_2 =
        \begin{cases}
            & X \gets f^2_X(u_{XY}) = u_{XY} \\
            & Y \gets f^2_Y(x, u_{XY}, u_{Y0}, u_{Y1}) \\
            &= 
            \begin{cases}
                \neg x \wedge u_{Y0} & u_{XY} = 0 \\
                u_{Y1} & u_{XY} = 1
            \end{cases}
        \end{cases} \\
        P(\*U_2) =
        \begin{cases}
            P(U_{XY} = 1) = 0.34 \\
            P(U_{Y0} = 1) = 0.1 \\
            P(U_{Y1} = 1) = 0.852941
        \end{cases}
    \end{cases}
    \end{equation}
    }

    Every value of $P^{\cM_L}(\*V_L)$ can be computed using the table in Fig.~\ref{fig:drug-ex-u-tab}, producing the table in Fig.~\ref{fig:drug-ex-pv}. The corresponding values of $\tau(P^{\cM_L}(\*V_L))$ can then be computed as shown in Fig.~\ref{fig:drug-ex-taupv} via Eq.~\ref{eq:q-tau-consistency}. One can verify that both $P^{\cM_1}(\*V_H)$ and $P^{\cM_2}(\*V_H)$ match the values in Fig.~\ref{fig:drug-ex-taupv}, implying that both $\cM_1$ and $\cM_2$ are $P(\*V_L)$-$\tau$ consistent with $\cM_L$. Moreover, $\cM_1$ and $\cM_2$ also induce $\cG_{\bbC}$ from Fig.~\ref{fig:ex-drug-cdag}: $f_Y$ takes $X$ as input, and $f_X$ and $f_Y$ share $U_{XY}$ as a confounding variable. Also, $f_X$ does not contain $Y$ as an input.

    However, computing $\tau(Q) = P(Y_{X = 1} = 1)$ leads to two different answers, i.e.:
    \begin{align*}
        & P^{\cM_1}(Y_{X = 1} = 1) \\
        &= P(U_{XY} = 0) + P(U_{XY} = 1, U_{Y1} = 1) \\
        &= 0.66 + (0.34)(0.852941) = 0.95, \stepcounter{equation}\tag{\theequation}
    \end{align*}
    and
    \begin{align*}
        & P^{\cM_2}(Y_{X = 1} = 1) \\
        &= P(U_{XY} = 1, U_{Y1} = 1) \\
        &= (0.34)(0.852941) = 0.29. \stepcounter{equation}\tag{\theequation}
    \end{align*}
    These values are not only different, but also neither are equal to the true value $P^{\cM_L}(Y_{A = 1, B = 1} = 1) = 0.5$, as computed in Eq.~\ref{eq:ex-q-tau-cons-l2q}. A scientist using model $\cM_1$ may conclude that the treatment is extremely effective, while a scientist using model $\cM_2$ may conclude the opposite: the treatment is not only ineffective, it is even harmful. The query is not $\tau$-ID in this case, and no further inferences of this query should be made at this stage.
    \hfill $\blacksquare$
\end{example}

Now consider the following $\tau$-ID example.

\begin{example}[Example \ref{ex:drug-tau} continued]
    \label{ex:drug-id}
    The data scientist team is studying causal inference and notes that Example \ref{ex:drug-nonid} showed that the query $Q = P(Y_{A = 1, B = 1} = 1)$ is non-ID from the available data. However, suppose instead, a different set of clusters is used that includes $R$, i.e.~$\bbC = \{\*C_1 = \{A, B\}, \*C_2 = \{Y\}, \*C_3 = \{R\}\}$. The domain of $R$ would remain the same, so $\tau_{\*C_3}(R) = R$. Constructing $\tau$ with these clusters, the high level variables are revised to $\*V_H = \{R, X, Y\}$, and $\tau(P^{\cM_L}(\*V_L))$ can be computed as shown in the table in Fig.~\ref{fig:ex-drug-newtaupv}. The C-DAG $\cG_{\bbC}$ would be updated to the one in Fig.~\ref{fig:ex-drug-newcdag}.

    It turns out that now, $Q = P(Y_{A = 1, B = 1} = 1)$ is $\tau$-ID, and this can be shown by applying the backdoor-criterion \citep[Thm.~3.3.2]{pearl:2k}, adjusting over the variable $R$ as follows.
    \begin{align*}
        &P(Y_{A = 1, B = 1} = 1) \\
        &= \sum_{r}P(Y = 1 \mid r, A = 1, B = 1)P(r) \\
        &= \sum_{r}P(Y = 1 \mid r, X = 1)P(r). \stepcounter{equation}\tag{\theequation}
    \end{align*}
    
    Then $\sum_{r}P(Y = 1 \mid r, X = 1)P(r)$ could immediately be computed from $P(\*V_H)$.

    As an example, consider the following choice of $\cM_H$, designed in a systematic way to match $\tau(P^{\cM_L}(\*V_L))$.
    {
    \allowdisplaybreaks
    \begin{align}
        \*U_H &= \{U_{RY}, U_{X0}, U_{X1}, U_{Y0}, U_{Y1}, U_{Y2}, U_{Y3}\} \\
        \*V_H &= \{R, X, Y\} \\
        \cF_H &=
        \begin{cases}
            & R \gets f_R(u_{RY}) = u_{RY} \\
            & X \gets f_X(r, u_{X0}, u_{X1}) =
            \begin{cases}
                u_{X0} & r = 0 \\
                u_{X1} & r = 1
            \end{cases} \\
            & Y \gets f_Y(x, u_{RY}, u_{Y0}, u_{Y1}, u_{Y2}, u_{Y3}) \\
            &=
            \begin{cases}
                u_{Y0} & x = 0, u_{RY} = 0 \\
                u_{Y1} & x = 0, u_{RY} = 1 \\
                u_{Y2} & x = 1, u_{RY} = 0 \\
                u_{Y3} & x = 1, u_{RY} = 1
            \end{cases}
        \end{cases} \\
        P(\*U_H) &=
        \begin{cases}
            P(U_{RY} = 1) &= 0.5 \\
            P(U_{X0} = 1) &= 0.04 \\
            P(U_{X1} = 1) &= 0.64 \\
            P(U_{Y0} = 1) &= 0.1 \\
            P(U_{Y1} = 1) &= 0.1 \\
            P(U_{Y2} = 1) &= 0.1 \\
            P(U_{Y3} = 1) &= 0.9
        \end{cases}
    \end{align}
    }
    One can verify that, indeed, $P^{\cM_H}(\*V_H)$ matches $\tau(P^{\cM_L}(\*V_L))$ from Fig.~\ref{fig:ex-drug-newtaupv}. It is also clear that $\cM_H$ induces $\cG_{\bbC}$. Further, we can compute the query to find that
    \begin{align*}
        & P^{\cM_H}(Y_{X = 1} = 1) \\
        &= P(U_{RY} = 0, U_{Y2} = 1) + P(U_{RY} = 1, U_{Y3} = 1) \\
        &= (0.5)(0.1) + (0.5)(0.9) \\
        &= 0.5, \stepcounter{equation}\tag{\theequation}
    \end{align*}
    which matches the true value $P^{\cM_L}(Y_{A = 1, B = 1} = 1) = 0.5$ from Eq.~\ref{eq:ex-q-tau-cons-l2q}.
    
    \hfill $\blacksquare$
\end{example}

The definition of $\tau$-ID provides rigorous semantics to answer whether a query can be inferred across abstractions. The next step is to establish an approach to determine $\tau$-ID when given the available data and graph. For this purpose, one fundamental result is that the notion of $\tau$-ID is actually equivalent to classical identification in the higher level space, as shown by the following proposition.

\begin{restatable}[Dual Abstract ID]{theorem}{dualabsid}
    \label{thm:dual-abs-id}
    Consider a counterfactual query $Q$ over $\*V_L$, a constructive abstraction function $\tau$ w.r.t.~clusters $\bbC$ and $\bbD$, a C-DAG $\cG_{\bbC}$, and data $\bbZ$ from $\*V_L$. $Q$ is $\tau$-ID from $\cG_{\bbC}$ and $\bbZ$ if and only if $\tau(Q)$ is ID from $\cG_{\bbC}$ and $\tau(\bbZ)$.
    \hfill $\blacksquare$
\end{restatable}

In words, $\tau$-ID and classical ID on the high level space are equivalent. This is a powerful result since it implies that inferences can be made about the low level space by using existing results in the high level space. Our goal is to learn a higher level SCM $\cM_H$ to make inferences about $\cM_L$, and we build on the machinery of Neural Causal Models (NCMs) \citep{xia:etal21} toward this goal. NCMs allow one to take the graph $\cG_{\bbC}$ as an inductive bias (a $\cG_{\bbC}$-NCM as described in Def.~\ref{def:gncm}) and leverage gradient-based methods to fit any SCM within the constrained space. Indeed, identification in NCMs can be shown to be equivalent to classical identification when considering models of the same granularity \citep[Thm.~3]{xia:etal23}. When combined with Thm.~\ref{thm:dual-abs-id}, this implies the following result.

\begin{restatable}[Abstract ID with NCMs]{corollary}{absidncm}
    \label{corol:abs-id-ncm}
    $Q$ is $\tau$-ID from $\cG_{\bbC}$ and $\bbZ$ iff $\tau(Q)$ is Neural-ID from $\widehat{\Omega}(\cG_{\bbC})$ and $\tau(\bbZ)$. Moreover, if it is ID, then $Q$ can be computed by $\tau(Q)$ from any $\cG_{\bbC}$-NCM $\widehat{M}$ that is $\tau(\bbZ)$-consistent.
    \hfill $\blacksquare$
\end{restatable}

In words, determining $\tau$-ID is equivalent to determining neural identification (identification in the space of NCMs) on the higher level space. Further, to evaluate $Q$ in identifiable cases, $\tau(Q)$ can be queried from any $\cG_{\bbC}$-NCM $\hM$ that is $\tau(\bbZ)$-consistent. Corol.~\ref{corol:abs-id-ncm} implies that we can perform causal identification and estimation across abstractions using the NeuralID algorithm \citep[Alg.~1]{xia:etal23} on the high level space. This procedure is shown in Alg.~\ref{alg:ncm-solve-absid}. First, $\tau$ is constructed as described in Def.~\ref{def:tau} given the clusters. Then, a $\cG_{\bbC}$-NCM is constructed over high-level variables $\*V_H$. Two parameterizations of the NCM are created. Both are optimized to fit the transformed data $\tau(\bbZ)$, but one is optimized to maximize the transformed query $\tau(Q)$ while the other is optimized to minimize it. If both parameterizations return the same result, then it must be the true value of the query; otherwise, the query is not identifiable.

To implement this algorithm in practice, we leverage the GAN-NCM approach introduced in \citet{xia:etal23}; see details in Appendix \ref{app:experiments}. Alg.~\ref{alg:ncm-solve-absid} is sound and complete for solving the abstract identification problem, as shown below.

\begin{restatable}[Soundness and Completeness]{corollary}{absidcomplete}
    \label{corol:completeness}
    Let $\cM_L$ be the low-level SCM, $\bbC$ and $\bbD$ be inter/intravariable clusters of $\*V_L$, $\cG_{\bbC}$ be a C-DAG, $Q$ be a query, and $\widehat{Q}$ be the result from running Alg.~\ref{alg:ncm-solve-absid} with inputs $\bbZ(\cM_L) > 0$, $\bbC$, $\bbD$, $\cG_{\bbC}$, and $Q$. Then, $Q$ is $\tau$-ID from $\cG_{\bbC}$ and $\bbZ$ if and only if $\widehat{Q}$ is not \texttt{FAIL}. Moreover, if $\widehat{Q}$ is not \texttt{FAIL}, then $\widehat{Q} = Q(\cM_L)$.
    \hfill $\blacksquare$
\end{restatable}

While Alg.~\ref{alg:ncm-solve-absid} solves the abstract ID problem, the consequences of the results in this section are more general. Notably, if $Q$ is indeed $\tau$-ID (which can be verified through Alg.~\ref{alg:ncm-solve-absid}), the algorithm produces a neural model $\widehat{M}$ that serves as a proxy SCM that is $Q$-$\tau$ consistent with the true model $\cM_L$. Such an SCM could serve as a generative model of the distribution $Q$, which has many uses. The samples generated from such a model could be used to estimate the query, or, in more complex settings such as with image data, it may be desirable to simply have novel generated samples consistent with the causal invariances embedded in the system.

%\begin{wrapfigure}{r}{0.57\textwidth}
%\vspace{-0.2in}
%\hspace{-0.8cm}
%\IncMargin{1em}
%\begin{algorithm}[H]
\begin{algorithm}[t]
    %\scriptsize
    %\setstretch{0.9}
    %\renewcommand{\thealgocf}{\arabic{algocf} (NeuralID)}
    %\renewcommand{\AlCapSty}[1]{\normalfont\scriptsize{\textbf{#1}}\unskip}
    
    \DontPrintSemicolon
    \SetKwData{ncmdata}{$\widehat{M}$}
    \SetKwData{graphdata}{$\cG_{\bbC}$}
    \SetKwData{variabledata}{$\*V_H$}
    \SetKwData{pvdata}{$P(\*v)$}
    \SetKwData{thetamin}{$\bm{\theta}_{\min}^*$}
    \SetKwData{thetamax}{$\bm{\theta}_{\max}^*$}
    \SetKwFunction{ncmfunc}{NCM}
    \SetKwInOut{Input}{Input}
    \SetKwInOut{Output}{Output}
    
    \Input{ query $Q$, $\cL_2$ datasets $\bbZ(\cM_L)$, C-DAG \graphdata, and admissible inter/intravariable clusters $\bbC$ and $\bbD$ satisfying invariance condition}
    \Output{ $Q(\cM_L)$ if identifiable, \texttt{FAIL} otherwise.}
    \BlankLine
    \textls[-20]{
    $\*V_H \gets \bbC, \cD_{\*V_H} \gets \bbD$\;
    $\tau \gets \absfunc(\bbC, \bbD)$ \tcp*{from Def.\ \ref{def:tau}}
    $\ncmdata \gets \ncmfunc{\variabledata, \graphdata}$ \tcp*{from Def.\ \ref{def:gncm}}
    $\thetamin \! \gets \! \arg \min_{\bm{\theta}} \tau(Q)(\ncmdata(\bm{\theta}))$ s.t. $\tau(\bbZ)(\ncmdata(\bm{\theta})) \! = \! \tau(\bbZ(\cM_L))$\;
    $\thetamax \! \gets \! \arg \max_{\bm{\theta}} \tau(Q)(\ncmdata(\bm{\theta}))$ s.t. $\tau(\bbZ)(\ncmdata(\bm{\theta})) \! = \! \tau(\bbZ(\cM_L))$\;
    \eIf{$\tau(Q)(\ncmdata(\thetamin)) \neq \tau(Q)(\ncmdata(\thetamax))$} {
        \Return \texttt{FAIL}
    }{
        \Return $\tau(Q)(\ncmdata(\thetamin))$ \tcp*{choose min or max arbitrarily} 
    }
    }
    %\caption{\scriptsize \textbf{NeuralAbstractID} -- Identifying/estimating high-level causal/counterfactual queries with NCMs.}
    \caption{\textbf{NeuralAbstractID} -- Identifying and estimating queries across abstractions using NCMs.}
    \label{alg:ncm-solve-absid}
\end{algorithm}
%\DecMargin{1em}
%\vspace{-0.2in}
%\end{wrapfigure}

\section{Representations in Learning Abstractions} \label{sec:applications}

In many applications, the choice of intervariable clusters $\bbC$ is natural and can be made in tandem when deciding the assumptions of the C-DAG $\cG_{\bbC}$\footnote{Please refer to App.~\ref{app:cons-hierarchy} for best practices on how to choose or learn intervariable clusters when they are not given.}. On the other hand, fully specifying the intravariable clusters $\bbD$ is usually challenging when working with high-dimensional data like images. If an intervariable cluster contained three binary variables from $\*V_L$, then specifying its intravariable cluster would require specifying some partition over its eight values, $(0, 0, 0), (0, 0, 1), \dots, (1, 1, 1)$, which is not difficult. However, if an intervariable cluster contained, for example, $128 \times 128$ pixels, each with 256 possible values, then the size of the domain of this cluster would be $256^{128 \times 128}$. Specifying an arbitrary partition over this many values is infeasible, as doing so would require an enumeration of every possible image along with some label designating each one to a cluster. In this section, we investigate the problem of learning abstractions when the intravariable clusters $\bbD$ are left unspecified.

While coarser clusters tend to be better in practice due to the dimensionality reduction, the theory in this paper can be applied for any choice of $\bbD$ so long as the AIC (Def.~\ref{def:invariance-condition}) holds. Hence, a possible constraint when learning $\bbD$ is to find a set of clusters such that the AIC is not violated. To this effect, the following result can be leveraged.

\begin{restatable}{proposition}{fullintracluster}
    \label{prop:full-intra-cluster}
    Consider a low level SCM $\cM_L$ and constructive abstraction function $\tau$ w.r.t.~clusters $\bbC$ and $\bbD$. $\cM_L$ is guaranteed to satisfy the AIC w.r.t.~$\tau$ if and only if $\bbD_{\*C_i} = \{\{\*c_i\} : \*c_i \in \cD_{\*C_i}\}$ for all $\*C_i \in \bbC$.
    \hfill $\blacksquare$
\end{restatable}

In words, the AIC is satisfied when the intravariable clusters are maximal (i.e.~each value is assigned to its own cluster). Consequently, this means that Alg.~\ref{alg:ncm-solve-absid} can be applied in any case where $\tau_{\*C_i}$ is a bijective mapping between $\cD_{\*C_i}$ and $\cD_{V_{H, i}}$. Also implied by this result is that, without additional information, one cannot choose any coarser clustering without potentially violating the AIC\footnote{In many cases, there may be additional information in the form of invariances (e.g.~rotational invariance in image data). In such cases, this information can be leveraged to learn coarser clusters. See Appendix \ref{app:rep-learning} for more details.}. For some intuition on why this is the case, consider the following example.

\begin{example}[Example \ref{ex:bmi} continued]
    Consider once again the nutrition example where carbohydrates ($C$), fat ($F$), and protein ($P$) were combined to form the high-level variable, calories ($Z$). Calories was defined to be $Z = 4C + 9F + 4P$, resulting in a set of intravariable clusters where $(c_1, f_1, p_1)$ and $(c_2, f_2, p_2)$ were clustered together if $4c_1 + 9f_1 + 4p_1 = 4c_2 + 9f_2 + 4p_2$.

    It turns out that, without additional information, it is possible that this choice of clustering violates the AIC. For example, it could be the case that in the true model $\cM_L$, the function for BMI $f_B$ depends heavily on protein, since muscle density may affect BMI more than fat. In an extreme example, suppose
    \begin{equation}
        B \gets f_B(c, f, p, u_B) = p + u_B,
    \end{equation}
    that is, $f_B$ only depends on protein out of the three macronutrients. In this case, $(C = 0, F = 0, P = 20)$ and $(C = 20, F = 0, P = 0)$ would result in different values of $B$ despite being clustered together in the same intravarable cluster, therefore, violating the AIC. Without information about $f_B$, it is not known whether any given two tuples, $(c_1, f_1, p_1)$ and $(c_2, f_2, p_2)$, would witness this violation, so the only option is to leave all values of $(C, F, P)$ in their own intravariable clusters.
    \hfill $\blacksquare$
\end{example}
This example illustrates that without further information about the functions of the underlying SCM, it is, in general, impossible to cluster two intravariable values together without potentially violating the AIC. This is also the case with the cholesterol discussion in Ex.~\ref{ex:cholesterol-inv-cond} and \ref{ex:inv-cond-satisfied}. Without knowing $f_Y$, it would be impossible to determine whether the cluster choice in Ex.~\ref{ex:inv-cond-satisfied} would work better than the one in Ex.~\ref{ex:cholesterol-inv-cond}.

Still, Prop.~\ref{prop:full-intra-cluster} states that the intravariable clustering which leaves every value in its own cluster will always satsify the AIC. While this choice of $\bbD$ does not reduce the dimensionality of the abstracted space, this means that we are not restricted to the original space of $\*V_L$ and can choose any $\*V_H$ with the same cardinality. In practice, this means that we can choose the option for $\*V_H$ that is the most beneficial for a given task. For example, some choices of $\*V_H$ and corresponding domains $\cD_{\*V_H}$ may have desirable properties such as simpler gradient computation, disentangled variables, or compatibility with arithmetic operations. In order to leverage this insight, we introduce the representational NCM.

\begin{definition}[Representational NCM (RNCM)]
    \label{def:rncm}
    A representational NCM (RNCM) is a tuple $\langle \widehat{\tau}, \widehat{M} \rangle$, where $\widehat{\tau}(\*v_L; \bm{\theta}_{\tau})$ is a function parameterized by $\bm{\theta}_{\tau}$ mapping from $\*V_L$ to $\*V_H$, and $\widehat{M}$ is an NCM defined over $\*V_H$. A $\cG_{\bbC}$-constrained RNCM ($\cG_{\bbC}$-RNCM) is an RNCM $\langle \widehat{\tau}, \widehat{M} \rangle$ such that $\widehat{\tau}$ is composed of subfunctions $\widehat{\tau}_{\*C_i}$ for each $\*C_i \in \bbC$ (each with its own parameters $\bm{\theta}_{\tau_{\*C_i}}$), and $\widehat{M}$ is a $\cG_{\bbC}$-NCM (Def.~\ref{def:gncm}).
    \hfill $\blacksquare$
\end{definition}

\begin{figure}
%\begin{wrapfigure}{r}{0.5\textwidth}
%\vspace{-0.2in}
%\hspace{-0.05cm}
    \centering
\includegraphics[width=0.9\linewidth]{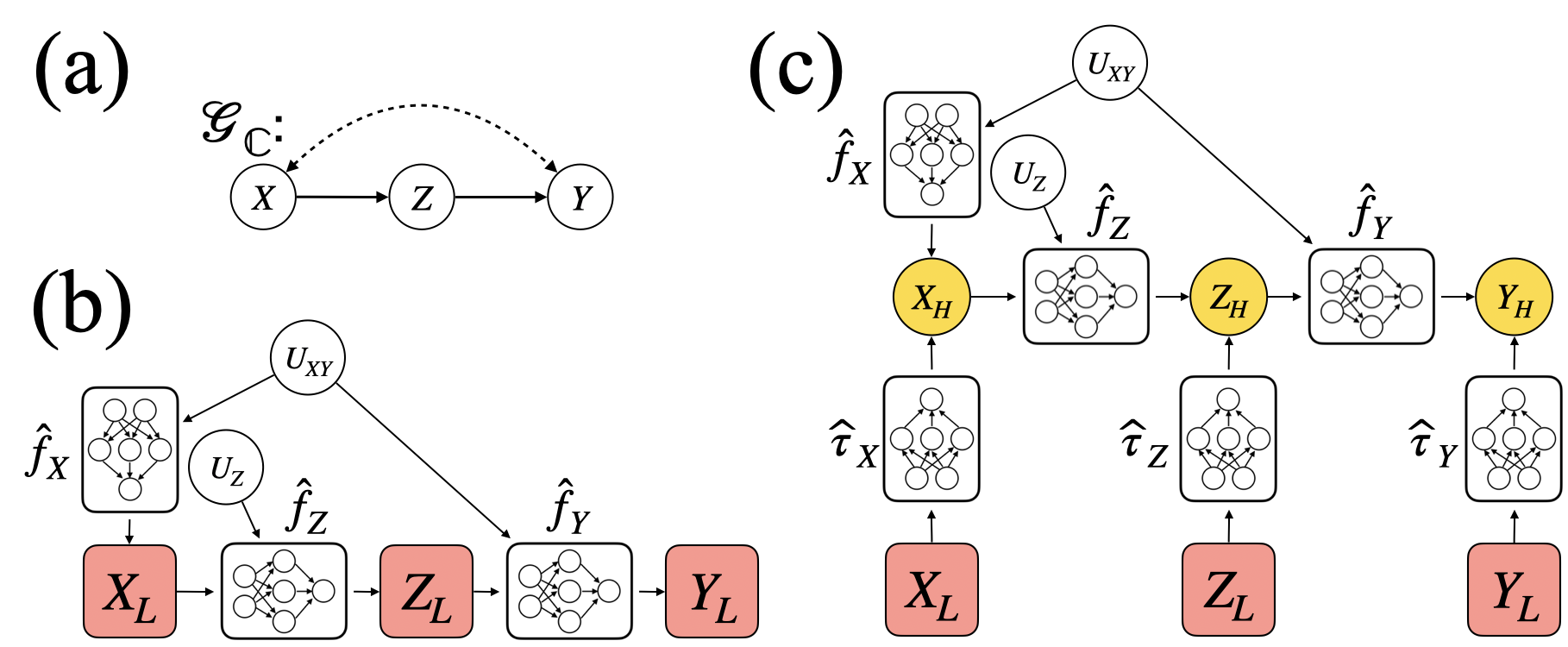}
\caption{Example using graph $\cG_{\bbC}$ shown in (a) to compare (b) the $\cG_{\bbC}$-NCM and (c) $\cG_{\bbC}$-RNCM. Functions of the NCM directly output values of the lower level variables (grouped by clusters in $\bbC$), while functions of the RNCM output values of their higher level counterparts, mapped by $\widehat{\tau}$.}
\label{fig:rncm-example}
%\vspace{-0.3in}
%\end{wrapfigure}
\end{figure}

In words, an RNCM is a pair of a parameterized abstraction function $\widehat{\tau}$ and an NCM $\hM$ defined over the space of high level variables $\*V_H$ obtained from $\widehat{\tau}(\*V_L)$. The $\cG_{\bbC}$-RNCM is simply an RNCM constrained over the C-DAG $\cG_{\bbC}$, where $\widehat{\tau}$ must use the intervariable clusters specified by $\bbC$. Fig.~\ref{fig:rncm-example} illustrates the difference between a $\cG_{\bbC}$-RNCM and a standard $\cG_{\bbC}$-NCM. Given $\cG_{\bbC}$ in (a), the $\cG_{\bbC}$-NCM (b) directly defines the domains of the inputs and outputs of the functions to be the space of the variables $\*V_L$. On the other hand, the $\cG_{\bbC}$-RNCM (c) defines the domains over $\*V_H = \widehat{\tau}(\*V_L)$ instead.

Training can be done as a two-step process, where first $\widehat{\tau}$ is trained to map to an optimal task-specific space, and then $\hM$ can be trained on $\widehat{\tau}(\*V_L)$ (e.g., through Alg.~\ref{alg:ncm-solve-absid}). In the first step, $\widehat{\tau}$ (parameterized by $\theta_{\tau}$) can be trained according to a loss such as
\begin{align}
    L_{\widehat{\tau}}(\*v_L) &= \norm{\widehat{\tau}^{-1}(\widehat{\tau}(\*v_L, \theta_{\tau}); \theta_{\tau^{-1}}) - \*v_L}^2 \nonumber \\
    &+ \lambda_r L_r(\*v_L), \label{eq:rncm-loss-general}
\end{align}
where $\widehat{\tau}^{-1}$ is a neural network parameterized by $\theta_{\tau^{-1}}$ that attempts to invert $\widehat{\tau}$ and recover the original $\*v_L$, $L_r$ is a separate representation regularizer, and $\lambda_r$ is its regularization strength. The first term in Eq.~\ref{eq:rncm-loss-general} is used to enforce bijectivity between $\cD_{\*C_i}$ and $\cD_{V_{H, i}}$, as required by Prop.~\ref{prop:full-intra-cluster}. In words, one can train $\widehat{\tau}$ in an autoencoder-like setup \citep{kramer:91, KingmaW13} with a reconstruction loss. The second term is left open-ended and can be used to impose a desired form on the output of $\widehat{\tau}$. The following example helps to ground this point. 

\begin{example}
    Consider a simple example with $\*V_L = \*X \cup \{Y\}$, where $\*X$ is the collection of pixels representing an image of a cat or dog, and $Y$ is a binary label (caused by $\*X$) predicting whether the animal depicted in $\*X$ is aggressive or docile. Suppose researchers are studying the relationship between $\*X$ and $Y$ and would like to work in a more abstract space. Denote $X_H$ as the high-level counterpart of $\*X$. Moreover, suppose the researchers are given another set of labels of $\*X$, say $Z$, which state whether the animal in the image is a cat or dog. $Z$ is not included in the study with $\*V_L$, but it could potentially be used for learning the space of $X_H$.
    
    Specifically, the researchers would like to train $\widehat{\tau}_X: \cD_{\*X} \rightarrow \cD_{X_H}$ such that the AIC is satisfied from Prop.~\ref{prop:full-intra-cluster}, and $X_H$ additionally encodes some information about $Z$. One clever approach is to introduce another neural network $g(x_H; \theta_g)$ parameterized by $\theta_g$ that aims to classify $Z$ from $X_H$. The researchers could train $\widehat{\tau}_X$ with the loss $L_{\widehat{\tau}}$ from Eq.~\ref{eq:rncm-loss-general} such that $L_r$ is the classification loss of $g$. This would result in $X_H$ encoding both enough information to reconstruct $\*X$ and also to classify $Z$. Although $X_H$ may not have lower dimensionality than the original $X$, it may be more useful in a downstream task, such as building a classifier for $Y$ (i.e.~$X_H$ may be a more well-behaved set of features for a classifier of $Y$ than the original set of pixel values $\*X$).
    \hfill $\blacksquare$
\end{example}

In general, $\widehat{\tau}$ can be thought of as a function mapping to a representation space, where the second term in Eq.~\ref{eq:rncm-loss-general} can be used to regularize the representation space for a desired task. The flexibility of this approach makes it amenable to the wide developments of the representation learning literature \citep{10.1109/TPAMI.2013.50}. We empirically demonstrate this approach below in the experiment of Sec.~\ref{sec:exp-mnist}.

\section{Experiments} \label{sec:experiments}

In this section, we empirically evaluate the effects of utilizing abstractions in causal inference tasks. More details on the data-generating models and architectures can be found in Appendix \ref{app:experiments}. Implementation code is publicly available at \url{https://github.com/CausalAILab/NeuralCausalAbstractions}.

\subsection{Nutritional Study}
\label{sec:exp-nutrition}

We follow up on the nutrition study discussed in Ex.~\ref{ex:bmi}. Since a BMI of 25 or over is considered overweight, the goal is to identify and estimate the query $Q = P(B_{D=d} \geq 25)$, the causal effect of diet on weight, given the available graphical constraints and observational data $P(\*V_L)$ using Alg.~\ref{alg:ncm-solve-absid}. $R$ and $D$ are 32-dimensional one-hot vectors, and the others are real-valued, so the query may be difficult to answer given such high-dimensional variables. Instead, it may be more effective to work in an abstract space with the proposed intervariable clusters $\bbC = \{D_H = \{D\}, Z = \{C, F, P\}, B_H = \{B\}\}$. %Consider intervariable clusters $\bbC = \{\{D\}, \{C, F, P\}, \{B\}\}$, as outlined in blue in Fig.~\ref{fig:exp1-graph}, and denote these three clusters as $F_H$, $N_H$, and $B_H$, respectively.
The original graph $\cG$ and corresponding C-DAG $\cG_{\bbC}$ are shown in Fig.~\ref{fig:cdag-examples}. We are also given intravariable clusters $\bbD$ such that all values of $D_H$, $Z$, and $B_H$ are clustered into binary categories. Specifically, $D_H = 1$ denotes unhealthy dishes, $Z = 1$ denotes high calorie count, and $B_H = 1$ denotes an overweight BMI ($\geq 25$).

We compare the effectiveness identifying and estimating $Q$ with NCMs in three different settings, with results shown in Fig.~\ref{fig:exp1-results}. The first approach (red) attempts to solve the problem directly in the space of $\*V_L$ by identifying and estimating $Q$ from the original causal diagram $\cG$ and observational dataset from $P(\*V_L)$. The second approach (yellow) solves the same task but first normalizes each variable\footnote{This normalization approach is equivalent to using a constructive abstraction function $\tau$ over the full set of clusters $\bbC = \*V_L$ and $\bbD = \cD_{\*V_L}$ (with a bit of abuse of notation). That is, each variable and value are placed in their own cluster, and all values are simply remapped to different values.} of the data between 0 and 1. The third approach (blue) is the newly proposed approach and leverages the concept of $\tau$-ID, identifying and estimates $Q$ from the C-DAG $\cG_{\bbC}$ and high level data $\tau(P(\mathbf{V}_L))$. The model is trained over the abstract space of $\*V_H$ computed using the constructive abstraction function $\tau$ defined on $\bbC$ and $\bbD$. All three approaches are implemented in the style of GAN-NCM \citep{xia:etal23}. Since $Q$ is identifiable, the gap between the max and min queries computed in Alg.~\ref{alg:ncm-solve-absid} are expected to be as small as possible. As shown in Fig.~\ref{fig:exp1-id-gaps}, the proposed approach (blue) converges quickly while others fail to close the gap between the max and min queries. Fig.~\ref{fig:exp1-est-mae} also shows that the proposed approach can estimate $Q$ with significantly lower error. Furthermore, since the proposed approach uses the C-DAG $\cG_{\bbC}$ instead of the original causal diagram $\cG$, the approach operates under fewer assumptions of domain knowledge.

\begin{figure}
    \centering
    \begin{subfigure}[t]{0.48\linewidth}
        \centering
        \includegraphics[width=\linewidth]{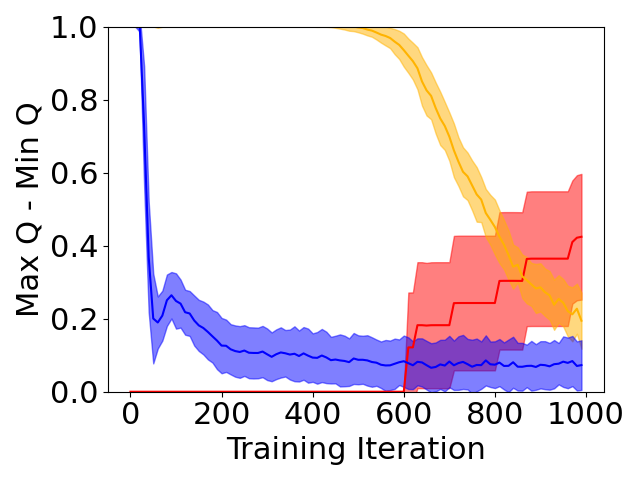}
        \caption{Gaps between max and min query across 1000 training iterations when running Alg.~\ref{alg:ncm-solve-absid}.}
        \label{fig:exp1-id-gaps}
    \end{subfigure}%
    \hfill
    \begin{subfigure}[t]{0.48\linewidth}
        \centering
        \includegraphics[width=\linewidth]{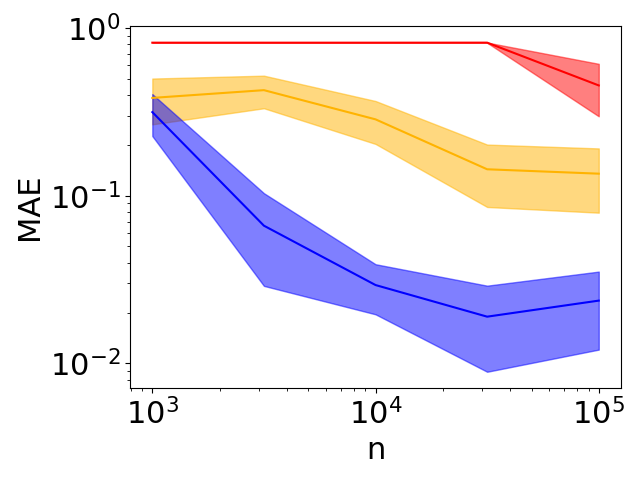}
        \caption{Mean absolute error (MAE) v.~dataset size (in log-log scale) for query estimation.}
        \label{fig:exp1-est-mae}
    \end{subfigure}
    \caption{Results of the nutrition experiment. Our approach (blue) is compared with a GAN-NCM trained on raw data (red) and one trained on normalized data (yellow).}
    \label{fig:exp1-results}
\end{figure}

\begin{figure}
    \centering
    \begin{subfigure}[b]{0.49\linewidth}
        \centering
        \begin{tikzpicture}[xscale=0.7, yscale=1]
            \tiny
            \node[draw, circle] (C) at (-0.5, 1) {$C$};
            \node[draw, circle] (D) at (-1, 0) {$D$};
            \node[draw, circle] (I) at (1, 0) {$I$};
            
            \path [-{Latex[length=1mm]}] (D) edge (I);
            \path [-{Latex[length=1mm]}] (C) edge (I);
            \path [{Latex[length=1mm]}-{Latex[length=1mm]}, dashed, bend left] (D) edge (C);
        \end{tikzpicture}
        \caption{$\cG_{\bbC}$ for Colored MNIST.}
        \label{fig:exp2-graph}
    \end{subfigure}%
    \hfill
    \begin{subfigure}[b]{0.49\linewidth}
        \centering
        \includegraphics[width=\linewidth,keepaspectratio]{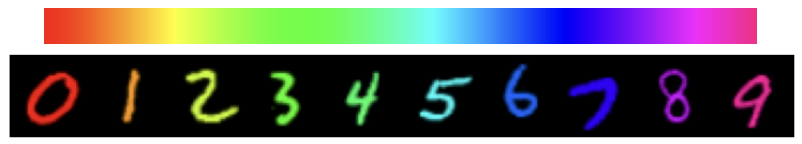}
        \caption{Image samples. Digits are highly correlated with the corresponding gradient color.}
        \label{fig:exp2-legend}
    \end{subfigure}
    \caption{Colored MNIST Experimental Setup}
    \label{fig:exp-mnist-setup}
\end{figure}
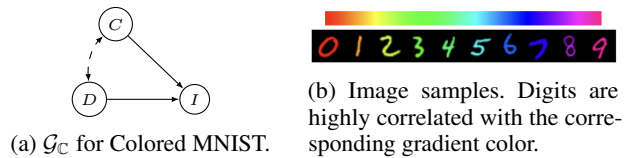

\subsection{Colored MNIST Digits}
\label{sec:exp-mnist}

\begin{figure*}
    \begin{center}
    \includegraphics[width=\textwidth,keepaspectratio]{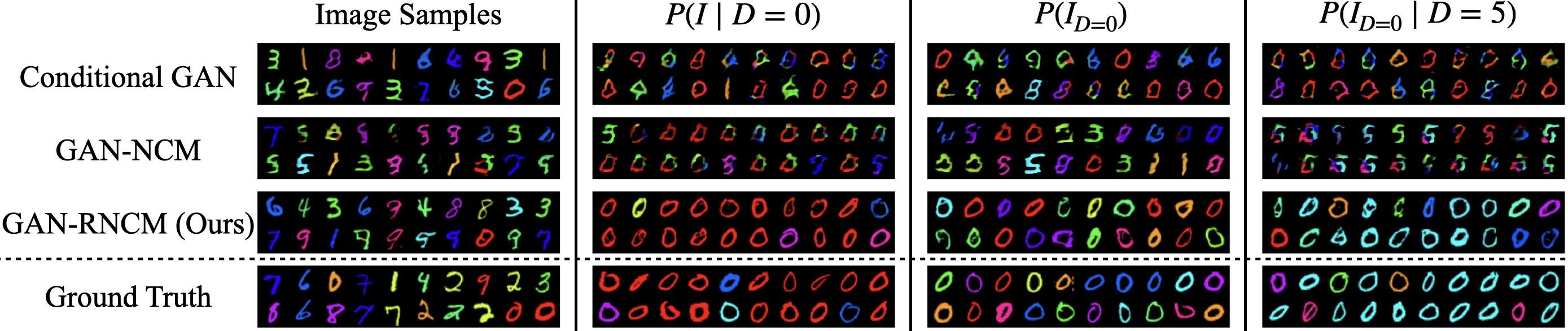}
    \caption{Colored MNIST results. Samples from various causal queries (top) are collected from competing approaches (left), with the ground truth samples from the data generating model shown in the bottom row. The left column simply shows image samples from $P(I)$ from each of the models, while the second, third, and fourth columns show samples generated from an $\cL_1$, $\cL_2$, and $\cL_3$ query, respectively.
    }
    \label{fig:mnist-exp-results}
    \end{center}
    %\vspace{-0.5cm}
\end{figure*}

We evaluate the RNCM in a high-dimensional image dataset of colorized MNIST \citep{deng2012mnist} digits. Each image ($I$) has a corresponding digit ($D$) and color ($C$) label, and their relationships are shown in the C-DAG $\cG_{\bbC}$ in Fig.~\ref{fig:exp2-graph}. Color and digit are highly correlated (e.g.~0s are typically red, while 5s are cyan), as shown in Fig.~\ref{fig:exp2-legend}. We evaluate three approaches in the task of sampling images from causal queries. The first approach is a na\"ive conditional GAN that does not take causality into account. The second is a standard GAN-NCM as described in \citet{xia:etal23}. The third is called a GAN-RNCM, a GAN implementation of the representational NCM following the approach described in Sec.~\ref{sec:applications}.

Samples of the results are shown in Fig.~\ref{fig:mnist-exp-results}. All models are capable of producing digit images, as shown in the first column. The second column illustrates $P(I \mid D=0)$, the images conditioned on digit $=0$. Many red 0s are expected since most 0s are red in the dataset. The third column illustrates the interventional query $P(I_{D=0})$, the images with digits forced to be $0$ through intervention. As interventions ignore the spurious correlations between color and digit, 0s of all colors are expected. Finally, the fourth column illustrates the counterfactual query $P(I_{D=0} \mid D=5)$, indicating what the digits would have looked like had they been 0, given that they were originally 5. Since 5s tend to be cyan, the samples are expected to be 0s that retain the cyan color of the 5s.

In all cases, GAN-RNCM (new approach) produces results closer to the expected outcomes, as shown by the ground truth. The conditional GAN fails to distinguish causal queries from conditional queries, and samples appear similar in the 2nd, 3rd, and 4th columns. 
The standard GAN-NCM faces challenges disentangling color from digit, as shown from the presence of several non-zero digits in the 3rd column and several digits that resemble 5s in the 4th column. 
Further, both the conditional GAN and the GAN-NCM face challenges in capturing the relationship between color and image in more complex distributions, as evident from the mosaic coloring in many of the samples. 
The GAN-RNCM significantly outperforms the other approaches in terms of sample quality for the causal queries.

\section{Conclusions}

\begin{figure*}
    \centering
    \begin{tabular}{c|c|c|c|c|c|c|c|c|c}
    \toprule
    \multirow{2}{*}{Approach} &  \multicolumn{4}{|c|}{SCM $\cM_L$} & \multicolumn{3}{|c|}{Abstraction $\tau$} & SCM $\cM_H$ & \multirow{2}{*}{Output} \\ \cline{2-9}
    {} &  $\cM_L$ & Data $\bbZ$ & Graph $\cG$ & C-DAG $\cG_{\bbC}$ & $\tau$ & $\bbC$ & $\bbD$ & $\cM_H$ & {} \\
    \midrule
    Existing works  & \checkmark   & \textcolor{lightgray}{\checkmark}   & \textcolor{lightgray}{\checkmark}   & \textcolor{lightgray}{\checkmark} & \checkmark & \textcolor{lightgray}{\checkmark} & \textcolor{lightgray}{\checkmark} & \checkmark & Abstraction: Yes/No \\
    Sec.~\ref{sec:abstract-ncm}   & \checkmark   & \textcolor{lightgray}{\checkmark}   & \textcolor{lightgray}{\checkmark}   & \textcolor{lightgray}{\checkmark} & - & \checkmark & \checkmark & - & $\cL_3$-$\tau$ consistent $\cM_H$ \\
    Sec.~\ref{sec:learning-abs}   & -   & \checkmark    & -   & \checkmark & - & \textcolor{lightgray}{\checkmark}  & \checkmark & - & $\bbZ$-$\tau$ / $\cG_{\bbC}$ consistent $\cM_H$ \\
    Sec.~\ref{sec:applications}   & -   & \checkmark    & -   & \checkmark & - & \textcolor{lightgray}{\checkmark}  & - & - & $\bbZ$-$\tau$ / $\cG_{\bbC}$ consistent $\cM_H$ \\
    \bottomrule
    \end{tabular}
    \caption{Summary table of contributions of each section in terms of input assumptions and outputs. A checkmark (\checkmark) indicates that the corresponding information is assumed to be available in that section. A gray checkmark indicates that it is implied by a stronger assumption (e.g. data $\bbZ$ can be sampled from $\cM_L$ if $\cM_L$ is available). A dash (-) indicates that the information is not assumed. An approach that has fewer checkmarks makes fewer assumptions and is therefore more applicable in practice.}
    \label{tab:summary-table}
\end{figure*}

In this paper, we developed a new framework of abstractions based on the PCH's layers with the goal of learning a high-level causal model at a coarser granularity. 
In each section of the paper, we relaxed certain assumptions that are not easily achievable in practice, and showed how to obtain the high-level model under these particular settings. 
These increasingly more refined results are summarized in Fig.~\ref{tab:summary-table}.

We started by noting that previous works on causal abstractions, such as \citet{beckers2019abstracting}, set a strong foundation of defining abstractions in rigorous terms. (A detailed discussion is provided in Appendix \ref{sec:related-work}.) 
The main drawback is that these definitions are declarative, meaning that given a low level SCM $\cM_L$, high level SCM $\cM_H$, and an abstraction function $\tau$, the definitions in these works can be used to decide whether $\cM_H$ is an abstraction of $\cM_L$ or not. 
Still, solving this decision task requires a substantial amount of input information, as highlighted in the first row of Fig.~\ref{tab:summary-table}, and which is unavailable in many practical settings.

We then examined in Sec.~\ref{sec:abstract-ncm}~the challenge of obtaining the abstract model $\cM_H$ when the same is unavailable, as shown in the second row of Fig.~\ref{tab:summary-table}. 
We provided a systematic way of constructing $\tau$ given inter/intravariable clusters (Defs.~\ref{def:var-clusterings} and ~\ref{def:tau}). 
Building on this, we developed Alg.~\ref{alg:map-abstraction}, which allows one to obtain the higher level abstraction $\cM_H$ given $\cM_L$ and the constructive abstraction function $\tau$. 

Still, these results can be hard to apply since $\cM_L$ is not commonly available in some real world settings. Rather, partial information about $\cM_L$ through its data distributions may be available. We then described in Sec.~\ref{sec:learning-abs} how to obtain $\cM_H$ under these restrictions, as shown in the third row of Fig.~\ref{tab:summary-table}. 
Causal inferences on higher levels of the PCH cannot be realized using lower layers alone, as shown by Prop.~\ref{prop:abs-cht}, which means that assumptions are needed. 
The assumptions considered in this paper take the form of C-DAGs (Def.~\ref{def:cdag}), an abstract version of causal diagrams leveraging the intervariable clusters. Then, given the data $\bbZ$, the C-DAG $\cG_{\bbC}$, and the abstraction function $\tau$ constructed from the clusters, Alg.~\ref{alg:ncm-solve-absid} can be used to learn the abstract NCM $\widehat{M}_H$. Then, $\widehat{M}_H$ can be used for tasks such as identification, estimation, or sampling.

We showed in Sec.~\ref{sec:applications} how to learn $\tau$, noting that acquiring intravariable clusters $\bbD$ may be challenging in practice. 
The fourth row of Fig.~\ref{tab:summary-table} highlight this task. 
Specifically, we introduced the representational NCM or RNCM (Def.~\ref{def:rncm}), which parameterizes $\widehat{\tau}$ as a neural network.
We then developed a learning procedure for intravariable clusters,  considering task-specific objectives. 
Whenever additional information about the problem is available in the form of invariances (e.g., ~translational, rotational, permutation), coarser clusters can be learned, as elaborated in Appendix~\ref{app:rep-learning}.

Finally, although not shown in the table, additional technical content can be found in Appendix \ref{app:discussion}. Specifically, Appendix \ref{app:cons-hierarchy} discusses how to make the best choice of intervariable clusters $\bbC$ when it is not given. Appendix \ref{app:invariance-condition} describes possible ways to relax the AIC (Def.~\ref{def:invariance-condition}), which is an assumption that is made throughout the paper. We encourage further research on the topics covered in this paper, such as on the best way to learn $\widehat{\tau}$ in an RNCM.

\section*{Acknowledgements}
This research was supported in part by the NSF, ONR, AFOSR, DARPA, DoE, Amazon, JP Morgan, and The Alfred P. Sloan Foundation.

%\clearpage
\bibliography{references}

\clearpage
\appendix
\section{Proofs} \label{app:proofs}

\subsection{Counterfactual Axioms and Properties}

The proofs in this work will rely on the following counterfactual axioms from \citet{galles:pea98} and \citet{halpern:98}:
\begin{fact}[{\citep[Thms.~1,2]{galles:pea98}}]
    The following properties hold in SCMs:
    \begin{enumerate}[label=\arabic*]
        \item (Composition). For any two singleton variables $Y$ and $W$, and any set of variables $\*X$ in a causal model, we have
        \begin{equation*}
            W_{\*x}(\*u) = w \Longrightarrow Y_{\*x, w}(\*u) = Y_{\*x}(\*u)
        \end{equation*}

        \item (Effectiveness) For all variables $\*X$ and $\*W$, $\*X_{\*x \*w}(\*u) = \*x$.
    \end{enumerate}
\hfill $\blacksquare$
\end{fact}

As a consequence of these axioms, we can prove the following result:
\begin{corollary}
    \label{corol:ctf-parent-form}
    For any set of variables $\*Y$ and $\*X$, we have
    \begin{equation*}
        \*Y_{\*x}(\*u) = \*Y_{\pai{\*Y}^{(1)}, \Pai{\*Y[\*x]}^{(2)}(\*u)}(\*u) = \*Y_{\Pai{\*Y[\*x]}(\*u)}(\*u),
    \end{equation*}
    where $\Pai{\*Y}^{(1)} = \Pai{\*Y} \cap \*X$, $\pai{\*Y}^{(1)}$ are its corresponding values of $\*x$, and $\Pai{\*Y}^{(2)} = \Pai{\*Y} \setminus \Pai{\*Y}^{(1)}$.
    \hfill $\blacksquare$

    \begin{proof}
        We have
        \begin{align}
            \*Y_{\*x}(\*u) &= \*Y_{\pai{\*Y}^{(1)}, \*x'}(\*u) \label{eq:ctf-parent-form-1} \\
            &= \*Y_{\pai{\*Y}^{(1)}, \Pai{\*Y[\*x]}^{(2)}(\*u), \*x'}(\*u) \label{eq:ctf-parent-form-2} \\
            &= \*Y_{\pai{\*Y}^{(1)}, \Pai{\*Y[\*x]}^{(2)}(\*u)}(\*u) \label{eq:ctf-parent-form-3} \\
            &= \*Y_{\Pai{\*Y[\*x]}^{(1)}(\*u), \Pai{\*Y[\*x]}^{(2)}(\*u)}(\*u) \label{eq:ctf-parent-form-4} \\
            &= \*Y_{\Pai{\*Y[\*x]}(\*u)}(\*u). \label{eq:ctf-parent-form-5}
        \end{align}
        Here, $\*x'$ denotes the values of $\*x$ corresponding to $\*X \setminus \Pai{\*Y}^{(1)}$. Eq.~\ref{eq:ctf-parent-form-1} holds by definition. Eq.~\ref{eq:ctf-parent-form-2} holds by the composition property, which can be applied iteratively to each variable in $\*Y$ and $\Pai{\*Y}^{(2)}$. Eq.~\ref{eq:ctf-parent-form-3} holds because the value of $\*Y$ is deterministic once $\Pai{\*Y}$ and $\*U$, the inputs to its functions, are fixed. Eq.~\ref{eq:ctf-parent-form-4} holds by the effectiveness property. Finally Eq.~\ref{eq:ctf-parent-form-5} holds by definition.
    \end{proof}
\end{corollary}

We also leverage the following results from \citet{correa:etal21}:

\begin{definition}[{\citep[Def.~3]{correa:etal21}}]
    The set of (counterfactual) ancestors of $Y_{\*x}$ w.r.t.\ graph $\cG$, denoted $An(Y_{\*x})$, consist of each $W_{\*z}$ such that $W \in An(Y)_{\cG_{\underline{X}}}$ (which includes $Y$ itself), and $\*z = \*x \cap An(W)_{G_{\overline{\*X}}}$.
    \hfill $\blacksquare$
\end{definition}

\begin{fact}[{\citep[Thm.~1]{correa:etal21}}]
    \label{fact:ancestral-factorization}
    Let $\*W_*$ be an ancestral set, that is, $An(\*W_*) = \*W_*$, and let $\*w_*$ be a vector with a value for each variable in $\*W_*$. Then,
    \begin{equation*}
        P(\*W_* = \*w_*) = P\left(\bigwedge_{W_{\*t} \in \*W_*} W_{\pai{W}} = w \right),
    \end{equation*}
    where each $w$ is $w_t$ and $\pai{W}$ is determined for each $W_t \in \*W_*$ as follows:
    \begin{enumerate}[label=(\roman*)]
        \item the values for variables in $\Pai{W} \cap \*T$ are the same as in $\*t$, and
        \item the values for variables in $\Pai{W} \setminus \*T$ are taken from $\*w_*$, corresponding to the parents of $W$.
    \end{enumerate}
    \hfill $\blacksquare$
\end{fact}

For the proofs in this work, we leverage a key concept of counterfactuals which we define as functional counterfactuals.

\begin{definition}[Functional Counterfactuals]
    \label{def:function-ctf}
    For any SCM $\cM = \langle \*U, \*V, \cF, P(\*U) \rangle$, denote
    \begin{equation}
        \label{eq:function-ctf}
        \*F = \left\{V_{i[\pai{V_i}^{(j)}]}: V_i \in \*V, \pai{V_i}^{(j)} \in \cD_{\Pai{V_i}}\right \}
    \end{equation}
    as the functional counterfactual set of $\cM$, a set of counterfactual variables containing each variable intervened on every possible instantiation of its parents. Denote $\*f$ and $\cD_{\*F}$ as its instantiation and domain respectively\footnote{We intentionally use the notation of $\*F$ because this set of counterfactual quantities is heavily related to the functions of the SCM, $\cF$. Notably, the behavior of a function $f_{V_i} \in \cF$ w.r.t. $\*U$ can be specified fully by the joint counterfactual $P\left(\bigwedge_{\pai{V_i}^{(j)}} V_{i[\pai{V_i}^{(j)}]}\right)$. This property is leveraged in the proof of Lem.~\ref{lem:f-ctf-completeness}.}. If $\Pai{V_i} = \emptyset$, then $V_i \in \*F$ with no intervention.
    \hfill $\blacksquare$
\end{definition}

The idea behind this definition is to establish a standard family of counterfactual quantities which generalizes all other counterfactuals. We will see the power of functional counterfactuals through the following lemmas.

\begin{lemma}[Functional Counterfactual Uniqueness]
    \label{lem:u-f-uniqueness}
    Let $\cM = \langle \*U, \*V, \cF, P(\*U) \rangle$ be an SCM with functional counterfactual set $\*F$. Let $\cD_{\*U}(\*f) \subseteq \cD_{\*U}$ be the set of values of $\*U$ such that for every $\*u \in \cD_{\*U}(\*f)$, we have $\*F = \*f$ when evaluating $\cM$ with $\*U = \*u$. Then, for any $\*u \in \cD_{\*U}$, there exists a unique $\*f \in \cD_{\*F}$ such that $\*u \in \cD_{\*U}(\*f)$.
    \hfill $\blacksquare$

    \begin{proof}
        Note that for any $\*u \in \cD_{\*U}$, we can construct the value of $\*f$ such that $\*u \in \cD_{\*U}(\*f)$ as follows. For every $\*V_i \in \*V$ and $\pai{V_i}^{(j)} \in \cD_{\Pai{V_i}}$, choose $v_i^{(j)} = f_{V_i}(\pai{V_i}^{(j)}, \*u)$. Collectively, these values of $v_i^{(j)}$ can be used to form $\*f$. It is clear that $\*u \in \cD_{\*U}(\*f)$ because $V_{i[\pai{V_i}^{(j)}]}(\*u) = f_{V_i}(\pai{V_i}^{(j)}, \*u) = v_i^{(j)}$ for all $i$ and $j$, implying that $\*F = \*f$. Further, $\*f$ is unique, since $V_{i[\pai{V_i}^{(j)}]}(\*u)$ can only be equal to one unique deterministic value once $\*u$ is fixed.
    \end{proof}
\end{lemma}

\begin{lemma}[Functional Counterfactual Completeness]
    \label{lem:f-ctf-completeness}
    Let $\cM = \langle \*U, \*V, \cF, P(\*U) \rangle$, $\cM' = \langle \*U', \*V, \cF', P(\*U') \rangle$ be two SCMs both defined over $\*V$ and with the same functional counterfactual set $\*F$. Then $P^{\cM}(\*F = \*f) = P^{\cM'}(\*F = \*f)$ for all $\*f \in \cD_{\*F}$ if and only if $\cL_3(\cM) = \cL_3(\cM')$.
    \hfill $\blacksquare$

    \begin{proof}
        The backward direction of this proof is trivial since all functional counterfactuals belong to the set of all counterfactuals (i.e.\ $P^{\cM}(\*F = \*f) \in \cL_3(\cM)$ and $P^{\cM'}(\*F = \*f) \in \cL_3(\cM')$ for all $\*f \in \cD_{\*F}$). Therefore, $\cL_3(\cM) = \cL_3(\cM')$ implies $P^{\cM}(\*F = \*f) = P^{\cM'}(\*F = \*f)$ for all $\*f \in \cD_{\*F}$.

        To prove the forward direction, we must show that $P^{\cM}(\*F = \*f) = P^{\cM'}(\*F = \*f)$ for all $\*f \in \cD_{\*F}$ implies $\cL_3(\cM) = \cL_3(\cM')$. Consider any arbitrary $\cL_3$ query from $\cM$,
        \begin{equation*}
            Q = P^{\cM}(\*y_{1[\*x_1]}, \*y_{2[\*x_2]}, \dots) \in \cL_3(\cM).
        \end{equation*}
        Denote $Q'$ as the equivalent value from $\cM'$. Denote $\*Y_*$ as the set of all of the counterfactual terms of $Q$. For each term $\*y_{i[\*x_i]}$, consider its ancestral set $An(\*y_{i[\*x_i]})$, and denote $An(\*Y_*)$ as the union of all of these sets. Then note that
        \begin{equation*}
            Q = \sum_{\*a \in {An(\*Y_*) \setminus \*Y_*}} P^{\cM}(An(\*y_{1[\*x_1]}), An(\*y_{2[\*x_2]}), \dots),
        \end{equation*}
        that is, $Q$ is equal to the joint distribution of its ancestral set with all of the terms not in the original query marginalized away. Then, by Fact \ref{fact:ancestral-factorization}, we have
        \begin{equation}
            \label{eq:f-ctf-factorization}
            Q = \sum_{\*a \in An(\*Y_*) \setminus \*Y_*} P^{\cM} \left(\bigwedge_{W_{\*t} \in An(\*Y_*)} W_{\pai{W}} = w \right),
        \end{equation}
        where $w$ and $\pai{w}$ are defined as specified in Fact \ref{fact:ancestral-factorization}. This ancestral set factorization leaves every term with a single variable under the intervention of its parents. If there are duplicate terms (i.e. $W_{\pai{W}}$ appears multiple times for the same $W$ and $\pai{W}$), then either they can be reduced to a single term if every value of $w$ is equal (since $p \wedge p = p$ for any proposition $p$), or, if $w$ is not equal for every term, then $Q = 0$ for both $\cM$ and $\cM'$ since $W_{\pai{W}}$ cannot be equal to two values at once.

        Finally, we note that the probability term is simply a marginalized quantity from the functional counterfactual, so we see that
        \begin{equation*}
            Q = \sum_{a \in An(\*Y_*) \setminus \*Y_*} \left( \sum_{\*f' \in \*F \setminus \*W_*} P^{\cM}(\*F = \*f) \right),
        \end{equation*}
        where $\*W_*$ refers to the set of $\*W_{\pai{W}}$ terms in Eq.~\ref{eq:f-ctf-factorization}, and $\*f$ is defined for each value of $V_{\pai{V}} \in \*F$ as $w$ if $V_{\pai{V}} \in \*W_*$ or the value from $\*f'$ otherwise.

        Therefore, since $P^{\cM}(\*F = \*f) = P^{\cM'}(\*F = \*f)$, this implies that $Q = Q'$. With this being true for all values of $Q \in \cL_3(\cM)$, this means that $\cL_3(\cM) = \cL_3(\cM')$.
    \end{proof}
\end{lemma}

\subsection{Proofs of Sec.~\ref{sec:abstract-ncm}}

The abstractions in this work follow the theory developed by \citet{beckers2019abstracting}. We first note that \citet{beckers2019abstracting} utilizes the idea of ``allowed interventions''. Specifically, for an SCM $\cM_L$ over variables $\*V_L$, the notation $\cI_L$ is used to indicate a set of interventions on $\*V_L$ that are ``allowed'' in $\cM_L$. This is relevant when defining interventions across abstractions, since not all interventions on the lower level will have a corresponding intervention on the higher level, as will be shown. Further, the notation $\cI_L^*$ is used to define the set of all possible interventions over $\*V_L$.

We use the following definitions, translated to use the notation in our work\footnote{Note that there can be at most one such possible intervention $\*X_H \gets \*x_H$ such that $\omega_{\tau}(\*X_L \gets \*x_L) = \*X_H \gets \*x_H$. It is possible that no such intervention exists, but for this work, we only consider cases where there $\omega_{\tau}(\*X_L \gets \*x_L)$ exists for all $\*X_L \gets \*x_L \in \cI_L$.}.

\begin{restatable}[{\citep[Def.~3.12]{beckers2019abstracting}}]{definition}{beckersnotation}
    Some relevant notation is defined as follows:
    \begin{itemize}
        \item Given a set of variables $\*V$, $\*X \subseteq \*V$, and $\*x \in \cD_{\*X}$, let $\rst(\*V, \*x) = \{\*v \in \cD_{\*V}: \*v \text{ is consistent with } \*x\}$.

        \item Given variables $\*V_L$ and $\*V_H$, mapping $\tau: \cD_{\*V_L} \rightarrow \cD_{\*V_H}$, and value set $\*T \subseteq \cD_{\*V_L}$, denote $\tau(\*T) = \{\tau(\*v_L) : \*v_L \in \*T\}$.
        
        \item Given allowed interventions $\cI_L$ and $\cI_H$ over $\*V_L$ and $\*V_H$ respectively, define $\omega_{\tau}: \cI_L \rightarrow \cI_H$ such that $\omega_{\tau}(\*X_L \gets \*x_L) = \*X_H \gets \*x_H$, where $\tau(\rst(\*V_L, \*x_L)) = \rst(\*V_H, \*x_H)$.
    \end{itemize}
\end{restatable}
%\beckersnotation*

\begin{restatable}[$\tau$-Abstraction {\citep[Def.~3.13]{beckers2019abstracting}}]{definition}{beckerstauabs}
    \label{def:tau-abs}
    Let $\cM_L = \langle \*U_L, \*V_L, \cF_L, P(\*U_L) \rangle$ and $\cM_H = \langle \*U_H, \*V_H, \cF_H, P(\*U_H)\rangle$ be two SCMs. Let $\cI_L$ and $\cI_H$ be the sets of allowed interventions respectively. Given $\tau: \cD_{\*V_L} \rightarrow \cD_{\*V_H}$, we say that $(\cM_H, \cI_H)$ is a $\tau$-abstraction of $(\cM_L, \cI_L)$ if:
    \begin{enumerate}
        \item $\tau$ is surjective;
        \item There exists surjective $\tau_{\*U}: \cD_{\*U_L} \rightarrow \cD_{\*U_H}$ that is compatible with $\tau$, i.e.
        \begin{equation}
            \label{eq:tau-u-compatibility}
            \tau(\cM_{L[\*X_L \gets \*x_L]}(\*u_L)) = \cM_{H[\omega_{\tau}(\*X_L \gets \*x_L)]}(\tau_{\*U}(\*u_L)),
        \end{equation}
        for all $\*u_L \in \cD_{\*U_L}$ and all $(\*X_L \gets \*x_L) \in \cI_L$;
        \item $\cI_H = \omega_{\tau}(\cI_L)$.
    \end{enumerate}
    \hfill $\blacksquare$
\end{restatable}
%\beckerstauabs*

Further, we will assume that if $(\cM_H, \cI_H)$ is a $\tau$-abstraction of $(\cM_L, \cI_L)$, then $P(\*U_H) = \tau_{\*U}(P(\*U_L)) = P(\tau_{\*U}(\*U_L))$, that is, the distribution of $P(\*U_H)$ can be obtained from $P(\*U_L)$ via the push-forward measure through $\tau_{\*U}$. While it is not explicitly stated in the definition, this property aligns with the intention of linking the spaces of $\*U_L$ and $\*U_H$ through $\tau_{\*U}$.

\begin{restatable}[Strong $\tau$-Abstraction {\citep[Def.~3.15]{beckers2019abstracting}}]{definition}{beckersstrongtauabs}
    \label{def:strong-tau-abs}
    We say that $\cM_H$ is a strong $\tau$-abstraction of $\cM_L$ if $(\cM_H, \cI_H)$ is a $\tau$-abstraction of $(\cM_L, \cI_L)$ and $\cI_H = \cI_H^*$.
    \hfill $\blacksquare$
\end{restatable}

\begin{restatable}[Constructive $\tau$-Abstraction {\citep[Def.~3.19]{beckers2019abstracting}}]{definition}{beckersconstauabs}
    \label{def:cons-tau-abs}
    $\cM_H$ is a constructive $\tau$-abstraction of $\cM_L$ if $\cM_H$ is a strong $\tau$-abstraction of $\cM_L$, and there exists a partition of $\*V_L$, $\bbC = \{\*C_1, \*C_2, \dots, \*C_{n+1}\}$ (where $n = |\*V_H|$) with nonempty $\*C_1$ to $\*C_n$, such that $\tau$ can be decomposed as $\tau = (\tau_{\*C_1}, \tau_{\*C_2}, \dots, \tau_{\*C_n})$, where each $\tau_{\*C_i} : \cD_{\*C_i} \rightarrow \cD_{V_{H, i}}$ maps the $i$th partition to the $i$th variable of $\*V_H$.
    \hfill $\blacksquare$
\end{restatable}
%\beckersconstauabs*

In typical causal inference settings where nature is modeled by an SCM, every possible intervention is well-defined. In practice, some interventions may not intuitively correspond to an explicit action. For example, in a medical dataset, perhaps cholesterol level is measured from each person in the dataset. It may not make sense to consider interventions on cholesterol level, since it is not clear how, in practice, one would fix or change someone's cholesterol level to specific values. Nonetheless, it is still possible to theoretically study the effects of such an intervention through the semantics of SCMs. Therefore, the notion of ``allowed interventions'' is not typically discussed outside of works of abstractions as every intervention is well-defined.

This no longer holds true in the discussion of abstractions, which is why definitions of abstractions like Def.~\ref{def:tau-abs} explicitly require the allowed interventions to be specified. Notably, since every intervention in an SCM involves the creation of a submodel through the mutilation procedure, an intervention on a lower level SCM may not have an obvious counterpart in the higher level SCM. For an in depth understanding of the intricacies of this, we defer readers to read \citet{beckers2019abstracting}. In this work, given that one would like to use the higher level SCM $\cM_H$ for downstream inference tasks, we provide the maximum possible flexibility and assume that $\cI_H = \cI_H^*$. 

Defining $\cI_L$ is somewhat trickier, since in any nontrivial abstraction, there are cases where an intervention on the lower level does not have an equivalent intervention on the higher level. To decide on a criteria of which interventions are allowed on the lower level model, consider the following lemmas.

\begin{lemma}
    \label{lem:omega-tau-connection}
    Let $\tau$ be a constructive abstraction function w.r.t.~$\bbC$ and $\bbD$. If $\*X_L$ is a union of clusters in $\bbC$ (that is, there exists $\bbC' \subseteq \bbC$ such that $\*X_L = \bigcup_{\*C_i \in \bbC'} \*C_i$), then $\omega_{\tau}(\*X_L \gets \*x_L)$ exists and is equal to $\tau(\*X_L) \gets \tau(\*x_L)$.
    \hfill $\blacksquare$

    \begin{proof}
        If $\*X_L$ is a union of clusters in $\bbC$, then without loss of generality, suppose it can be decomposed as $\*X_L = (\*C_1, \*C_2, \dots, \*C_k)$, and $\*V_L \setminus \*X_L = (\*C_{k + 1}, \*C_{k+2}, \dots, \*C_n)$. By Def.~\ref{def:tau}, we can then compute, for any $\*v_L \in \rst(\*V_L, \*x_L)$,
        \begin{align*}
            \tau(\*v_L) &= \tau((\*x_L, \*v_L \setminus \*x_L)) \\
            &= \tau((\*c_1, \dots, \*c_{k}, \*c_{k+1}, \dots, \*c_n)) \\
            &= (\tau_{\*C_1}(\*c_1), \dots, \tau_{\*C_k}(\*c_k), \tau_{\*C_{k+1}}(\*c_{k+1}), \dots, \tau_{\*C_n}(\*c_n)) \\
            &= (\tau(\*x_L), \tau(\*v_L \setminus \*x_L)) \in \rst(\tau(\*V_L), \tau(\*x_L)).
        \end{align*}
        Moreover, for any $\*v_H \in \rst(\tau(\*V_L), \tau(\*x_L))$, there exists $\*v_L \in \rst(\*V_L, \*x_L)$ such that $\tau(\*v_L) = \*v_H$. Specifically, if $\*v_H = (\cD^{j_1}_{\*C_1}, \cD^{j_2}_{\*C_2}\dots, \cD_{\*C_n}^{j_n})$, then any $\*v_L = (\*c_1, \*c_2, \dots, \*c_n)$ satisfies this relationship if $\*c_i \in \cD_{\*C_i}^{j_i}$ for all $i$. Hence, by definition, $\omega_{\tau}(\*X_L \gets \*x_L)$ exists and is equal to $\tau(\*X_L) \gets \tau(\*x_L)$.
    \end{proof}
\end{lemma}

\begin{lemma}
    \label{lem:omega-surjectivity}
    Let $\tau$ be a constructive abstraction function over $\bbC$ and $\bbD$. For any intervention $\*X_H \gets \*x_H \in \cI_H^*$, there exists $\*X_L$ such that $\*X_L$ is a union of clusters of $\bbC$, and $\omega_{\tau}(\*X_L \gets \*x_L) = \*X_H \gets \*x_H$.
    \hfill $\blacksquare$

    \begin{proof}
        Without loss of generality, suppose that $\*X_H = \{X_1, X_2, \dots, X_k\}$, corresponding to clusters $\{\*C_1, \*C_2, \dots, \*C_k\}$ respectively. Then, choose $\*X_L = \bigcup_{i=1}^k \*C_i$, a union of clusters. The proof holds if for each $\*C_i$, there exist values $\*c_i \in \cD_{\*C_i}$ such that $\tau_{\*C_i}(\*c_i) = x_i$. Note that by definition, $x_i$ corresponds to some $\cD^j_{\*C_i}$ such that $\tau_{\*C_i}(\*c_i) = x_i$ for all $\*c_i \in \cD^j_{\*C_i}$. As $\cD^j_{\*C_i}$ is an element of a partition of $\cD_{\*C_i}$, it must be nonempty. Hence, the claim holds, and $\*x_L$ can be constructed by taking the one such value $\*c_i$ for each $i$. Lemma \ref{lem:omega-tau-connection} can then be used to show that $\omega_{\tau}(\*X_L \gets \*x_L) = \tau(\*X_L) \gets \tau(\*x_L)$, concluding the proof.
    \end{proof}
\end{lemma}

In other words, by Lemma \ref{lem:omega-tau-connection}, an intervention on a union of intervariable clusters will always have an intuitive well-defined corresponding intervention. Moreover, Lemma \ref{lem:omega-surjectivity} shows that all high level interventions are accounted for on the lower level model\footnote{There are some contrived settings in which $\omega_{\tau}(\*X_L \gets \*x_L)$ still exists even when $\*X_L$ is not a union of clusters, but it is inconsequential to omit these cases since all possible high level interventions are covered by Lemma \ref{lem:omega-surjectivity}.}. Hence, we make the assumption that $\cI_L$ is defined such that $\*X_L \gets \*x_L \in \cI_L$ if and only if $\*X_L$ is a union of clusters of $\bbC$. If, for some reason, this choice of $\cI_L$ is not desirable for some application, it may indicate that the choice of clusters, $\bbC$ and $\bbD$, should be revised. Under these assumptions, the work in the main text can be presented without the need to explicitly consider allowed interventions, simplifying the discussion.

We now show the main connection between this work and established works by proving Prop.~\ref{prop:abs-connect}.

\absconnect*

\begin{proof}
    For this proof, define $\cM_L = \langle \*U_L, \*V_L, \cF_L, P(\*U_L) \rangle$ and $\cM_H = \langle \*U_H, \*V_H, \cF_H, P(\*U_H) \rangle$.

    We first show the forward direction: if $\cM_H$ is $\cL_3$-$\tau$ consistent with $\cM_L$, then there exists SCM $\cM_H'$ such that $\cL_3(\cM_H') = \cL_3(\cM_H)$ and $\cM_H'$ is a constructive $\tau$-abstraction of $\cM_L$.

    We will start this proof by first attempting to show that $\cM_H$ itself is a constructive $\tau$-abstraction of $\cM_L$. We first note that the intervariable clustering $\bbC$ is defined to be a partition of a subset of $\*V_L$, and $\tau$ is constructed by Def.~\ref{def:tau}, which is specifically defined to be decomposed into subfunctions $\tau_{\*C_i}$ mapping cluster $\*C_i$ to a corresponding $V_{H, i}$ for each $i \in \{1, \dots, n\}$. If we further define $\*C_{n+1} = \*V_L \setminus \bigcup_{\*C_i \in \bbC} \*C_i$ (variables that are projected out through the abstraction), then we see that $\{\*C_1, \*C_2, \dots, \*C_{n+1}\}$ forms a partition over $\*V_L$.

    What remains is to show that $\cM_H$ is a strong $\tau$-abstraction of $\cM_L$. As $\cI_H = \cI_H^*$ by assumption, this reduces to showing that it is a $\tau$-abstraction. We show that the three requirements of Def.~\ref{def:tau-abs} hold:

    \begin{enumerate}
        \item Consider any $\*v_H \in \cD_{\*V_H}$. For all $V_{H, i} \in \*V_H$, note that $V_{H, i}$ must correspond to some $\*C_i$, and $v_{H, i}$ must correspond to some $\cD^j_{\*C_i}$ by definition. Since $\cD^j_{\*C_i}$ is part of a partition, it must be nonempty, so $\tau_{\*C_i}(\*c_i) = v_{H, i}$ for any choice of $\*c_i \in \cD^j_{\*C_i}$. Hence, $\tau_{\*C_i}$ is surjective for all $i$, implying that $\tau$ is surjective as it is simply a collection of all $\tau_{\*C_i}$.

        \item Consider the functional counterfactual set $\*F_H$ of $\cM_H$ (as defined in Eq.~\ref{eq:function-ctf}). For every $\*f_H \in \cD_{\*F_H}$, consider the counterfactual quantity
        \begin{equation*}
            \begin{split}
                & P(\*F_H = \*f_H) = \\
                & P\left(\bigwedge_{V_{H,i} \in \*V_H} \bigwedge_{\pai{V_{H,i}}^{(j)} \in \cD_{\Pai{V_{H,i}}}} V_{H,i\left[\pai{V_{H_i}}^{(j)}\right]} = v_{H,i}^{(j)}\right).
            \end{split}
        \end{equation*}
        By Eq.~\ref{eq:q-tau-consistency}, this is equal to
        \begin{equation*}
            \begin{split}
                & P(\*F_H = \*f_H) = \\
                & \sum_{\forall i, j \*c_{L, i}^{(j)} \in \cD_{\*V_L} : \tau(\*c_{L, i}^{(j)}) = v_{H, i}^{(j)}} P\left(\bigwedge_{i, j} \*C_{L, i\left[\*x_{L, i}^{(j)}\right]} = \*c_{L, i}^{(j)}\right)
            \end{split}
        \end{equation*}
        for all choices of $\*x_{L, i}^{(j)}$ such that $\tau(\*x_{L, i}^{(j)}) = \pai{V_{H, i}}^{(j)}$. Recall that $\cD_{\*U_H}(\*f_H) \subseteq \cD_{\*U_H}$ is the set of values of $\*U_H$ such that $\*u_H \in \cD_{\*U_H}(\*f_H)$ if and only if $\*F_H = \*f_H$ when $\*U_H = \*u_H$. Similarly, define $\cD_{\*U_L}(\*f_H) \subseteq \cD_{\*U_L}$ as the set of values of $\*U_L$ such that $\*u_L \in \cD_{\*U_L}(\*f_H)$ if and only if $\bigwedge_{i,j}\*C_{L, i\left[ \*x_{L, i}^{(j)}\right]} = \*c_{L, i}^{(j)}$ when $\*U_L = \*u_L$.

        Note that there may exist $\*u_L \in \cD_{\*U_L}$ that do not belong to $\cD_{\*U_L}(\*f_H)$ for any choice of $\*f_H$. However, the total probability measure of all such cases must be 0, since $P(\*U_L \in \cD_{\*U_L}(\*f_H)) = P(\*U_H \in \cD_{\*U_H}(\*f_H))$, and $\{\cD_{\*U_H}(\*f_H) : \*f_H \in \cD_{\*F_H}, \cD_{\*U_L}(\*f_H) \neq \emptyset\}$ forms a partition over $\cD_{\*U_H}$ by Lemma~\ref{lem:u-f-uniqueness} and must therefore have probability 1. Hence, we can define $\*U'_L$ with domain $\cD_{\*U'_L}$ that is equivalent to $\cD_{\*U_L}$ but with these cases omitted. Correspondingly, we can define $\cF'$ as $\cF$ but excluding the outputs when $\*U_L$ takes a value not in $\cD_{\*U'_L}$. We can then define $\cM'_L = \langle \*U'_L, \*V_L, \cF'_L, P(\*U'_L) \rangle$, where $\cL_3(\cM'_L) = \cL_3(\cM_L)$ because only a measure zero portion of $P(\*U_L)$ was removed.

        Define $U'$ as a random variable with domain $\cD_{U'} = \{u'_{\*f_H} : \*f_H \in \cD_{\*F_H}, \cD_{\*U_L}(\*f_H) = \emptyset\}$, that is, a value for each choice of $\*f_H$ with a nonempty $\cD_{\*U_L}(\*f_H)$. Define $P(U' = u'_{\*f_H}) = P(\*F_H = \*f_H)$. Choose $\tau_{\*U}: \cD_{\*U'_L} \rightarrow \cD_{\*U'}$ such that $\tau_{\*U}(\*u'_L) = u'_{\*f_H}$ if and only if $\*u'_L \in \cD_{\*U_L}(\*f_H)$. Note that this function is surjective because $u'_{\*f_H} \in \cD_{U'}$ only if $\cD_{\*U_L}(\*f_H)$ is nonempty. Furthermore, $P(U' = u'_{\*f_H}) = P(\*F_H = \*f_H) = P(\*u_L \in \cD_{\*U_L}(\*f_H))$, so the probability distributions are consistent.

        It is not necessarily the case that property 2 of $\tau$-abstractions holds between $\cM_H$ and $\cM_L$, but we can create a new SCM $\cM_H' = \langle \*U'_H = \{U'\}, \*V_H, \cF'_H, P(\*U'_H) \rangle$ such that the property holds between $\cM'_H$ and $\cM'_L$. Define $\cF'_H = \{f'^H_{V_{H, i}} : V_{H, i} \in \*V_H\}$ such that each $f'^H_{V_{H, i}}(\pai{V_{H, i}}, u'_{\*f_H}) = f^H_{V_{H, i}}(\pai{V_{H, i}}, \*u_H)$ if and only if $\*u_H \in \cD_{\*U_H}(\*f_H)$. Given that $P(U' = u'_{\*f_H}) = P(\*F_H = \*f_H) = P(\*u_H \in \cD_{\*U_H}(\*f_H))$, this implies that $P^{\cM'_H}(\*F_H = \*f_H) = P^{\cM_H}(\*F_H = \*f_H)$, further implying that $\cL_3(\cM'_H) = \cL_3(\cM_H)$ by Lemma \ref{lem:f-ctf-completeness}.

        % Fix the two lemmas in the above paragraph

        We now show that $\tau(\cM'_{L[\*X_L \gets \*x_L]}(\*u'_L)) = \cM'_{H[\omega_{\tau}(\*X_L \gets \*x_L)]}(\tau_{\*U}(\*u'_L))$, for all $\*u'_L \in \cD_{\*U'_L}$ and all $(\*X_L \gets \*x_L) \in \cI_L$. For the rest of the proof, assume that any notation involving subscripts $L$ and $H$ refers to $\cM'_L$ and $\cM'_H$ rather than $\cM_L$ and $\cM_H$. By Lemma \ref{lem:omega-tau-connection}, we know that $\omega(\*X_L \gets \*x_L) = \tau(\*X_L) \gets \tau(\*x_L)$, so let $\*x_H = \tau(\*x_L)$. Assume on the contrary that there exists $\*u'_L \in \cD_{\*U'_L}$ and $(\*X_L \gets \*x_L) \in \cI_L$ such that this claim does not hold. Then there must exist $\*C_{L, i}$ (and $V_{H, i} = \tau(\*C_{L, i})$) such that $\tau(\*C_{L, i[\*x_L]}(\*u'_L)) \neq V_{H, i[\*x_H]}(\tau_{\*U}(\*u'_L)))$. Let $\*f_H$ be the value of $\*F_H$ such that this value of $\*u'_L \in \cD_{\*U_L}(\*f_H)$. We show using proof by induction that this poses a contradiction.

        There must be a topological ordering to $\*V_H$ as $\cM'_H$ is recursive. In the base case, assume that $\Pai{V_{H, i}} = \emptyset$. This means that $\tau(\*C_{L, i[\*x_L]}(\*u'_L)) = \tau(\*C_{L, i[\emptyset]}(\*u'_L))$ and $V_{H, i[\*x_H]}(\tau_{\*U}(\*u'_L))) = V_{H, i[\emptyset]}(\tau_{\*U}(\*u'_L))))$. Note that $\tau(\*C_{L,i[\emptyset]}(\*u'_L)) = \tau(\*c_{L, i}^{(j)}) = v_{H, i}^{(j)} = V_{H, i[\emptyset]}(\tau_{\*U}(\*u'_L))$ by definition of $\*f_H$, contradicting the claim that $\tau(\*C_{L, i[\*x_L]}(\*u'_L)) \neq V_{H, i[\*x_H]}(\tau_{\*U}(\*u'_L)))$.

        Now assume for the sake of induction that $\tau(\*C_{L, i'[\*x_L]}(\*u'_L)) = V_{H, i'[\*x_H]}(\tau_{\*U}(\*u'_L)))$ for all $V_{H, i'} \in \Pai{V_{H, i}}$. Note that $\tau(\*C_{L, i[\*x_L]}(\*u'_L)) = \tau(\*C_{L, i[\Pai{\*C_{L, i}[\*x_L]}(\*u'_L)]}(\*u'_L))$ and $V_{H, i[\*x_H]}(\tau_{\*U}(\*u'_L))) = V_{H, i[\Pai{V_{H, i}[\*x_H]}(\tau_{\*U}(\*u'_L))]}(\tau_{\*U}(\*u'_L))))$ due to Corol.~\ref{corol:ctf-parent-form}. However, $\tau(\Pai{\*C_{L, i}[\*x_L]}(\*u'_L)) = \Pai{V_{H, i}[\*x_H]}(\tau_{\*U}(\*u'_L))$ by the inductive hypothesis. This means that $\tau(\*C_{L, i[\Pai{\*C_{L, i}[\*x_L]}(\*u'_L)]}(\*u'_L)) = \tau(\*c_{L, i}^{(j)}) = v_{H, i}^{(j)} = V_{H, i[\Pai{V_{H, i}[\*x_H]}(\tau_{\*U}(\*u'_L))]}(\tau_{\*U}(\*u'_L))))$ from $\*f_H$, once again contradicting the claim that $\tau(\*C_{L, i[\*x_L]}(\*u'_L)) \neq V_{H, i[\*x_H]}(\tau_{\*U}(\*u'_L)))$.

        Therefore, it must be the case that $\tau(\cM'_{L[\*X_L \gets \*x_L]}(\*u'_L)) = \cM'_{H[\omega_{\tau}(\*X_L \gets \*x_L)]}(\tau_{\*U}(\*u'_L))$, for all $\*u'_L \in \cD_{\*U'_L}$ and all $(\*X_L \gets \*x_L) \in \cI_L$.

        \item $\cI_H = \omega_\tau(\cI_L)$ holds by the assumption of $\cI_L$ and Lemma \ref{lem:omega-surjectivity}.
    \end{enumerate}

    This completes the forward direction of the proof.

    We now show the backward direction: if there exists SCMs $\cM'_L$ and $\cM_H'$ such that $\cL_3(\cM'_L) = \cL_3(\cM_L)$, $\cL_3(\cM_H') = \cL_3(\cM_H)$, and $\cM_H'$ is a constructive $\tau$-abstraction of $\cM'_L$, then $\cM_H$ is $\cL_3$-$\tau$ consistent with $\cM_L$. It is sufficient to simply show that $\cM'_H$ is $\cL_3$-$\tau$ consistent with $\cM'_L$, since if $\cL_3(\cM'_L) = \cL_3(\cM_L)$ and $\cL_3(\cM_H') = \cL_3(\cM_H)$, this would also imply that $\cM_H$ is $\cL_3$-$\tau$ consistent with $\cM_L$.

    This can be proven by showing that Eq.~\ref{eq:q-tau-consistency} holds, that is
    \begin{equation*}
        \begin{split}
            & \sum_{\forall i \*y_{L, i} \in \cD_{\*Y_{L, i}} : \tau(\*y_{L, i}) = \*y_{H, i}} \hspace{-15mm} P^{\cM'_L}(\*y_{L,1[\*x_{L, 1}]}, \*y_{L,2[\*x_{L, 2}]}, \dots) \\
            &= P^{\cM'_H}(\*y_{H,1[\tau(\*x_{H, 1})]}, \*y_{H,2[\tau(\*x_{L, 2})]}, \dots)
        \end{split}
    \end{equation*}
    for all choices of $\*y_{H, i}$ and $\*x_{H, i}$. Denote $Q$ as the l.h.s.\ of the equation and $\tau(Q)$ as the r.h.s. Denote $\cD_Q \subset \cD_{\*U'_L}$ as the set of values of $\*u'_L$ such that $\bigwedge_i \*Y_{L, i[\*x_{L, i}]} = \*y_{L, i[\*x_{L, i}]}$ for all $\*y_{L, i}$ such that $\tau(\*y_{L, i}) = \*y_{H, i}$. Similarly, denote $\cD_{\tau(Q)} \subseteq \cD_{\*U'_H}$ as the set of values of $\*u'_H$ such that $\bigwedge_i \*Y_{H, i[\tau(\*x_{L, i})]} = \*y_{H, i[\tau(\*x_{L, i})]}$. Note that $Q = P(\*U'_L \in \cD_{Q})$ and $\tau(Q) = P(\*U'_H \in \cD_{\tau(Q)})$. We claim that $\*u'_L \in \cD_Q$ if and only if $\tau_{\*U}(\*u'_L) \in \cD_{\tau(Q)}$.

    By definition of constructive $\tau$-abstractions, there must exist $\tau_{\*U}$ such that $\tau(\cM_{L[\*X_L \gets \*x_L]}(\*u'_L)) = \cM_{H[\omega_{\tau}(\*X_L \gets \*x_L)]}(\tau_{\*U}(\*u'_L))$, for all $\*u'_L \in \cD_{\*U'_L}$ and all $(\*X_L \gets \*x_L) \in \cI_L$, and further that $P(\*U'_H) = P(\tau(\*U'_L))$ by assumption. This implies that $\tau(\*Y_{L, i[\*x_{L, i}]}(\*u'_L)) = \*Y_{H, i[\omega_{\tau}(\*x_{L, i})]}(\tau_{\*U}(\*u'_L)) = \*Y_{H, i[\tau(\*x_L)]}(\tau_{\*U}(\*u'_L))$ for all $i$ by Lemma \ref{lem:omega-tau-connection}. Hence, if $\*Y_{L, i[\*x_i]}(\*u'_L) = \*y_{L, i}$, where $\tau(\*y_{L, i}) = \*y_{H, i}$, then $\*Y_{H, i[\tau(\*x_L)]}(\tau_{\*U}(\*u'_L)) = \*y_{H, i}$. Considering this for all values of $i$, it must be the case that $\*u'_L \in \cD_Q$ if and only if $\tau_{\*U}(\*u'_L) \in \cD_{\tau(Q)}$.

    Since $P(\*U'_H) = P(\tau(\*U'_L))$, this implies that $\tau(Q) = P(\*U'_H \in \cD_{\tau(Q)}) = P(\*U'_L \in \cD_{Q}) = Q$, concluding the proof.
\end{proof}

We note that the theorem does not claim $\cL_3$-$\tau$ consistency is equivalent to constructive $\tau$-abstractions, rather making a weaker claim that there must exist a pair of (potentially different) SCMs that are $\cL_3$-equivalent and fit the definition of a constructive $\tau$-abstraction. The reason is that, in fact, the definition for constructive $\tau$-abstraction is stricter than $\cL_3$-$\tau$ consistency, but as evident in the proof, the only restriction is on the domains of $\*U_L$ and $\*U_H$. In the case of $\*U_L$, there may be a measure zero portion of the domain $\cD_{\*U_L}$ that do not translate to the higher level functional counterfactuals, and in the case of $\*U_H$, it is possible that the space of $\cD_{\*U_H}$ may not allow for $\tau_{\*U}$ to be surjective. However, the theorem still essentially states that they are equivalent, at least on the three levels of the PCH. We therefore argue that, given the unobserved nature of the exogenous variables and the generating SCM, these two concepts are equivalent on a practical level.

We now prove that Alg.~\ref{alg:map-abstraction} successfully returns an $\cL_3$-$\tau$ consistent model.

\mapabstractioncorrect*

\begin{proof}
    We can show this result by first showing that the output of Alg.~\ref{alg:map-abstraction}, $\cM_H$, is a constructive $\tau$-abstraction of $\cM_L$. By Def.~\ref{def:tau}, it is clear by construction that $\tau$ can be decomposed as $\tau = (\tau_{\*C_1}, \tau_{\*C_2}, \dots, \tau_{\*C_n})$, and each $\tau_{\*C_i}$ maps the $i$th partition to the $i$th variable of $\*V_H$, as established on line 2. Hence, we must simply show that $\cM_H$ is a $\tau$-abstraction of $\cM_L$. Out of the three properties of Def.~\ref{def:tau-abs}, properties 1 and 3 can be proven similarly to how it is done in the forward direction of Prop.~\ref{prop:abs-connect}. Then, to prove property 2, we must simply show that there exists surjective $\tau_{\*U}: \cD_{\*U_L} \rightarrow \cD_{\*U_H}$ such that $P(\*U_H) = \tau_{\*U}(P(\*U_L)) = P(\tau_{\*U}(\*U_L))$, and $\tau(\cM_{L[\*X_L \gets \*x_L]}(\*u_L)) = \cM_{H[\omega_{\tau}(\*X_L \gets \*x_L)]}(\tau_{\*U}(\*u_L))$, for all $\*u_L \in \cD_{\*U_L}$ and all $(\*X_L \gets \*x_L) \in \cI_L$.

    The choice of $\tau_{\*U}$ is simple---we can use the identity function as by construction in line 1, $\*U_H = \*U_L$ and $P(\*U_H) = P(\*U_L)$. Further, note that $\omega_{\tau}(\*X_L \gets \*x_L) = \tau(\*X_L) \gets \tau(\*x_L)$ by Lemma \ref{lem:omega-tau-connection}. Hence, we must simply show that $\tau(\cM_{L[\*X_L \gets \*x_L]}(\*u_L)) = \cM_{H[\tau(\*X_L) \gets \tau(\*x_L)]}(\*u_L)$.
    
    We start by showing that $V_{i[\tau(\*x_L)]}(\*u_L) = \tau(\*C_{i[\*x_L]}(\*u_L))$ for every $V_i \in \*V_H$ and corresponding $\*C_i \in \bbC$. Note that for every $V_i \in \*V_H$, line 5 dictates that
    \begin{equation*}
        V_i \gets f_i^H(\pai{V_i}, \ui{V_i}) = \tau \left(f_V^L(\tildepai{V}, \ui{V}) : V \in \*C_i \right),
    \end{equation*}
    which implies that
    \begin{equation}
        \label{eq:abs-cons-alg-line5}
        V_{i[\tau(\*x_L)]}(\*u_L) = \tau \left(f_V^L \left(\tildepai{V}, \ui{V} \right) : V \in \*C_i \right),
    \end{equation}
    where $\*u_V$ is compatible with $\*u_L$ and $\tildepai{V}$ refers to any value $\pai{V} \in \cD_{\Pai{V}}$ such that $\tau(\pai{V}) = \Pai{V_i[\tau(\*x_L)]}(\*u_L)$. Note that since $\cM_L$ satisfies the abstract invariance condition (AIC) w.r.t.\ $\tau$, the value of $f^L_V$ will not change based on this choice of $\pai{V}$.

    We continue using proof by induction. Since $\cM_H$ is recursive, this implies there is a topological ordering of the functions of $\cM_H$. In the base case, if $\Pai{V_i} = \emptyset$, then 
    \begin{equation*}
        V_{i[\tau(\*x_L)]}(\*u_L) = V_{i[\emptyset]}(\*u_L) = \tau \left(f_V^L(\ui{V}) : V \in \*C_i \right) = \tau(\*C_i(\ui{L})),
    \end{equation*}
    which aligns with the claim. For the inductive hypothesis, assume that $V_{j[\tau(\*x_L)]}(\*u_L) = \tau(\*C_{j[\*x_L]}(\*u_L))$ for every $V_j \in \Pai{V_i}$. Then we have
    \begin{align*}
        V_{i[\tau(\*x_L)]}(\*u_L) &= V_{i[\Pai{V_i [\tau(\*x_L)]}(\*u_L)]}(\*u_L) \\
        & \text{ by Corol.~\ref{corol:ctf-parent-form}} \\
        &= V_{i[\tau(\Pai{\*C_i[\*x_L]}(\*u_L))]}(\*u_L) \\
        & \text{ by inductive hypothesis} \\
        &= \tau \left(f_V^L(\Pai{\*C_i[\*x_L]}(\*u_L), \ui{V}) : V \in \*C_i \right) \\
        & \text{ by Eq.~\ref{eq:abs-cons-alg-line5}} \\
        &= \tau(\*C_{i[\*x_L]}(\ui{L})) \\
        & \text{ simplification,}
    \end{align*}
    proving the claim. Since $V_{i[\tau(\*x_L)]}(\*u_L) = \tau(\*C_{i[\*x_L]}(\*u_L))$ for every $V_i \in \*V_H$, this implies that $\tau(\cM_{L[\*X_L \gets \*x_L]}(\*u_L)) = \cM_{H[\omega_{\tau}(\*X_L \gets \*x_L)]}(\tau_{\*U}(\*u_L))$, proving that $\cM_H$ is a constructive $\tau$-abstraction of $\cM_L$. Finally, Prop.~\ref{prop:abs-connect} proves that $\cM_H$ must therefore be an $\cL_3$-$\tau$ abstraction of $\cM_L$.
\end{proof}

We leverage this property to prove that the abstract invariance condition (AIC) is necessary and sufficient for the existence of an abstraction.

\absconditions*

\begin{proof}
    If $\cM'_L$ satisfies the AIC w.r.t.~$\tau$, then one can use Alg.~\ref{alg:map-abstraction} to obtain an example of a model that is $\cL_3$-$\tau$ consistent with $\cM'_L$ (and therefore with $\cM_L$), as proven in Prop.~\ref{prop:map-abstraction}.

    We now consider the other direction. If $\cM_H$ is $\cL_3$-$\tau$ consistent with $\cM_L$, then by Prop.~\ref{prop:abs-connect}, there must exist some $\cM'_L$ and $\cM'_H$ such that $\cL_3(\cM'_L) = \cL_3(\cM_L)$, $\cL_3(\cM'_H) = \cL_3(\cM_H)$ and $\cM'_H$ is a constructive $\tau$-abstraction of $\cM'_L$. For the rest of this proof, assume that any terms with a subscripts of $L$ or $H$ refer to $\cM'_L$ or $\cM'_H$ instead of $\cM_L$ and $\cM_H$. The constructive $\tau$-abstraction property implies that there exists surjective $\tau_{\*U}: \cD_{\*U'_L} \rightarrow \cD_{\*U'_H}$ such that $\tau(\cM'_{L[\*X_L \gets \*x_L]}(\*u'_L)) = \cM'_{H[\omega_{\tau}(\*X_L \gets \*x_L)]}(\tau_{\*U}(\*u'_L))$, for all $\*u'_L \in \cD_{\*U'_L}$ and all $(\*X_L \gets \*x_L) \in \cI_L$. By Lemma \ref{lem:omega-tau-connection}, this means $\tau(\cM'_{L[\*X_L \gets \*x_L]}(\*u'_L)) = \cM'_{H[\tau(\*X_L) \gets \tau(\*x_L)]}(\tau_{\*U}(\*u'_L))$. 

    Assume for the sake of contradiction that it is not the case that $\cM'_L$ satisfies the AIC w.r.t.~$\tau$. Then there must exist $\*v_1, \*v_2 \in \cD_{\*V_L}$ such that $\tau(\*v_1) = \tau(\*v_2)$, yet
    \begin{equation*}
        \begin{split}
            & \tau \left( \left( f^L_V(\pai{V}^{(1)}, \*u'_V): V \in \*C_i \right) \right) \\
            &\neq \tau \left( \left( f^L_V(\pai{V}^{(2)}, \*u'_V): V \in \*C_i \right) \right)
        \end{split}
    \end{equation*}
    for some value of $\*u'_L \in \cD_{\*U'_L}$ and $\*C_i \in \bbC$.

    This would imply that $\tau \left(\*C_{i[\pai{\*C_i}^{(1)}]}(\*u'_L)\right) \neq \tau \left(\*C_{i[\pai{\*C_i}^{(2)}]}(\*u'_L)\right)$, so $\tau \left(\cM'_{L[\pai{\*C_i}^{(1)}]}(\*u'_L)\right) \neq \tau \left(\cM'_{L[\pai{\*C_i}^{(2)}]}(\*u'_L)\right)$. By the $\tau$-abstraction definition, we have $\tau \left(\cM'_{L[\pai{\*C_i}^{(1)}]}(\*u'_L)\right) = \cM'_{H[\tau(\pai{\*C_i}^{(1)})]}(\tau_{\*U}(\*u'_L))$ and $\tau \left(\cM'_{L[\pai{\*C_i}^{(2)}]}(\*u'_L)\right) = \cM'_{H[\tau(\pai{\*C_i}^{(2)})]}(\tau_{\*U}(\*u'_L))$. However, since $\tau(\pai{\*C_i}^{(1)}) = \tau(\pai{\*C_i}^{(2)})$, this implies that $\tau \left(\cM'_{L[\pai{\*C_i}^{(1)}]}(\*u'_L)\right) = \tau \left(\cM'_{L[\pai{\*C_i}^{(2)}]}(\*u'_L)\right)$, contradicting the earlier statement. Therefore, $\cM'_L$ must satisfy the AIC w.r.t.\ $\tau$, completing the proof.
\end{proof}

\subsection{Proofs of Sec.~\ref{sec:learning-abs}}

We start by noting the impossibility of performing causal inferences without additional assumptions, as implied by the Causal Hierarchy Theorem.

\begin{fact}[Causal Hierarchy Theorem (CHT) {\citep[Thm.~1]{bareinboim:etal20}}]
\label{fact:cht}
    Let $\Omega^*$ be the set of all SCMs. We say that Layer $j$ of the causal hierarchy for SCMs \emph{collapses} to Layer $i$ ($i < j$) relative to $\cM^* \in \Omega^*$ if $L_i(\cM^*) = L_i(\cM)$ implies that $L_j(\cM^*) = L_j(\cM)$ for all $\cM \in \Omega^*$. Then, with respect to the Lebesgue measure over (a suitable encoding of $L_3$-equivalence classes of) SCMs, the subset in which Layer $j$ of SCMs collapses to Layer $i$ is measure zero.
    \hfill $\blacksquare$ 
\end{fact}

Given this result, we note that the same principle applies to performing causal inferences across abstractions.

\begin{lemma}
    \label{lem:abs-map-completeness}
    Let $\Omega_L$ and $\Omega_H$ be the space of SCMs defined over $\*V_L$ and $\*V_H$ respectively, and let $\tau: \cD_{\*V_H} \rightarrow \cD_{\*V_L}$ be a constructive abstraction function defined over clusters $\bbC$ and $\bbD$. Let $\Omega'_L$ be the subset of $\Omega_L$ that satisfies the AIC. Define $\psi_{\tau}: \Omega'_L \rightarrow \Omega_H$ such that $\psi_{\tau}(\cM_L) = \cM_H$, where $\cM_H \in \Omega_H$ is $\cL_3$-$\tau$ consistent with $\cM_L$ (while there could be many such SCMs, they are all $\cL_3$-equivalent, so we can arbitrarily choose the output of Alg.~\ref{alg:map-abstraction} on inputs $\cM_L$ and $\tau$, which must exist due to Prop.~\ref{prop:map-abstraction}). Then, $\psi_{\tau}$ is surjective (i.e.~$\{\psi(\cM_L) : \cM_L \in \Omega'_L\} = \Omega_H$).
    \hfill $\blacksquare$

    \begin{proof}
        For any SCM $\cM_H = \langle \*U_H, \*V_H, \cF_H, P(\*U_H) \rangle \in \Omega_H$, one can construct SCM $\cM_L = \langle \*U_L, \*V_L, \cF_L, P(\*U_L) \rangle \in \Omega'_L$ such that $\cM_H$ is $\cL_3$-$\tau$ consistent as $\cM_L$ as follows:
        \begin{enumerate}
            \item Choose $\*U_L = \*U_H$ and $P(\*U_L) = P(\*U_H)$.
            \item For each $V_L \in \*V_L$, let $\*C \in \bbC$ be the intervariable cluster such that $V_L \in \*C$, and let $V_H = \tau(\*C)$. Define $\Pai{V_L} \subseteq \*V_L$ as the set of variables such that $\tau(\Pai{V_L}) = \Pai{V_H}$. Define $\Ui{V_L} \subseteq \*U_L$ as $\Ui{V_H}$.
            \item For all $\*C \in \bbC$, note that for any pair $V_1, V_2 \in \*C$, $\Pai{V_1} = \Pai{V_2}$ and $\Ui{V_1} = \Ui{V_2}$. Denote $V_H = \tau(\*C)$. For each $V \in \*C$, choose $f_{V_L}^L \in \cF_L$ arbitrarily such that $\tau(\{f^{L}_{V_L}(\pai{V}, \ui{V}) : V_L \in \*C\}) = f^{H}_{V_H}(\tau(\pai{V_L}), \ui{V_L})$. There must exist at least one such setting since for any possible input $\pai{V_H}, \ui{V_H}$ to $f_{V_H}^H$, there is at least one set of inputs $\pai{V_L}, \ui{V_L}$ to $f_{V_L}^L$ such that $\tau(\pai{V_L}) = \pai{V_H}$ (due to surjectivity of $\tau$) and $\ui{V_L} = \ui{V_H}$.
        \end{enumerate}

        One can easily verify that running Alg.~\ref{alg:map-abstraction} on this choice of $\cM_L$ will return $\cM_H$, implying that $\cM_H$ is $\cL_3$-$\tau$ consistent with $\cM_L$ and that $\psi_{\tau}$ is surjective.
    \end{proof}
\end{lemma}

\begin{customprop}{\ref{prop:abs-cht}}[Abstract Causal Hierarchy Theorem (Formal Version)]
%\begin{proposition}[Abstract Causal Hierarchy Theorem]
    %\label{prop:abs-cht}
    Let $\Omega_L$ and $\Omega_H$ be the space of models defined over $\*V_L$ and $\*V_H$ respectively, and let $\Omega'_L$ be the subset of $\Omega_L$ such that the abstract invariance condition holds. Let $\tau: \cD_{\*V_H} \rightarrow \cD_{\*V_L}$ be a constructive abstraction function. We say that Layer $j$ of the causal hierarchy for $\Omega_H$ $\tau$-collapses to Layer $i$ ($i < j$) relative to $\cM_L \in \Omega_L$ if $\cL_i$-$\tau$ consistency implies $\cL_j$-$\tau$ consistency of $\cM_H$ with $\cM_L$ for all $\cM_H \in \Omega_H$. Then, w.r.t.\ Lebesgue measure over (a suitable encoding of $\cL_3$-equivalence classes of) $\Omega'_L$, the subset in which Layer $j$ of $\Omega_H$ $\tau$-collapses to Layer $i$ has measure zero.
    \hfill $\blacksquare$
%\end{proposition}
\end{customprop}

\begin{proof}
    We first show that Layer $j$ of $\Omega_H$ $\tau$-collapses to Layer $i$ relative to $\cM_L \in \Omega'_L$ if and only if Layer $j$ of $\Omega_H$ collapses to Layer $i$ relative to $\psi_{\tau}(\cM_L) \in \Omega_H$ (as defined in Lem.~\ref{lem:abs-map-completeness}). If Layer $j$ of $\Omega_H$ $\tau$-collapses to Layer $i$ relative to $\cM_L$, then that implies that all SCMs in $\Omega_H$ that are $\cL_i$-$\tau$ consistent with $\cM_L$ are also $\cL_j$-$\tau$ consistent, including $\psi_{\tau}(\cM_L)$. This is only possible if they are all $\cL_i$- and $\cL_j$-consistent with each other, implying regular collapse relative to $\psi_{\tau}(\cM_L)$. Conversely, if Layer $j$ of $\Omega_H$ collapses to Layer $i$ relative to $\psi_{\tau}(\cM_L)$, then all SCMs in $\Omega_H$ that are $\cL_i$-consistent with $\psi_{\tau}(\cM_L)$ must also be $\cL_j$-consistent. By definition, $\psi_{\tau}(\cM_L)$ is $\cL_3$-$\tau$ consistent with $\cM_L$, so this implies that all SCMs in $\Omega_H$ that are $\cL_i$-$\tau$ consistent with $\cM_L$ are also $\cL_j$-$\tau$ consistent.

    Fact \ref{fact:cht} states that the subset of $\Omega_H$ in which Layer $j$ of $\Omega_H$ collapses to Layer $i$ is measure 0. Hence, the subset of $\Omega'_L$ (under the same encoding w.r.t.\ the set $\{\psi(\cM_L) : \cM_L \in \Omega'_L\} = \Omega_H$ (as proven in Lem.~\ref{lem:abs-map-completeness})) in which Layer $j$ of $\Omega_H$ $\tau$-collapses to Layer $i$ is also measure 0.
\end{proof}

As a consequence, causal assumptions are necessary to make causal inferences. For this work, we leverage cluster causal diagrams (C-DAGs), from Def.~\ref{def:cdag}.

For the following proofs, consider the classical definition of identifiability.
\begin{definition}
    \label{def:classic-id}
    Let $\Omega^*$ be the space containing all SCMs defined over endogenous variables $\*V$. We say that a causal query $Q$ is identifiable (ID) from the available data $\bbZ$ and the causal diagram $\cG$ if $Q(\cM_1) = Q(\cM_2)$ for every pair of models $\cM_1, \cM_2 \in \Omega^*$ such that $\cM_1$ and $\cM_2$ both induce $\cG$ and $\bbZ(\cM_1) = \bbZ(\cM_2)$.
    \hfill $\blacksquare$
\end{definition}

We can now prove that abstract identification is equivalent to classical identification on the higher level.

\dualabsid*

\begin{proof}
    Let $\Omega_L$ and $\Omega_H$ be the space of SCMs defined over $\*V_L$ and $\*V_H$ respectively, and let $\Omega_L(\cG_{\bbC})$ and $\Omega_H(\cG_{\bbC})$ be their corresponding subsets that induce graph $\cG_{\bbC}$. If $Q$ is $\tau$-ID from $\cG_{\bbC}$ and $\bbZ$, then every pair of $\cM_L \in \Omega_L(\cG_{\bbC}), \cM_H \in \Omega_H(\cG_{\bbC})$ such that $\cM_H$ is $\bbZ$-$\tau$ consistent with $\cM_L$ must have $\cM_H$ be $Q$-$\tau$ consistent with $\cM_L$. For all such $\cM_H$, $\bbZ$-$\tau$ consistency and $Q$-$\tau$ consistency with $\cM_L$ implies that $\cM_H$ is $\tau(\bbZ)$-consistent and $\tau(Q)$-consistent by Def.~\ref{def:q-tau-consistency}. For any pair $\cM_1, \cM_2 \in \Omega_H$ that induce $\cG_{\bbC}$, $\tau(\bbZ)(\cM_1) = \tau(\bbZ)(\cM_2)$ therefore implies that both $\cM_1$ and $\cM_2$ must be $\bbZ$-$\tau$ consistent with $\cM_L$ and must therefore both be $Q$-$\tau$ consistent, so $\tau(Q)(\cM_1) = \tau(Q)(\cM_2)$. Hence, $\tau(Q)$ is ID from $\cG_{\bbC}$ and $\tau(\bbZ)$ by Def.~\ref{def:classic-id}.

    Conversely, if $\tau(Q)$ is ID from $\cG_{\bbC}$ and $\tau(\bbZ)$, then for any $\cM_1, \cM_2 \in \Omega_H$ that induces $\cG_{\bbC}$ such that $\tau(\bbZ)(\cM_1) = \tau(\bbZ)(\cM_2)$, it must be the case that $\tau(Q)(\cM_1) = \tau(Q)(\cM_2)$. For every $\cM_L \in \Omega_L(\cG_{\bbC})$, Prop.~\ref{prop:map-abstraction} states that there exists some $\cM_H \in \Omega_H(\cG_{\bbC})$ that is $\cL_3$-$\tau$ consistent with $\cM_L$, implying that $\cM_H$ is both $\bbZ$-$\tau$ consistent and $Q$-$\tau$ consistent with $\cM_L$. Since all $\cM_H \in \Omega_H(\cG_{\bbC})$ that match in $\tau(\bbZ)$ must also match in $\tau(Q)$, it must be the case that all such $\cM_H$ that are $\bbZ$-$\tau$ consistent with $\cM_L$ must also be $Q$-$\tau$ consistent with $\cM_L$. Hence, by definition, $Q$ is $\tau$-ID from $\cG_{\bbC}$ and $\bbZ$.
\end{proof}

We also connect this result to the results of neural identification with NCMs.

\begin{definition}[Neural Counterfactual Identification {\citep[Def.~4]{xia:etal23}}]
    \label{def:ncm-l3-id}
    Consider an SCM $\cM^*$ and the corresponding causal diagram $\cG$. Let $\bbZ = \{P(\*V_{\*z_k})\}_{k=1}^{\ell}$ be a collection of available interventional (or observational if $\*Z_k = \emptyset$) distributions from $\cM^*$.
    The counterfactual query $P(\*Y_* = \*y_* \mid \*X_* = \*x_*)$ is said to be neural identifiable (identifiable, for short) from the set of $\cG$-constrained NCMs $\Omega(\cG)$ and $\bbZ$ if and only if $P^{\widehat{M}_1}(\*y_* \mid \*x_*) = P^{\widehat{M}_2}(\*y_* \mid \*x_*)$ for every pair of models $\widehat{M}_1, \widehat{M}_2 \in \Omega(\cG)$ s.t. they match $\cM^*$ on all distributions in $\bbZ$ (i.e. $\bbZ(\cM^*) = \bbZ(\cM_1) = \bbZ(\cM_2) > 0$).
    \hfill $\blacksquare$
\end{definition}

\begin{fact}[Counterfactual Graphical-Neural Equivalence (Dual ID) {\citep[Thm.~3]{xia:etal23}}]
    \label{fact:ncm-ctfid-equivalence}
    Let $\Omega^*, \Omega$ be the spaces including all SCMs and NCMs, respectively. Consider the true SCM $\cM^*$ and the corresponding causal diagram $\cG$. Let $Q = P(\*y_* \mid \*x_*)$ be the target query and $\bbZ$ the set of observational and interventional distributions available. Then, $Q$ is neural identifiable from $\Omega(\cG)$ and $\bbZ$ if and only if it is identifiable from $\cG$ and $\bbZ$.
    \hfill $\blacksquare$
\end{fact}

\begin{fact}[Neural Counterfactual Mutilation (Operational ID) {\citep[Corol.~1]{xia:etal23}}]
    \label{fact:neur-op-id}
    Consider the true SCM $\cM^* \in \Omega^*$, causal diagram $\cG$, a set of available distributions $\bbZ$, and a target query $Q$ equal to $P^{\cM^*}(\*y_* \mid \*x_*)$. Let $\hM \in \Omega(\cG)$ be a $\cG$-constrained NCM such that $\bbZ(\hM) = \bbZ(\cM^*)$. If $Q$ is identifiable from $\cG$ and $\bbZ$, then $Q$ is computable via Eq.~\ref{eq:def:l3-semantics} from $\hM$.
    \hfill $\blacksquare$
\end{fact}

The connection between abstract identification and neural identification follows naturally.

\absidncm*
\begin{proof}
    This is a direct consequence of Thm.~\ref{thm:dual-abs-id}, Fact \ref{fact:ncm-ctfid-equivalence}, and Fact \ref{fact:neur-op-id}.
\end{proof}

\absidcomplete*
\begin{proof}
    Lines 1-2 of Alg.~\ref{alg:ncm-solve-absid} constructs $\tau$ given $\bbC$ and $\bbD$. Lines 3-9 checks that $\tau(Q)$ is neural identifiable from $\widehat{\Omega(\cG_{\bbC})}$ and $\tau(\bbZ)$. Lines 4 and 5 find the two parameterizations $\theta^*_{\min}$ and $\theta^*_{\max}$ that minimize and maximize $\tau(Q)$ while simultaneously guaranteeing $\tau(\bbZ)$-consistency. Hence, if the two parameterizations result in the same value for $\tau(Q)$, then all such NCMs must match in $\tau(Q)$, guaranteeing neural identifiability. Otherwise, the two parameterizations provide the counterexample for two NCMs that do not match in $\tau(Q)$. Finally, Corol.~\ref{corol:abs-id-ncm} states that neural identifiability implies abstract identifiability.
\end{proof}

\subsection{Proofs of Sec.~\ref{sec:applications}}

We first start by showing the following result.

\begin{lemma}
    \label{lem:coarse-intra}
    For any choice of intravariable clusters $\bbD$ such that $\cM_L$ satisfies the AIC w.r.t.~the corresponding $\tau$, $\cM_L$ will also satisfy the AIC w.r.t.~any finer clustering $\bbD'$ (i.e.~for all $\bbD_{\*C_i} \in \bbD$ and all $\cD_{\*C_i}^{(j)} \in \bbD_{\*C_i}$, $\cD_{\*C_i}^{(j)}$ is a subset of some $\cD_{\*C_i}^{(j')} \in \bbD'_{\*C_i}$).
    \hfill $\blacksquare$

    \begin{proof}
        Fix intervariable clusters $\bbC$ and denote $\tau$ and $\tau'$ as the constructive abstraction function defined w.r.t.~$(\bbC, \bbD)$ and $(\bbC, \bbD')$ respectively. If $\cM_L$ satisfies that AIC w.r.t.~$\tau$, that implies that for all $\*v_1, \*v_2 \in \cD_{\*V_L}$ such that $\tau(\*v_1) = \tau(\*v_2)$, all $\*u \in \cD_{\*U_L}$, and all $\*C_i \in \bbC$,
        \begin{equation*}
            \begin{split}
                & \tau \left( \left( f^L_V(\pai{V}^{(1)}, \*u_V): V \in \*C_i \right) \right) \\
                &= \tau \left( \left( f^L_V(\pai{V}^{(2)}, \*u_V): V \in \*C_i \right) \right),
            \end{split}
        \end{equation*}
         where $\pai{V}^{(1)}$ and $\pai{V}^{(2)}$ are the values corresponding to $\*v_1$ and $\*v_2$ respectively. If $\bbD'$ is a finer clustering than $\bbD$, then that means that $\tau(\*v_1) = \tau(\*v_2)$ implies $\tau'(\*v_1) = \tau'(\*v_2)$ for all values $\*v_1, \*v_2 \in \cD_{\*V_L}$, implying that the above must also hold for $\tau'$.
    \end{proof}
\end{lemma}

This property implies that the constraints of $\bbD$ are one-sided, and although finding the most coarse set of clusters may be impossible, any finer set will also work. In the worst case, choosing $\bbD$ such that $\bbD_{\*C_i} = \cD_{\*C_i}$ (i.e.~every value in their own cluster) would still result in a valid abstraction, as shown below.

\fullintracluster*

\begin{proof}
    The first claim is directly implied by Lemma \ref{lem:coarse-intra}, since this choice of $\bbD$ simply clusters each value to its own cluster, resulting in a finer clustering than any other clustering.

    Without any additional information about $\cM_L$, it is possible for any other choice of clustering to result in $\cM_L$ failing to satisfy the AIC. This can be shown by constructing an adversarial example of $\cM_L$ for any other choice of intravariable clustering $\bbD'$. Since $\bbD' \neq \bbD$, this implies that there exists at least one pair of $\*c_1, \*c_2 \in \*C_i$ for some $\*C_i \in \bbC$ such that $\*c_1$ and $\*c_2$ are in the same cluster $\cD_{\*C_i}^{(j)} \in \bbD_{\*C_i}$ for some $j$.

    Let $\*C_k \in \bbC$ be a cluster such that there exists at least one pair $V_1 \in \*C_i, V_2 \in \*C_k$ such that $V_1 \in \Pai{V_2}$. Since there are no restrictions on $\cM_L$ aside from basic assumptions like recursiveness, we can construct one such that such a $\*C_k$ exists. Consider the set of functions $\mathcal{F}_{\*C_k} = \{f^{L}_{V} : V \in \*C_k\} \subseteq \cF_L$, from $\cM_L$, and denote $\cF_{\*C_k}(\*c_k, \*u_L) = (f^L_{V}(\*c_k, \*u_L) : V \in \*C_k)$. Provided that the domains of $\*V_H$ are nontrivial (each variable can take at least two values), there must exist $v_1, v_2 \in \cD_{V_k}$, where $V_k = \tau_{\*C_k}(\*C_k)$, such that $v_1 \neq v_2$, and there exists $\*c_{k, 1}, \*c_{k, 2} \in \cD_{\*C_k}$ such that $\tau_{\*C_k}(\*c_{k, 1}) = v_1$ and $\tau_{\*C_k}(\*c_{k, 2}) = v_2$. Hence, we can construct each function of $\cF_{\*C_k}$ such that for some setting of $\*U_L = \*u_L$, $\tau_{\*C_k}(\cF_{\*C_k}(\*c_1, \*u_L)) \neq \tau_{\*C_k}(\cF_{\*C_k}(\*c_2, \*u_L))$, violating the AIC.
\end{proof}

%\clearpage
\section{Background on Causal Abstractions}
\label{sec:related-work}

In this section, we discuss some of the prior works in causal abstractions \citep{rubenstein:etal17-causalsem, beckers2019abstracting, Beckers2019-BECACA-8, geiger2023causal, pmlr-v213-massidda23a}. In many established causal inference tasks, it is typically assumed that there is a well-specified and known set of variables of interest $\*V$, and nature is modeled by a collection of mechanisms that assign values to each of these variables. However, the definition of $\*V$ may not always be clear in practice. In particular, the variables of interest may not align with the features of the data. For example, in an economic system, perhaps data on each individual consumer is collected, but the variable of interest is an aggregate measure like gross domestic product (GDP). In image data, perhaps the pixel values are collected, but the variables of causal interest are related to the objects of the image, not the individual pixels.

Acknowledging that the data is not always provided in the best choice of granularity, existing works of causal abstractions typically define two sets of variables, $\*V_L$ and $\*V_H$, which describe the lower level and higher level settings, respectively. For example, $\*V_L$ might describe the pixels of an image, while $\*V_H$ might describe its structural content. They are typically modeled by corresponding causal models $\cM_L$ and $\cM_H$, respectively. In this section, we describe some relevant works in this context and will employ our notation for consistency purposes when their notation differs.

The connection between $\*V_H$ and $\*V_L$ can be described through a mapping, $\tau: \cD_{\*V_L} \rightarrow \cD_{\*V_H}$, between their domains. However, even if $\tau$ is known, it is not guaranteed that a model over $\*V_H$, $\cM_H$, is an abstraction of a model over $\*V_L$, $\cM_L$. In short, while $\tau$ connects the domains of the variables, there is nothing guaranteeing any kind of connection between the models $\cM_H$ and $\cM_L$, be it the functions, the exogenous noise, or the induced distributions.

One of the earliest works that formally discuss abstractions in the context of causal models is \citet{rubenstein:etal17-causalsem}, which establishes the idea of \emph{exact transformations}, where a high-level SCM $\cM_H$ could be considered an ``abstraction'' of a low-level SCM $\cM_L$ if $\cM_H$ is an exact $\tau$-transformation of $\cM_L$. In addition to $\tau$, which connects the domains of $\*V_H$ and $\*V_L$, exact transformations connect the two models $\cM_L$ and $\cM_H$ through their induced interventional distributions. This requires mapping the set of low level interventions $\cI_L$ (over $\*V_L$) to their corresponding high-level counterparts $\cI_H$ (over $\*V_H$), which is done through another function $\omega: \cI_L \rightarrow \cI_H$. This idea leads to the following definition.

\begin{definition}[Exact Transformation {\citep[Def.~3]{rubenstein:etal17-causalsem}}]
    \label{def:exact-transform}
    Let $\cM_L$ and $\cM_H$ be SCMs and $\tau: \cD_{\*V_L} \rightarrow \cD_{\*V_H}$ be a function. We say that $(\cM_H, \cI_H)$ is an exact $\tau$-transformation of $(\cM_L, \cI_L)$ if there exists a surjective order preserving map $\omega : \cI_L \rightarrow \cI_H$ such that
    \begin{equation}
        \label{eq:abstract-commute}
        P(\tau(\*V_{L[\*X_L = \*x_L]})) = P(\*V_{H[\omega(\*X_L = \*x_L)]}).
    \end{equation}
    \hfill $\blacksquare$
\end{definition}

In this definition, the variables and corresponding distributions of $\cM_L$ are linked to those of $\cM_H$ through the function $\tau$, and corresponding causal interventions are linked through the function $\omega$.

The interventional sets $\cI_L$ and $\cI_H$ are called the ``allowed'' interventions of $\cM_L$ and $\cM_H$, respectively. They can be specified to contain any possible intervention and exclude others. Since Eq.~\ref{eq:abstract-commute} only applies in cases where the intervention $\*X_L = \*x_L$ is contained in $\cI_L$, any intervention that is not in $\cI_L$ or $\cI_H$ is deemed irrelevant in the context of exact transformations, and no restrictions are placed on their corresponding interventional distributions according to the definition. In an extreme case, if $\cI_L$ and $\cI_H$ only contained the empty intervention (i.e. $\cI_L = \cI_H = \{\emptyset\}$), then Eq.~\ref{eq:abstract-commute} would only require that $P(\tau(\*V_{L})) = P(\*V_H)$ and makes no statements about any interventional distributions from $\cL_2$. As opposed to requiring $\cI_L$ and $\cI_H$ to contain all interventions, this flexibility allows one to specify which interventions are well-defined, which is important since not every intervention may translate well across an abstraction. Consider the following example for concreteness.

\begin{example}
    \label{ex:allowed-interventions}
    Suppose a two-branch government is voting on a law, where $Y$ is a binary variable denoting whether the law is enacted, and $X_1$ and $X_2$ are the binary variables representing the votes of the two branches. In this case, $\*V_L = \{X_1, X_2, Y\}$. The law is only considered if both branches vote ``yes'', so instead of representing the two branches' votes separately, one could introduce a new variable
    \begin{equation}
    X_H \gets (X_1 = \text{``yes''}) \wedge (X_2 = \text{``yes''})
    \end{equation}
    as an abstraction of $X_1$ and $X_2$, with $\*V_H = \{X_H, Y\}$.
    
    In this case, the low level intervention $(X_1 \gets \text{``yes''}, X_2 \gets \text{``yes''})$ (simultaneously intervening on both $X_1$ and $X_2$), would map to a high level intervention $(X_H \gets 1)$. Eq.~\ref{eq:abstract-commute} then dictates that
    \begin{equation}
    P(\tau(\*V_{L[X_1 = \text{``yes''}, X_2 = \text{``yes''}]})) = P(\*V_{H[X_H = 1]}).
    \end{equation}
    However, an intervention like $(X_1  \gets \text{``yes''})$, which only intervenes on $X_1$, does not have a corresponding high level counterpart. The value of $X_H$ under this intervention would still depend on $X_2$. Hence, the intervention $(X_1 \gets \text{``yes''})$ should be excluded from $\cI_L$, which implies no restrictions on $P(\tau(\*V_{L[\*X_1 = \text{``yes''}]}))$.
\hfill $\blacksquare$
\end{example}

For any interventional set $\cI$, there exists a natural partial ordering $\leq$ such that $i \leq j$ for $i, j \in \cI$ if and only if the interventional values of $i$ are a subset of those in $j$ (e.g., $(A \gets a, B \gets b) \leq (A \gets a, B \gets b, C \gets c)$). Given the orderings $\leq_L$ and $\leq_H$ of $\cI_L$ and $\cI_H$, respectively, the order preserving property of $\omega$ is defined to mean that $i \leq_{L} j$ for $i, j \in \cI_L$ implies $\omega(i) \leq_{H} \omega(j)$. This property enforces a kind of regularity condition on $\omega$ ensuring that low-level interventions are still related when translated to the higher level. For example, an intervention of $(A \gets a, B \gets b, C \gets c)$ on the low level may map through $\omega$ to an intervention $(X \gets x, Z \gets z)$ on the high level. If we consider the same intervention, but with one more added value, such as $(A \gets a, B \gets b, C \gets c, D \gets d)$, we would expect that $\omega$ would map it to a similar intervention, possible with more values on the high level, such as $(X \gets x, Z \gets z, Y \gets y)$. The order preserving property of $\omega$ prevents it from mapping the intervention to one with fewer values such as $(X \gets x)$ or ones with different values like $(X \gets x, Z \gets z')$.

With all of these properties, exact $\tau$-transformations establish an important foundational property expected from all abstractions, namely, that the abstraction mapping $\tau$ commutes with applied interventions (Eq.~\ref{eq:abstract-commute}), as illustrated in Fig.~\ref{fig:abs-commutativity}.

\begin{example}
\label{ex:related}
Consider a low level SCM  $\cM_L = \langle \*U_L, \*V_L, \cF_L, P(\*U_L) \rangle$ that models an alarm system. Suppose $\*V_L = \{E, S, A\}$, all binary, where the alarm rings ($A = 1$) if either there is an earthquake ($E = 1$) or smoke from a fire ($S = 1$), with some possible noise. Formally, the causal mechanisms are described as follows:

\begin{align}
    \*U_L &= \{U_E, U_S, U_A\} \label{eq:ex-related-ML-U} \\
    \*V_L &= \{E, S, A\} \\
    \cF_L &= \begin{cases}
        E \gets f^{L}_E(u_E) &= u_E \\
        S \gets f^L_S(u_S) &= u_S \\
        A \gets f^L_A(e, s, u_A) &= (e \vee s) \oplus u_A
    \end{cases} \\
    P(\mathbf{U}_L) &: P(U_E \! = \! 1) \! = \! P(U_S \! = \! 1) \! = \! P(U_A \! = \! 1) \! = \! 0.5 \label{eq:ex-related-ML-PU}
\end{align}

Now suppose instead of considering both earthquake and fire individually, we would like to abstract both of these events into a less granular variable $D$, representing whether or not some disaster has occurred. That is, $D = E \vee S$. To be precise, this means the high level variables can be defined as $\*V_H = \{D, A_H\}$, and $\tau$ can be defined such that 
\begin{equation}
(D, A_H) \gets \tau(e, s, a) = (e \vee s, a).
\end{equation}
Now consider the following SCM $\cM_H^{(1)}$ defined over these variables:

%\begin{wrapfigure}{r}{0.4\textwidth}
\begin{figure}
%\vspace{-0.1in}
%\hspace{-0.05cm}
\centering
\includegraphics[width=0.5\linewidth]{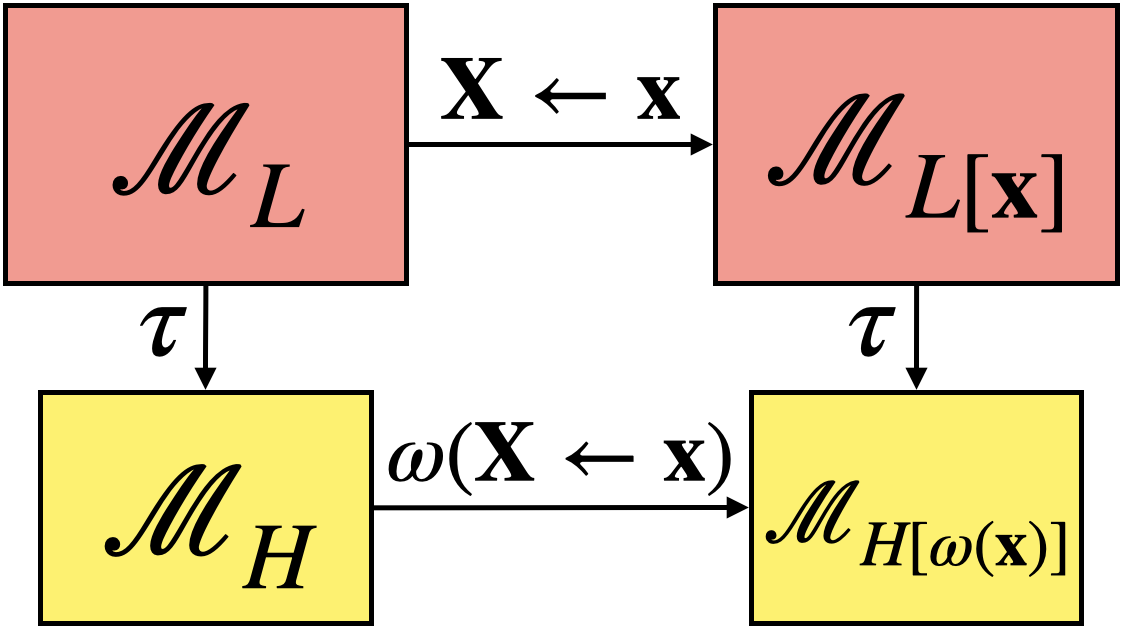}
\caption{Illustration of interventional commutativity. Applying low-level intervention $\*X \gets \*x$ followed by abstraction $\tau$ is equivalent to first applying $\tau$, followed by applying high-level intervention $\omega(\*X \gets \*x)$.} 
\label{fig:abs-commutativity}
\end{figure}
%\vspace{-0.4in}
%\end{wrapfigure}

\begin{align}
        \*U_1 &= \{U_D, U_A\} \label{eq:ex-related-MH1-U} \\
        \*V_H &= \{D, A_H\} \\
        \cF_1 &= \begin{cases}
            D \gets f^{1}_D(u_E) &= u_D \\
            A_H \gets f^{1}_{A_H}(d, u_A) &= d \oplus u_A
        \end{cases} \\
        P(\mathbf{U}_1) &: P(U_D = 1) = 0.75, P(U_A = 1) = 0.5 \label{eq:ex-related-MH1-PU}
\end{align}

Further suppose that the lists of allowed interventions are
\begin{equation}
\begin{split}
\cI_L =& \{\emptyset, (E \leftarrow 1), (S \leftarrow 1), (E\leftarrow e, S\leftarrow s), \\
& (A\leftarrow a), (E\leftarrow 1, A\leftarrow a), \\
& (S \leftarrow 1, A\leftarrow a), (E\leftarrow e, S\leftarrow s, A\leftarrow a)\},
\end{split}
\end{equation}
and
\begin{equation}
\cI_H^{(1)} = \{\emptyset, (D\leftarrow d), (A_H\leftarrow a), (D\leftarrow d, A_H\leftarrow a)\},
\end{equation}
for all settings of $e, s, a, d$. The partial ordering of these interventions are in the listed order.

One can verify that $(\cM_H^{(1)}, \cI_H^{(1)})$ is an exact $\tau$-transformation of $(\cM_L, \cI_L)$. For $\omega$, we can choose one that maps any combination of $E \gets e, S \gets s, A \gets a$ to $D \gets e \vee s, A_H \gets a$. Any case with $E \gets 1$ or $S \gets 1$ automatically maps to $D \gets 1$. We note that it is order preserving (e.g., $(E\gets 0, S\gets 0) \leq_L (E\gets 0, S\gets 0, A\gets 0)$ and $(D \gets 0) \leq_H (D \gets 0, A_H \gets 0)$), and one can verify that Eq.~\ref{eq:abstract-commute} holds. For example,
\begin{equation}
P(\tau(E = 1, S = 0)) = 0.75 = P(D = 1)
\end{equation}
and 
\begin{equation}
P(\tau(A_{E=1, S=0}) = 1) = 0.5 = P(A_{H[D=1]} = 1).
\end{equation}
One notable property that allows this to occur is that $E \gets 0$ and $S \gets 0$ are not valid interventions according to $\cI_L$, since the corresponding intervention mapped by $\omega$ is ambiguous as discussed earlier in Ex.~\ref{ex:allowed-interventions}.

Now consider an alternative model $\cM_H^{(2)}$, described as follows:

\begin{align}
        \*U_2 &= \{U_D, U_A\} \\
        \*V_H &= \{D, A_H\} \\
        \cF_2 &= \begin{cases}
            D \gets f^{2}_D(a, u_E) &= u_D \\
            A_H \gets f^{2}_{A_H}(u_A) &= u_A
        \end{cases} \\
        P(\mathbf{U}_2) &: P(U_D = 1) = 0.75, P(U_A = 1) = 0.5.
\end{align}

It turns out that, if $\cI_H^{(2)} = \cI_H^{(1)}$, then $(\cM_H^{(2)}, \cI_H^{(2)})$ is an exact $\tau$-transformation of $(\cM_L, \cI_L)$ for the same choice of $\tau$, even though in $\cM_H^{(2)}$, the causal relationship between $D$ and $A_H$ no longer exists. In fact, consider another case where $\cI_L 
= \cI_H^{(2)} = \{\emptyset\}$, that is, only the empty intervention is allowed on either level, and $\omega(\emptyset) = \emptyset$. Perhaps surprisingly, in this case, $(\cM_H^{(2)}, \cI_H^{(2)})$ is an exact $\tau$-transformation of $(\cM_L, \cI_L)$ for several other counterintuitive choices of $\tau$ as well. For instance, we can choose $(D, A_H) \gets \tau(e, s, a) = (\neg e \vee a, s)$, which does not even map $A$ to $A_H$. We can even choose one that is not consistent across variables, for example $\tau(e, s, a)$ can map $(0, 1, 0), (0, 0, 1), (1, 0, 1)$ to $(1, 0)$; $(0, 0, 0), (0, 1, 1), (1, 1, 0)$ to $(1, 1)$; $(1, 1, 1)$ to $(0, 0)$; and $(1, 0, 0)$ to $(0, 1)$. One can verify that, in both cases, this still results in an exact $\tau$-transformation. Furthermore, one can imagine changing the names of the variables to describe something arbitrarily different. In fact, it seems that $\cM_H^{(2)}$ and $\cM_L$ are completely unrelated.
\hfill $\blacksquare$
\end{example}

The heart of the issue raised in the previous example is that in cases where several choices of $\*v_L$ have the same probability, one can still obtain a valid abstraction without violating Eq.~\ref{eq:abstract-commute} by rearranging values mapped by $\tau$ that have the same probability, even if the resulting rearrangement has no causal interpretation. Further when the allowed interventions are sparse, Eq.~\ref{eq:abstract-commute} is required to hold on fewer distributions, resulting in a weaker connection between the high and low-level models. In these cases, the definition of exact $\tau$-transformations becomes weak and can often no longer be used to define abstractions in any intuitive sense.

Building on the work of \citet{rubenstein:etal17-causalsem}, \citet{beckers2019abstracting} introduced several refined definitions that resolved these issues, including the notion of $\tau$-abstractions. We rewrite the definitions as shown below.

\beckersnotation*

\beckerstauabs*

\beckersstrongtauabs*

We discuss each point of Def.~\ref{def:tau-abs} in detail:
\begin{enumerate}
    \item The surjectivity of $\tau$ does not add any mathematical benefits but is required as a property because $\*V_H$ is expected to be less ``complex'' than $\*V_L$ if $\cM_H$ is to be called an abstraction of $\cM_L$.

    \item The addition of $\tau_{\*U}$ is the major constraint added to $\tau$-abstractions when compared to exact transformations. By establishing a connection between the lower and higher level exogenous variables, $\tau_{\*U}$ ensures that the distributions of $\*V_H$ meaningfully correspond to those of $\*V_L$. This, in fact, fixes the issue described earlier in Ex.~\ref{ex:related}, since even if Eq.~\ref{eq:abstract-commute} holds, it is not necessarily the case that Eq.~\ref{eq:tau-u-compatibility} will hold unless the values of $\*U_L$ that are used to compute the l.h.s.~of Eq.~\ref{eq:abstract-commute} match those of $\tau_{\*U}(\*U_L)$ that are used to compute the r.h.s. This point is illustrated in Ex.~\ref{ex:tau-abs-fix} below.

    \item With $\cI_L$ and $\omega_{\tau}$ fixed, $\cI_H$ should be fixed to $\omega_{\tau}(\cI_H)$ to remain consistent.
\end{enumerate}

Def.~\ref{def:strong-tau-abs} further fixes the issue of $\cI_L$ containing too few interventions. When $\cI_H$ is maximal, then every submodel $\cM_{H[\*x_H]}$ has a corresponding submodel $\cM_{L[\*x_L]}$ such that $\omega_{\tau}(\*X_L \gets \*x_L) = \*X_H \gets \*x_H$, as described in Eq.~\ref{eq:tau-u-compatibility}.

\begin{example}[Example \ref{ex:related} continued]
\label{ex:tau-abs-fix}
Consider the high level models $\cM_H^{(1)}$ and $\cM_H^{(2)}$ from Ex.~\ref{ex:related}. Using the definition of $\tau(e, s, a) = (e \vee s, a)$, $\omega_{\tau}$ is forced to take the mapping as specified earlier where 
\begin{equation}
E \gets e, S \gets s, A \gets a
\end{equation}
maps to 
\begin{equation}
D \gets e \vee s, A_H \gets a.
\end{equation}
Note that in this case, $(\cM_H^{(1)}, \cI_H^{(1)})$ is still a $\tau$-abstraction of $(\cM_L, \cI_L)$. Specifically, choose $\tau_{\*U}$ such that
\begin{equation}
(U_D, U_A) \gets \tau_{\*U}(u_E, u_S, u_A) = (u_E \vee u_S, u_A).
\end{equation}
One can verify, for example, that
\begin{equation}
    \tau(\cM_{L[E=e, S=s]}(\*u_L)) = \cM_{H[D=e \vee s]}^{(1)}(\tau_{\*U}(\*u_L))
\end{equation}
for all values of $e$, $s$, and $\*u_L$, aligning with Eq.~\ref{eq:tau-u-compatibility}.

On the other hand, $(\cM_H^{(2)}, \cI_H^{(2)})$ is not a $\tau$-abstraction of $(\cM_L, \cI_L)$. For example, Eq.~\ref{eq:tau-u-compatibility} states that
\begin{equation}
    \label{eq:tau-incompatility-1}
    \tau(\cM_{L[E=0, S=0]}(\*u_L)) = \cM_{H[D=0]}^{(2)}(\tau_{\*U}(\*u_L))
\end{equation}
and
\begin{equation}
    \label{eq:tau-incompatility-2}
    \tau(\cM_{L[E=1, S=0]}(\*u_L)) = \cM_{H[D=1]}^{(2)}(\tau_{\*U}(\*u_L)).
\end{equation}
However, note that
\begin{equation}
    \tau(A_{E=0, S=0}(u_A)) \neq \tau(A_{E=1, S=0}(u_A)).
\end{equation}
For example, when $E = 0, S = 0, U_A = 0$, then $f_A^L$ will assign $A = 0$, but when $E = 1, S = 0, U_A = 0$, $f_A^L$ will assign $A = 1$.
On the other hand,
\begin{equation}
    A_{H[D=0]}^{(2)}(\tau_{\*U}(u_A)) = A_{H[D=1]}^{(2)}(\tau_{\*U}(u_A))
\end{equation}
for any choice of $\tau_{\*U}$ because $f_{A_H}^2$ does not take $D$ as an input. This contradicts the equalities enforced by Eqs.~\ref{eq:tau-incompatility-1} and \ref{eq:tau-incompatility-2}.

\hfill $\blacksquare$
\end{example}

The definitions introduced so far are effective at describing one SCM as an abstraction of another. For example, if two SCMs $\cM_H$ and $\cM_L$ are provided, as well as the function $\tau$, Def.~\ref{def:tau-abs} can be used to decide whether $\cM_H$ is indeed an abstraction of $\cM_L$. However, this may not be particularly useful in cases where the higher level model $\cM_H$ is not known in advance, and one would like to find or learn such an abstraction. \citet{beckers2019abstracting} makes an important step in the direction of applying such works by defining a more concrete class of abstractions that can be obtained by construction.

\beckersconstauabs*

In this definition, variables of $\*V_L$ are specifically partitioned into clusters $\*C_1, \dots, \*C_{n+1}$, and $\tau$ is defined such that each cluster $\*C_i$ maps to a high level variable $V_i \in \*V_H$. The definition of $\tau$ and corresponding high level space $\*V_H$ are concretely defined in this definition. We leverage a similar concept in this paper, allowing the higher level variables $\*V_H$ (and correspondingly, $\tau$), to be defined by construction based on predetermined clusters of lower level variables.

\subsection{Comparisons with Sec.~\ref{sec:abstract-ncm}}
The approach used in this work leverages similar ideas to constructive $\tau$-abstractions (Def.~\ref{def:cons-tau-abs}) for the purpose of obtaining the high level model $\cM_H$ constructively. Notably, the intervariable clusters in Def.~\ref{def:var-clusterings} partition the variable space in the same way, and the corresponding choice of $\tau$ from Def.~\ref{def:tau} is defined around these clusters, by utilizing a different subfunction $\tau_{\*C_i}$ for each intervariable cluster $\*C_i$, similar to Def.~\ref{def:cons-tau-abs}. For these reasons, any choice of $\tau$ that follows Def.~\ref{def:tau} is called a constructive abstraction function.

Still, the major difference is that Def.~\ref{def:cons-tau-abs} focuses on the relationship between the full models $\cM_L$ and $\cM_H$, while Def.~\ref{def:tau} only defines the mapping $\tau$ that connects the variable spaces $\*V_L$ and $\*V_H$. That is, Def.~\ref{def:tau} by itself makes no claims about how other aspects of $\cM_L$ (such as the functions $\cF_L$ or exogenous noise $P(\*U_L)$) relate to $\cM_H$, other than the variables $\*V_L$ and $\*V_H$. This is a new approach to abstraction work. Note that exact transformations (Def.~\ref{def:exact-transform}) and $\tau$-abstractions (Def.~\ref{def:tau-abs}) place no requirements on the definition of $\tau$, and Def.~\ref{def:cons-tau-abs} only requires that $\tau$ can be decomposed relative to a partition. Indeed, placing requirements on $\tau$ reduces its generality, but ensuring that $\tau$ follows the form illustrated in Def.~\ref{def:tau} has several advantages:

\begin{enumerate}
    \item \textbf{[Query-Specific Abstractions]} The primary purpose of Def.~\ref{def:tau} is to introduced a relaxed notion of abstractions that are defined on specific distributions of the PCH. As opposed to exact transformations and $\tau$-abstractions (including constructive ones), which focus on the entire SCMs $\cM_L$ and $\cM_H$, the concept of $Q$-$\tau$ consistency (Def.~\ref{def:q-tau-consistency}) allows one to define ``partial'' abstractions. For instance, a choice of $\cM_H$ can be considered an abstraction of $\cM_L$ for $Q_1$ but not $Q_2$ if $\cM_H$ is $Q_1$-$\tau$ consistent with $\cM_L$ but not $Q_2$-$\tau$ consistent. This subtlety is lost in $\tau$-abstractions for example, where any mismatch of Eq.~\ref{eq:tau-u-compatibility} disqualifies $\cM_H$ from being considered an abstraction of $\cM_L$. See Example \ref{ex:partial-abstraction} below for a more concrete explanation on this distinction. Indeed, when $\cM_H$ is $\cL_3$-$\tau$ consistent with $\cM_L$, that is, $\cM_H$ is $Q$-$\tau$ consistent with $\cM_L$ on every possible counterfactual query, then it turns out that $\cM_H$ behaves like a constructive-$\tau$ abstraction of $\cM_L$ (see Prop.~\ref{prop:abs-connect}).\footnote{Note that all $\tau$-abstractions are exact $\tau$-transformations, a result from \citet{beckers2019abstracting}. However, exact transformations are not necessarily $\cL_3$-$\tau$ consistent because Eq.~\ref{eq:abstract-commute} is focused on $\cL_2$ and is oblivious to the counterfactual level.}

    Defining abstractions on the level of individual queries enables a more practical approach to learning abstractions. The true model $\cM_L$ is rarely available in practice, and instead, one is often given data from $\cM_L$ from its induced distributions (e.g., the observational distribution $P(\*V_L)$). As discussed in Sec.~\ref{sec:learning-abs}, one would ideally be able to construct a high-level model $\cM_H$ that is $Q$-$\tau$ consistent with the available distributions. Such a model may not be $Q$-$\tau$ consistent with $\cM_L$ on choices of $Q$ that were not provided in the data, but through Alg.~\ref{alg:ncm-solve-absid}, one can determine precisely which choices of $Q$ do indeed match across abstractions. $\cM_H$ can be considered an abstraction of $\cM_L$ for those cases.

    \item \textbf{[Natural Interventional Mapping]} As opposed to exact transformations and $\tau$-abstractions, the concept of a function $\omega$ that maps between intervention spaces $\cI_L$ and $\cI_H$ is no longer required since the corresponding high-level intervention of a low-level one is straightforward. Specifically, the intervention $\*X_L \gets \*x_L$ maps to $\tau(\*X_L) \gets \tau(\*x_L)$ (see Lem.~\ref{lem:omega-tau-connection}). Additionally, the concept of ``allowed interventions'' is no longer needed. One can simply set $\cI_H$ as the set of all high level interventions $\cI_H^*$, and choose $\cI_L$ as the set of corresponding interventions that map to $\cI_H^*$ (interventions of unions of clusters, as shown in Lem.~\ref{lem:omega-surjectivity}). For these reasons, the presentation in Sec.~\ref{sec:abstract-ncm} does not include any references of $\cI_L$, $\cI_H$ or $\omega$, leaving the focus of the discussion on the abstraction of the variables $\*V_L$.

    \item \textbf{[True Constructiveness]} Despite the progress that constructive $\tau$-abstractions (Def.~\ref{def:cons-tau-abs}) make in the direction of constructively building $\cM_H$ from $\cM_L$, the definition alone does not accomplish this task. By defining $\tau$ as a mapping across clusters of both variables and values, one can leverage Alg.~\ref{alg:map-abstraction} from this paper to obtain the high level model $\cM_H$ when given the low level $\cM_L$. When $\cM_L$ is not provided, and data from $\cM_L$ is provided instead, one can use Alg.~\ref{alg:ncm-solve-absid} to obtain a model $\cM_H$ which is still an abstraction of $\cM_L$ on identifiable queries. This approach is implementable in practice leveraging neural optimization approaches, and the experiments provided in Sec.~\ref{sec:experiments} demonstrate their applicability.

    \item \textbf{[Intuitive Abstractions]} Abstractions that are constructed with a choice of $\tau$ that does not follow Def.~\ref{def:tau} can be quite esoteric (see Ex.~\ref{ex:non-constructive-tau} in App.~\ref{app:examples}). The concept of clustering is intuitive, and the relationship between $\*V_H$ and $\*V_L$ when $\tau$ is a constructive abstraction function is straightforward and interpretable. Intervariable clusters can be determined based on the needs of the task or constructed algorithmically via Alg.~\ref{alg:choose-intervariable-clusters} (in App.~\ref{app:cons-hierarchy}). These clusters also have a natural connection with cluster causal diagrams \citep{anand:etal23}, as illustrated in Sec.~\ref{sec:learning-abs}. Intravariable clusters are strongly tied to invariances in the data (two low-level values that are clustered together will map to the same high-level value), leading to a strong connection with representation learning (see Sec.~\ref{sec:applications} and App.~\ref{app:rep-learning}).
\end{enumerate}

Consider the following example comparing constructive $\tau$-abstractions to the concept of $Q$-$\tau$ consistency.

\begin{example}[Example \ref{ex:related} continued]
    \label{ex:partial-abstraction}
    Recall from Example \ref{ex:related} that the SCM $\cM_H^{(1)}$ as described by Eqs.~\ref{eq:ex-related-MH1-U} to \ref{eq:ex-related-MH1-PU} is a $\tau$-abstraction (and therefore also an exact $\tau$-transformation) of $\cM_L$ described by Eqs.~\ref{eq:ex-related-ML-U} to \ref{eq:ex-related-ML-PU}.

    It turns out that this choice of $\tau$, defined such that
    \begin{equation}
    (D, A_H) \gets \tau(e, s, a) = (e \vee s, a),
    \end{equation}
    is actually a constructive abstraction function with the intervariable clusters
    \begin{equation}
    \bbC = \{\*C_1 = \{E, S\}, \*C_2 = \{A\}\}
    \end{equation}
    and intravariable clusters
    \begin{equation}
    \bbD = \{\bbD_{\*C_1}, \bbD_{\*C_2}\},
    \end{equation}
    where
    \begin{equation}
    \bbD_{\*C_1} = 
    \begin{cases}
    d_0 &= \{(E = 0, S = 0)\}, \\
    d_1 &= \{(E = 0, S = 1), (E = 1, S = 0), \\
    & (E = 1, S = 1)\},
    \end{cases}
    \end{equation}and $\bbD_{\*C_2}$ retains the same values of $\cD_{A}$.
    
    It is then easy to verify that $\cM_H^{(1)}$ is also $\cL_3$-$\tau$ consistent with $\cM_L$. For example, $P^{\cM_L}(A_{E = 0, S = 0} = 1 \mid A = 1, E = 0, S = 1) = 0$, which is the counterfactual probability that the alarm would ring had neither earthquake nor fire occurred, given the reality that the alarm indeed rang when there was a fire but no earthquake. It is also true that the equivalent query mapped across $\tau$ (from Def.~\ref{def:q-tau-consistency}) is consistent, that is, $P^{\cM_H^{(1)}}(A_{H[D = 0]} = 1 \mid A = 1, D = 1) = 0$. This $\cL_3$-$\tau$ consistency is a consequence of Prop.~\ref{prop:abs-connect}.

    Now consider another high level model $\cM_H^{(3)}$ also defined over $\*V_H$ from the same $\tau$.
    \begin{align}
        \*U_3 &= \{U_D, U_{A0}, U_{A1}\} \\
        \*V_H &= \{D, A_H\} \\
        \cF_3 &= \begin{cases}
            f^{3}_D(u_E) &= u_D \\
            f^{3}_{A_H}(d, u_{A0}, u_{A1}) &=
            \begin{cases}
                u_{A0} & d = 0 \\
                u_{A1} & d = 1
            \end{cases}
        \end{cases} \\
        P(\mathbf{U}_3) &: P(U_D = 3) = 0.75, \\
        & P(U_{A0} = 1) = P(U_{A1} = 1) = 0.5
    \end{align}

    Note that $P^{\cM_H^{(3)}}(A_{H[D = 0]} = 1 \mid A = 1, D = 1) = P(U_{A0} = 1) = 0.5$, which is inconsistent with the result from $\cM_L$. Hence, $\cM_H^{(3)}$ is not $\cL_3$-$\tau$ consistent with $\cM_L$, nor is it a $\tau$-abstraction. Still, careful analysis of $\cM_H^{(3)}$ reveals that it is still $\cL_2$-$\tau$ consistent with $\cM_L$. For example, $P^{\cM_H^{(3)}}(A_{H[D = 1]} = 1) = 0.5 = P^{\cM_L}(A_{E = 1, S = 0} = 1)$. Therefore, $\cM_H^{(3)}$ may still be a valid abstraction of $\cM_L$ if used to infer layer 2 or interventional quantities. This highlights the limitations of a definition of abstractions that works on the level of the SCM, such as $\tau$-abstractions. If the user of the model is only interested in interventional quantities, it may be premature to discount $\cM_H^{(3)}$ as an invalid abstraction.
    
    \hfill $\blacksquare$
\end{example}

\subsection{Learning Abstraction Functions}
\label{app:learn-tau-given-mh}

The previous section discussed the benefits of defining the abstraction function $\tau$ to be a constructive abstraction function from Def.~\ref{def:tau}, with many of the reasons leading to the ability to learn the higher-level model $\cM_H$. Nonetheless, there are works that solve the inversion version of this problem, namely, where information about $\cM_H$ is given but the function $\tau$ is unknown.

For example, \citet{DBLP:conf/clear2/ZennaroDAWD23} solves the problem where, given both low and high-level models $\cM_L$ and $\cM_H$, as well as intervariable clusters $\bbC$, the goal is to learn subfunctions $\widehat{\tau}_{\*C_i}$ for each $\*C_i \in \bbC$ that satisfies desired properties of abstractions. Notably, the key properties required in this case are the commutativity of interventions, as illustrated in Fig.~\ref{fig:abs-commutativity}, as well as surjectivity of the functions. 
While a blind search over possible choices of $\tau$ is intractable, the paper cleverly models the commutativity property as a continuous error term that decreases as commutativity is closer to being achieved. 
They further model $\widehat{\tau}$ as a neural network and train it to minimize the error term using gradient descent, while regularizing the objective to enforce surjectivity. 
This approach allows for a more robust solution to the issue of intractability of specifying intravariable clusters, as discussed in Sec.~\ref{sec:applications}, albeit with the additional requirement that data is available from the high-level model $\cM_H$.

More recently, \citet{felekis:etal24} solved a similar problem with the requirement of having the full specification of $\cM_L$ and $\cM_H$ relaxed. The proposed approach leverages principles of optimal transport to acquire the abstraction map $\widehat{\tau}$ using only interventional data from the two models.
Since $\cM_L$ and $\cM_H$ are not provided, the optimization procedure instead enforces do-calculus constraints. In this setting, it is desirable to provide as much expressiveness as possible in the modeling of $\widehat{\tau}$ (e.g., with universal approximators such as neural networks), and the optimization may output a result for $\widehat{\tau}$ that is not necessarily a constructive abstraction function.

%\clearpage
\section{Experimental Details} \label{app:experiments}

This section provides details about our experimental setup and models. Our pipeline is primarily built with PyTorch \citep{paszke2017automatic}, and training is facilitated using PyTorch Lightning \citep{falcon2020framework}.

\subsection{Nutrition Experimental Setup}
\label{app:nutrition-setup}

The nutrition experiment in Sec.~\ref{sec:exp-nutrition} is a toy study of various individuals and their diets, performed over variables $\*V_L = \{R, D, C, F, P, B\}$, where $R$ is restaurant attended, $D$ is dish ordered at that restaurant, $C$ is carbohydrates of the dish in grams, $F$ is fat of the dish in grams, $P$ is protein of the dish in grams, and $B$ is the BMI of the individual. $R$ and $D$ have domains of size 32, indicating 32 different options of restaurant and dishes. $C$, $F$, $P$, and $B$ are real valued numbers. The data generating model $\cM_L = \langle \*U_L, \*V_L, \cF_L, P(\*U_L) \rangle$ is described below.
{
\allowdisplaybreaks
\begin{align*}
    \*U_L &= \{U_R, U_D, U_{RB}, U_{N1}, U_{N2}, U_{N3}, U_B\} \\
    \*V_L &= \{R, D, C, F, P, B\} \\
    \cF_L &= \{\\
        &f^L_R(u_R, u_{RB}) = (u_R + 16 \cdot u_{RB}) \% 32 \\
        &f^L_D(r, u_D) = (r + u_D) \% 32 \\
        &f^L_C(d, u_{N1}, u_{N2}, u_{N3}) = 216 \cdot u_{N1}[d\%3]\\
        &\cdot \left(0.25 \left(\left\lfloor \frac{f}{16} \right\rfloor \oplus u_{N3}\right) + 1\right) + 9 \cdot u_{N2} \\
        &f^L_P(d, u_{N1}, u_{N2}, u_{N3}) = 216 \cdot u_{N1}[(d + 1)\%3]\\
        &\cdot \left(0.25 \left(\left\lfloor \frac{f}{16} \right\rfloor \oplus u_{N3}\right) + 1\right) + 9 \cdot u_{N2} \\
        &f^L_F(d, u_{N1}, u_{N2}, u_{N3}) = 96 \cdot u_{N1}[(d + 2)\%3]\\
        &\cdot \left(0.25 \left(\left\lfloor \frac{f}{16} \right\rfloor \oplus u_{N3}\right) + 1\right) + 4 \cdot u_{N2} \\
        &f^L_B(c, f, p, u_B, u_{RB}) = \\
        &\left(\left(\frac{c}{9} + \frac{f}{4} + \frac{p}{9} + 3\cdot u_{RB}\right) - 30\right) \cdot (-1)^{u_B} + 25 \\
    P(\*U_L) &=
    \begin{cases}
        &P(U_R = u_R) \\
        &=
        \begin{cases}
            \frac{3}{64} & u_R \in \{0, 1, \dots, 15\} \\
            \frac{1}{64} & u_R \in \{16, 17, \dots, 31\} \\
            0 & \text{ otherwise}
        \end{cases} \\
        &P(U_D = u_D) \\
        &=
        \begin{cases}
            \frac{1}{7} & u_F \in \{-3, -2, -1, 0, 1, 2, 3\} \\
            0 & \text{ otherwise}
        \end{cases} \\
        &U_{RB} \sim \bern(0.25) \\
        &U_{N1} \sim \dirich(4, 1, 1) \\
        &U_{N2} \sim \unif(0, 1) \\
        &U_{N3} \sim \bern(0.1) \\
        &U_B \sim \bern(0.1)
    \end{cases}
\end{align*}
where $\%$ indicates the ``modulo'' operator, $\oplus$ is the binary XOR operator, and $u_{N1}[i]$ denotes the $i$th index of $u_{N1}$, which is a 3-dimensional variable. In the experiments, $R$ and $D$ are formatted as one-hot vectors.
}

For the abstraction of $\cM_L$, we choose intervariable clusters $\bbC = \{D_H = \{D\}, Z = \{C, F, P\}, B_H = \{B\}\}$, where $R$ is abstracted away, $D$ and $B$ are put into their own clusters, and $C$, $F$, and $P$ are clustered into a new variable $Z$, called ``calories''. Intravariable clusters $\bbD$ are chosen such that the values of each intervariable cluster are divided into two sets (i.e.~$D_H$, $Z$, $B_H$ are all binary variables). Specifically,
\begin{align*}
    \tau_{D_H}(d) &=
    \begin{cases}
        0 & d \in \{0, 1, \dots, 15\} \\
        1 & d \in \{16, 17, \dots, 31\}
    \end{cases} \\
    \tau_{Z}(c, f, p) &= \mathbf{1}\{4c + 9f + 4p \geq 1080\} \\
    \tau_{B_H}(b) &= \mathbf{1}\{b \geq 25\}
\end{align*}
For example,
\begin{align*}
    &\tau(R=7, D=24, C=80, F=70, P=40, B=32) \\
    &= (D_H = 1, Z = 1, B_H = 1).
\end{align*}
The high level variables $\*V_H = \tau(\*V_L)$ are defined to be $\{D_H, Z, B_H\}$.

The causal diagram $\cG$ over $\*V_L$ and the corresponding C-DAG $\cG_{\bbC}$ over $\*V_H$ are shown in Fig.~\ref{fig:cdag-examples}. The query of interest is $Q = P(B_{F = f} \geq 25)$, where $f$ is any arbitrary unhealthy food option $f \in \{16, 17, \dots, 31\}$. The query can be interpreted as the probability of someone being overweight if they are forced (intervened) to eat unhealthy food. The corresponding query on the higher level is $\tau(Q) = P(B_{H[D_H = 1]} = 1)$, computed from Eq.~\ref{eq:q-tau-consistency} in Def.~\ref{def:q-tau-consistency}. The task is to identify and estimate the query $Q$ given observational data $P(\*V_L)$ and causal diagram $\cG$. We test three approaches:
\begin{enumerate}
    \item The first approach is to directly identify and estimate $Q$ from $P^{\cM_L}(\*V_L)$ and $\cG$. The NeuralID algorithm \citep[Alg.~1]{xia:etal23} is used on these inputs. In this approach, the $\cG$-NCM $\widehat{M}$ is fitted over the variables $\*V_L$ and graph $\cG$ and is trained on the data from $P^{\cM_L}(\*V_L)$. It is optimized for the identification/estimation tasks, and $Q$ is directly queried from $\widehat{M}$.

    \item The second approach is to identical to the first, except values of $\*V_L$ (specifically $C$, $F$, $P$, and $B$) are all normalized between -1 and 1. The query is reformulated to $Q = P(B_{D = d} \geq 0)$, and the NCM $\widehat{M}$ is modified to use these normalized values\footnote{In fact, this second approach can also be considered abstracting the space of $\*V_L$. Specifically, the intervariable and intravariable clusters are all singleton clusters, but the values of $\*V_H$ are renamed such that optimization is easier (similar to the ideas discussed in Sec.~\ref{sec:applications}).}.

    \item The third approach is the abstraction approach, which instead identifies and estimates $\tau(Q)$ from $\tau(P^{\cM_L}(\*V_L))$ and $\cG_{\bbC}$, running Alg.~\ref{alg:ncm-solve-absid} on these inputs. That is, the $\cG_{\bbC}$-NCM $\hM$ is fitted over variables $\*V_H$ and graph $\cG_{\bbC}$ and is trained on the data of $P^{\cM_L}(\*V_L)$ passed through $\tau$. Note that this approach already requires fewer assumptions by using the C-DAG $\cG_{\bbC}$ instead of the full causal diagram $\cG$. The end result should be theoretically equivalent to identifying and estimating $Q$ due to Corol.~\ref{corol:abs-id-ncm}.
\end{enumerate}

The experimental results are shown in Fig.~\ref{fig:exp1-results}, where the first approach is shown in red, the second in yellow, and the third in blue.

\subsection{Nutrition Models and Hyperparameters}
\label{app:nutrition-hyperparams}

All three models used in the nutrition experiment are GAN-NCMs from \citet{xia:etal23}, which leverage generative adversarial networks (GAN) \citep{NIPS2014_5ca3e9b1}. The GAN-NCMs for the first two approaches without abstractions are SCMs $\widehat{M} = \langle \widehat{\*U}, \*V_L, \widehat{\cF}, P(\widehat{\*U})\rangle$ fitted to the graph $\cG$ (Fig.~\ref{fig:cdag-examples} (left)). Each function $\hat{f}_V \in \widehat{\cF}$ is a feedforward neural network with 3 hidden layers of width 32, with layer normalization applied \citep{ba2016layer}. Each exogenous variable $\widehat{U} \in \widehat{\*U}$ is a 2-dimensional vector, with each dimension sampled independently from a uniform distribution between -1 and 1. The discriminator is a feedforward neural network with 3 hidden layers of width 64. The GAN-NCM for the third approach with abstractions has the exact same parameter settings but is modeled over $\*V_H$ instead. Consequently, the NCM for the third approach has fewer parameters since it only requires three functions for $\*V_H$, compared to the six functions for $\*V_L$.

The GAN-NCMs are trained in the style of Wasserstein GANs \citep{arjovsky2017wasserstein}, where the objective is to minimize the Earth-Mover distance via the Kantorovich-Rubenstein duality \citep{villani2009optimal}:
\begin{equation}
    \label{eq:kanto-ruben-duality}
    \min_G \max_{D \in \cD_D} \bbE_{\*x \sim P_r}[D(\*x)] - \bbE_{\tilde{\*x} \sim P_g}[D(\tilde{\*x})],
\end{equation}
where $G$ is a generating model (e.g.~the NCM $\hM$), $D$ is a discriminatory model, also called a critic (not to be confused with the variable $D \in \*V_L$), $\cD_D$ is the set of 1-Lipshitz functions, $P_r$ is a real distribution (i.e.~from $\cM_L$), and $P_g$ is the distribution induced by $G$.

For identification experiments, models were trained for 1000 epochs on datasets with $n=10^4$ samples. 10 trials were performed with each approach, with 4 reruns for each trial for hypothesis testing purposes. In a single run, two parameterizations of the NCM are initialized with one aiming to minimize the query and one aiming to maximize it. In each iteration, a batch of real data is provided, and a batch of fake data is generated by the NCM. Given these two batches, the discriminator is trained to minimize the loss following Eq.~\ref{eq:kanto-ruben-duality}:
\begin{equation*}
    L_D = \bbE_{\tilde{\*x} \sim P_g}[D(\tilde{\*x})] - \bbE_{\*x \sim P_r}[D(\*x)],
\end{equation*}
where $P_g$ and $P_r$ refer to the fake and real datasets respectively. In words, the loss is computed by taking the expected score of the critic on fake samples subtracted by the expected score of the critic on real samples, indicating better performance if the critic gives higher scores to real samples. After each training iteration, the gradients of the discriminator $D$ are clamped between $[-0.01, 0.01]$ to enforce the Lipschitz constraint.

Following an iteration of the discriminator, another batch of fake data is sampled from the generator (NCM $\hM$), and the weights of the generator are updated with the loss
\begin{equation}
    \label{eq:gen-loss}
    L_G = -\bbE_{\tilde{\*x} \sim P_g}[D(\tilde{\*x})] + \lambda L_Q.
\end{equation}
The first term is the expected score of the critic on the fake samples, which should be maximized by the generator to create as convincing samples as possible. The second term is a query loss, intended to push the model to simultaneously maximize or minimize the query. In practice, this is done by calculating the distance between the intended value of the query and query samples from the generator using some distance function. For example, for the GAN-NCM in the first non-abstraction approach, $L_Q$ is defined as:
\begin{equation*}
    L_Q(\hat{\*b}) = \pm \left(\frac{1}{|\hat{\*b}|}\sum_{\hat{b} \in \hat{\*b}} \hat{b} - 25\right),
\end{equation*}
where $\hat{\*b}$ is a batch of samples from $P^{\hM}(B_{D = d})$ computed from \citet[Alg.~2]{xia:etal23}. If this quantity, which is simply a mean over the batch samples, is maximized (resp.~minimized), then that would also maximize (resp.~minimize) the query $Q = P^{\hM}(B_{D = d} \geq 25)$. For the second approach with normalized data, the 25 is not subtracted as it is already centered around 0. For the third approach working in the abstracted space, the log loss is calculated instead, since values of $B_H$ are binary. $\lambda$ is a hyperparameter to indicate the strength of the query loss term; in our experiments it was set to $10^{-4}$ and decreased logarithmically to $10^{-8}$ by the end of training.

For the visualization of the results in Fig.~\ref{fig:exp1-id-gaps}, the query is estimated from both the model which optimized to maximize it (denote as $Q_{\max}$) and the model which optimized to minimize it ($Q_{\min}$). Since the query is identifiable (see Sec.~\ref{app:id-proofs}), we expect $Q_{\max} - Q_{\min} = 0$ under perfect optimization. However, as optimization is not perfect, a hypothesis testing procedure must be used to check if $Q_{\max} - Q_{\min} < \varepsilon$ for some threshold $\varepsilon$. As suggested by \citet{xia:etal23}, we rerun each trial 4 times and take the upper 95$\%$ confidence bound of the mean of $Q_{\max} - Q_{\min}$ from the 4 reruns. Then the means of these upper confidence bounds across 10 trials are plotted in Fig.~\ref{fig:exp1-id-gaps} with $95\%$ confidence intervals.

For estimation experiments, models were trained for 200 epochs on datasets. 10 trials were performed for each approach and each setting of sample size varying logarithmically from $n=10^3$ to $10^5$ samples. The training procedure is performed identically to the identification experiments, except only one parameterization is trained, and the query loss in Eq.~\ref{eq:gen-loss} is not added. After training, queries are estimated from each of the models using \citet[Eq.~4]{xia:etal23} with $10^5$ Monte Carlo samples, and they are compared with the ground truth value calculated from $\cM_L$ as described in App.~\ref{app:nutrition-setup}. The mean absolute error (MAE) is computed between the two values and plotted in Fig.~\ref{fig:exp1-est-mae} with $95\%$ confidence intervals across the 10 trials for each of the sample size settings.

All NCMs are trained with a learning rate of $10^{-4}$, and discriminators are trained with a learning rate of $2 \times 10^{-4}$. Models are optimized with the RMSProp optimizer \citep{hinton_srivastava_swersky}, recommended by the WGAN paper. Estimation experiments are performed with a batch size of 128, and identification experiments are performed with a batch size of 1000 (larger size for more representative sample is important in identification). All feedforward networks are initialized with Glorot initialization \citep{pmlr-v9-glorot10a}. Hyperparameter tuning was done by hand with the suggestions from referenced sources. Similar hyperparameters did not have noticeable effects on performance, so more rigorous hyperparameter tuning was not conducted.

\subsection{Colored MNIST Experimental Setup}
\label{app:mnist-setup}

The Colored MNIST experiment in Sec.~\ref{sec:exp-mnist} is performed on a modified version of the MNIST dataset of handwritten digits \citep{deng2012mnist}. The setting is modeled over variables $\*V_L$ consisting of a digit label $D$, a color label $C$, and the pixels consisting of a $3 \times 32 \times 32$ MNIST image with color channels. Both $D$ and $C$ take integer values from 0 to 9, formatted as one-hot vectors in the data. However, the mechanisms for which the image is generated is unknown, since we do not know all of the details of how humans handwrite digits. Instead, we directly work on the high level space of variables $\*V_H = \{D, C, I\}$, obtained by clustering all of the pixels into one variable, called image $I$. Samples from the observational distribution $P(\*V_L)$ are generated using the following approach:
\begin{enumerate}
    \item A sample is drawn from exogenous variable $U_{CD}$, which takes values from 0-9, indicating what the intended digit is.
    \item With probability $0.85$, set $C \gets U_{CD}$. Otherwise, choose uniformly at random from the 10 values. Similarly, but independently from $C$, set $D \gets U_{CD}$ with probability $0.85$, otherwise choose uniformly at random.
    \item Given $C$ and $D$, sample an image from the MNIST dataset with label $D$, then color the digit with the color corresponding to $C$ on the gradient in Fig.~\ref{fig:exp2-legend}.
\end{enumerate}
The causal diagram $\cG$ over $\*V_L$ is unknown because it is unclear how individual pixels are related. However, the C-DAG $\cG_{\bbC}$ over $\*V_H$ is shown in Fig.~\ref{fig:exp2-graph}, which is compatible with the data generating process mentioned above. Specifically, image $I$ is caused by color $C$ and digit $D$, which are highly correlated through unobserved confounding.

The task is to train a model $\hM_H$ over variables $\*V_H$ constrained by the graph $\cG_{\bbC}$ such that $\hM_H$ induces the distribution $\tau(P^{\cM_L}(\*V_L))$ (i.e.~it is perfectly trained to match the observational data sampled from $P^{\cM_L}(\*V_L)$), and then use it to produce realistic digit samples from three different causal queries:
\begin{enumerate}
    \item $P(I \mid D = 0)$: the distribution of images conditional on digit $=0$. In the dataset, the digit 0 is highly correlated with the color red, so samples from this distribution should be images of handwritten 0s, most of which are red.

    \item $P(I_{D = 0})$: the distribution of images when intervened on digit $=0$. When an intervention is performed, spurious correlations are ignored. The color is sampled like normal, but then the digit is forced to become 0 regardless of the color. Hence, samples from this distribution should be images of handwritten 0s but with colors evenly distributed.

    \item $P(I_{D = 0} \mid D = 5)$: the counterfactual distribution of images of what they would have been had digit been forced to be $0$ given that the digit was originally 5. When conditioning on $D=5$, the samples are filtered such that only ones with $D=5$ remain, but then these samples are intervened and forced to take the digit 0 instead. Consequently, samples from this distribution should be images of handwritten 0s that retain the color of the 5s, which are typically cyan.
\end{enumerate}

Three different approaches are compared:
\begin{enumerate}
    \item The first is a basic conditional GAN that learns the correlation between digit $D$ and image $I$. The conditional GAN ignores the information in $\cG_{\bbC}$ and therefore is na\"ive to the causal invariances in the data.

    \item The second is a GAN-NCM \citep{xia:etal23} that is constrained by $\cG_{\bbC}$ and is directly fitted on the data $\tau(P^{\cM_L}(\*V_L))$. In this case $\tau$ simply clusters the pixels together into the image $I$, but the mapping between domains of $\*V_L$ and $\*V_H$ is the identity mapping. In other words, the intravariable clusters $\bbD$ can be thought of as the singleton partition of all domains, and as a result the space of $\*V_H$ is identical to $\*V_L$.

    \item The third is a GAN version of the RNCM from Sec.~\ref{sec:applications}, called GAN-RNCM. The GAN-RNCM is also constrained by $\cG_{\bbC}$ but learns its own abstraction function $\widehat{\tau}$. Specifically, in a typical instantiation of the $\cG_{\bbC}$-RNCM in this case, we would have $\widehat{\tau} = (\widehat{\tau}_{C}, \widehat{\tau}_D, \widehat{\tau}_I)$, where each subfunction learns a mapping to a representation space (akin to learning intravariable clusters). For this experiment, we only parameterized $\widehat{\tau}_{I}$, since $C$ and $D$ are low-dimensional and are already easy to learn. $\widehat{\tau}_I$ is trained to map to a space that preserves bijectivity (as demanded by Prop.~\ref{prop:full-intra-cluster}) as well as maximizing information retained about $C$ and $D$. See the next subsection for specific details.
\end{enumerate}

The results are illustrated in Fig.~\ref{sec:exp-mnist}. The GAN-RNCM clearly outperforms the other two approaches, and we even observed a shorter runtime. Although the GAN-NCM is, in theory, supposed to be able to capture the intended distributions, we believe its failure is a result of the difficulty of simultaneously optimizing two different tasks: (1) image generation is already a challenging task with a lot of attention in the deep learning community, and (2) learning a joint distribution with causal constraints is also challenging. The GAN-RNCM breaks the problem into two parts. The representation learning of $\widehat{\tau}$ solves the problem of dealing with high dimensional images, reducing the space to a much simpler space in which the distribution with causal constraints can be learned more easily.

\subsection{Colored MNIST Models and Hyperparameters}

The GAN-NCM used in the Colored MNIST experiment is different from the ones used in the Nutrition experiment since it is fitted on a different graph, specifically $\cG_{\bbC}$ from Fig.~\ref{fig:exp2-graph}. Since the function for the image, $\hat{f}_I$, is responsible for generating an entire image, we leverage the technology of convolutional neural networks to produce higher quality results. Specifically, we use the state-of-the-art research on conditional image generation implemented by \citet{brock2018large}, called BigGAN. $\hat{f}_I$ is designed by first mapping the inputs, color $C$ and digit $D$, through a feedforward neural network to an internal representation, which is then piped into the 32$\times$32 image-size architecture with 64 feature maps implemented by the BigGAN authors. The functions $\hat{f}_C$ and $\hat{f}_D$ are simply feedforward neural networks. In this model, all feedforward nets have 3 hidden layers, with widths that depend on the size of the inputs and outputs using the formula $2 \times i \times o$, where $i$ is the total dimensionality of all endogenous and exogenous inputs, and $o$ is the output dimensionality (number of channels for images). Each exogenous variable $\widehat{U} \in \widehat{\*U}$ is a $\delta$-dimensional vector, where $\delta$ is the sum of the dimensions of all variables in the confounded clique represented by $\widehat{U}$, and each dimension is sampled independently from a uniform distribution between -1 and 1. The discriminator first pipes image inputs through a deconvolutional component like implemented in BigGAN, before combining the internal representation with other variables to pipe through a feedforward neural network with 3 hidden layers of width 128. Layer normalization is applied between layers of feedforward nets, and batch normalization \citep{pmlr-v37-ioffe15} is applied between convolutional layers.

The conditional GAN approach is implemented similarly, but without $\hat{f}_C$. Training is done identically to the Wasserstein GAN approach in the Nutrition experiment (described in App.~\ref{app:nutrition-hyperparams}), but without the query loss in Eq.~\ref{eq:gen-loss}, as identification is not performed.

The GAN-RNCM $\langle \widehat{\tau}, \widehat{M} \rangle$ is trained in a two part procedure, first training $\widehat{\tau}$ and then training $\widehat{M}$ on the space defined by $\widehat{\tau}$. Only the abstraction function for the image, $\widehat{\tau}_{I}$, is trained for this experiment. It is modeled in two parts: (1) a convolutional neural network with three convolutional layers (with 64, 128, and 256 feature maps respectively) mapping the image to a 128-dimensional vector, and (2) a feedforward neural network with 3 hidden layers of width 128 mapping the convolutional output to a 64-dimensional representation space.

$\widehat{\tau}_I$ is trained for 500 epochs. In each epoch, a batch of the colored MNIST digits is sampled and passed through $\widehat{\tau}$ to obtain a representation. Then, $\widehat{\tau}$ is trained with the loss
\begin{align}
    L_{\tau_G}(\hat{i}_H, i_L, c_L, d_L) &= \norm{\tau^{-1}(\hat{i}_H; \theta_{\tau^{-1}}) - i_L}^2 \label{eq:recon-loss} \\
    &+ \lambda_g L_g(g(\hat{i}_H; \theta_g), c_L, d_L), \label{eq:class-loss}
\end{align}
where $\hat{i}_H$ is the 64-dimensional representation output from $\widehat{\tau}_H(i_L)$; and $i_L$, $c_L$, $d_L$ are the original data points\footnote{Although $C$ and $D$ are not part of the cluster with $I$, they can still be used in the training process for $\widehat{\tau}_I$ as long as they are not used as inputs to $\widehat{\tau}_I$.}. The first term (Eq.~\ref{eq:recon-loss}) is a reconstruction loss that ensures that Prop.~\ref{prop:full-intra-cluster} holds. $\tau^{-1}$ is another neural network (parameterized by $\theta_{\tau^{-1}}$) in the style of BigGAN that upscales the representation $\hat{i}_H$ back to an image of size $3 \times 32 \times 32$. The term is simply the MSE of the reconstruction with the original image, ensuring that both the encoder $\tau$ and decoder $\tau^{-1}$ are trained to be able to reproduce the input. Later, when sampling images of $I$, $\tau^{-1}$ is also used to reconstruct image samples. The second term (Eq.~\ref{eq:class-loss}) is a classification loss added to improve the learned representation to differentiate between different values of $c$ and $d$. $g$ is feedforward neural network parameterized by $\theta_g$ with 3 hidden layers of width 128, which outputs a prediction for $c$ and $d$ given the representation $\hat{i}_H$. Any classification loss can be used for $L_g$, and we choose binary cross-entropy loss since $C$ and $D$ are one-hot vectors. $\lambda_g$ is a regularization term which takes a value of 0.1 in our experiments.

After training $\widehat{\tau}_I$, the NCM $\hM$ of the RNCM is trained similarly to the other approaches, but it is instead trained on top of $\widehat{\tau}(\*V_L)$ instead of $\*V_L$. That is, instead of outputting an image from $\hat{f}_I$, it outputs a 64-dimensional real vector, representing $\widehat{\tau}_{I}(I)$. Hence, in the RNCM, $\hat{f}_I$ is replaced with a feedforward neural network with 3 hidden layers of width $2 \times i \times o$, as with the other functions.

For the training of $\widehat{\tau}_I$, as well as the training of all three generative models, $10^5$ samples are provided in the dataset $P(\*V_L)$. All models are trained with a learning rate of $10^{-4}$, and discriminators are trained with a learning rate of $2 \times 10^{-4}$. GAN models are optimized with the RMSProp optimizer, and the training procedure for $\widehat{\tau}_I$ is optimized with Adam \citep{KingmaB14}. All training is performed with a batch size of 128. All feedforward networks are initialized with Glorot initialization. Hyperparameter tuning was done by hand with the suggestions from referenced sources. Similar hyperparameters did not have noticeable effects on performance, so more rigorous hyperparameter tuning was not conducted.

Samples from the three competing approaches, as well as the original data generating ground truth, are collected from the three queries discussed in Sec.~\ref{app:mnist-setup} and compared in Fig.~\ref{fig:mnist-exp-results}. Queries are sampled via \citet[Alg.~1]{xia:etal23}. 

\subsection{Proofs of Identifiability}
\label{app:id-proofs}

In this section, we show proofs that the queries in the experiments are identifiable, leveraging do-calculus (denote R1, R2, and R3 as the three rules) and counterfactual axioms \citep{pearl:2k}.

\begin{proposition}
    $P(B_{D=d} = b)$ is identifiable from $P(\*V_L)$ and $\cG$ from Fig.~\ref{fig:cdag-examples} (left).
    \hfill $\blacksquare$

    \begin{proof}
        \begin{align*}
            & P(B_{D = d} = b) \\
            &= \sum_r P(b_d \mid r_d)P(r_d) \\
            &= \sum_r P(b \mid r, d)P(r_d) & \text{ R2} \\
            &= \sum_r P(b \mid r, d)P(r) & \text{ R3},
        \end{align*}
        and the final result can be computed from observational $P(\*V_H)$ as there are no more interventional terms.
    \end{proof}
\end{proposition}

\begin{proposition}
    $P(B_{H[D_H=d]})$ is identifiable from $P(\*V_H)$ and $\cG_{\bbC}$ from Fig.~\ref{fig:cdag-examples} (right).
    \hfill $\blacksquare$

    \begin{proof}
        \begin{align*}
            & P(B_{H[D_H = d]} = b) \\
            &= \sum_{z} P(z_d)P(b_d \mid z_d) \\
            &= \sum_{z} P(z \mid d)P(b_d \mid z_d) & \text{ R2}\\
            &= \sum_{z} P(z \mid d)P(b_{dz}) & \text{ R2}\\
            &= \sum_{z} P(z \mid d)P(b_{z}) & \text{ R3} \\
            &= \sum_{z} P(z \mid d)\sum_{d'} P(b_{z} \mid d'_{z})P(d'_{z}) \\
            &= \sum_{z} P(z \mid d)\sum_{d'} P(b \mid d', z)P(d'_{z}) & \text{ R2} \\
            &= \sum_{z} P(z \mid d)\sum_{d'} P(b \mid d', z)P(d') & \text{ R3,}
        \end{align*}
        and the final result can be computed from observational $P(\*V_H)$ as there are no more interventional terms.
    \end{proof}
\end{proposition}

\begin{proposition}
    $P(I = i \mid D = d)$ is identifiable from $P(\*V_H)$ and $\cG_{\bbC}$ from Fig.~\ref{fig:exp2-graph}.
    \hfill $\blacksquare$

    \begin{proof}
        This result is trivial, as $P(I = i \mid D = d)$ is an observational quantity and can therefore be computed as
        \begin{align*}
            &P(I = i \mid D = d) \\
            &= \frac{P(i, d)}{P(d)} = \frac{\sum_{c'} P(i, c', d)}{\sum_{i', c'} P(i', c', d)} = \frac{\sum_{c} P(\*v_H)}{\sum_{i, c} P(\*v_H)}.
        \end{align*}
    \end{proof}
\end{proposition}

\begin{proposition}
    $P(I_{D=d} = i)$ is identifiable from $P(\*V_H)$ and $\cG_{\bbC}$ from Fig.~\ref{fig:exp2-graph}.
    \hfill $\blacksquare$

    \begin{proof}
        \begin{align*}
            & P(I_{D = d} = i) \\
            &= \sum_c P(i_d \mid c_d)P(c_d) \\
            &= \sum_c P(i \mid c, d)P(c_d) & \text{ R2} \\
            &= \sum_c P(i \mid c, d)P(c) & \text{ R3},
        \end{align*}
        and the final result can be computed from observational $P(\*V_H)$ as there are no more interventional terms.
    \end{proof}
\end{proposition}

\begin{proposition}
    $P(I_{D=d} \mid D=d')$ is identifiable from $P(\*V_H)$ and $\cG_{\bbC}$ from Fig.~\ref{fig:exp2-graph}.
    \hfill $\blacksquare$

    \begin{proof}
        \begin{align*}
            & P(I_{D = d} = i \mid D=d') \\
            &= \frac{P(i_d, d')}{P(d')} \\
            &= \frac{\sum_c P(i_d, d', c)}{P(d')} \\
            &= \frac{\sum_c P(i_{dc}, d', c)}{P(d')} & \text{ C1} \\
            &= \frac{\sum_c P(i_{dc})P(d', c)}{P(d')} & \text{ C2} \\
            &= \frac{\sum_c P(i \mid d, c)P(d', c)}{P(d')} & \text{ R2} \\
            &= \sum_c P(i \mid d, c)P(c \mid d')
        \end{align*}
        where ``C1'' refers to the counterfactual axiom of composition, and ``C2'' refers to the C-factor decomposition of counterfactual variables \citep{correa:etal21}. The final result can be computed from observational $P(\*V_H)$ as there are no more interventional terms.
    \end{proof}
\end{proposition}

\subsection{Hardware}

All models were trained on NVIDIA Tesla V100 GPUs provided by Amazon Web Services, totalling approximately 2000 GPU hours for the final results.

%\clearpage
\section{Further Discussion} \label{app:discussion}

\subsection{The Constitution Hierarchy and Learning Intervariable Clusters}
\label{app:cons-hierarchy}

This section provides a detailed discussion on intervariable clusters $\bbC$, on the relationship of clusters at different levels of granularity, and best practices on how to choose the right level of granularity when several options are available.

On the intervariable level, one could cluster several lower level variables together and call the cluster a variable itself. The low-level variables do not \emph{cause} the high level variable, but rather they \emph{constitute} it (as discussed in \citet{10.5555/3020847.3020867}, and similar to discussions in probabilistics relational models on part-whole relationships \citep{morton:etal87}). In other words, they are two interpretations of the same content. The difference is important: the relationship is bidirectional, and one cannot directly intervene on one without simultaneously intervening on the other.

When comparing variables at different levels of granularity, a hierarchy arises, which we call the \emph{constitutional hierarchy}. That is, any phenomenon can be viewed across another dimension that determines the level of granularity of interpreting the variables. This is illustrated in Fig.~\ref{fig:cons-hierarchy}. Variables can be organized by what constitutes what, and, in a proper abstraction, causal properties should be preserved across different levels of granularity. Generally speaking, most studies focus on one specific level of abstraction, so the task of choosing the most appropriate level of granularity can be important. This is precisely the problem of choosing intervariable clusters, as the coarseness of the clusters induces a natural interpretation of the variables.

If the data scientist finds themselves in a situation where they have to construct the causal diagram themselves, or if a provided causal diagram is not at the right level of abstraction, then refining the abstraction level becomes a nontrivial task. The decision of whether to cluster variables together depend on various factors, which we list and elaborate in the sequel:

\begin{enumerate}[label=\textbf{C\arabic*.}]
    \item \emph{Non-causal} relationships should not be visible.
    \item The resulting clustering should be admissible.
    \item The queries of interest should be answerable.
    \item The queries of interest should be identifiable.
    \item The result should be as coarse as possible.
\end{enumerate}

Each of these conditions can be formalized to enable a systematic discussion on how to choose an appropriate intervariable clustering.

\smallskip

\noindent \textbf{Condition C1.} When considering models at extremely low levels of abstraction, there may be too much detail to properly label every relationship as a causal one.\footnote{We do not provide an exact definition of a causal relationship, as this is a deeply philosophical topic that is out of the scope of this work. Still, we acknowledge that there may be cases where two low-level variables are be related in a way that is not well-defined with respect to interventions in an SCM. We use the term ``non-causal'' as a bucket term for all such cases and make no assumptions about the natures of these relationships.} For example, like in the top level of Fig.~\ref{fig:cons-hierarchy}, at the atom or molecule level, there are interactions between particles that are studied in the physical sciences such as bonds. That is, two particles may have linked behavior, but it is not accurate to call such a relationship causal. As another example, suppose we consider images at the pixel level. In many tasks, one may be interested in the local dependences between pixels, and it may therefore be reasonable to model pixels using an undirected model (e.g., like a Markov random field). Once again, the relationship between the pixels may not necessarily be considered causal.

\begin{figure}
    \centering
    \includegraphics[width=\linewidth]{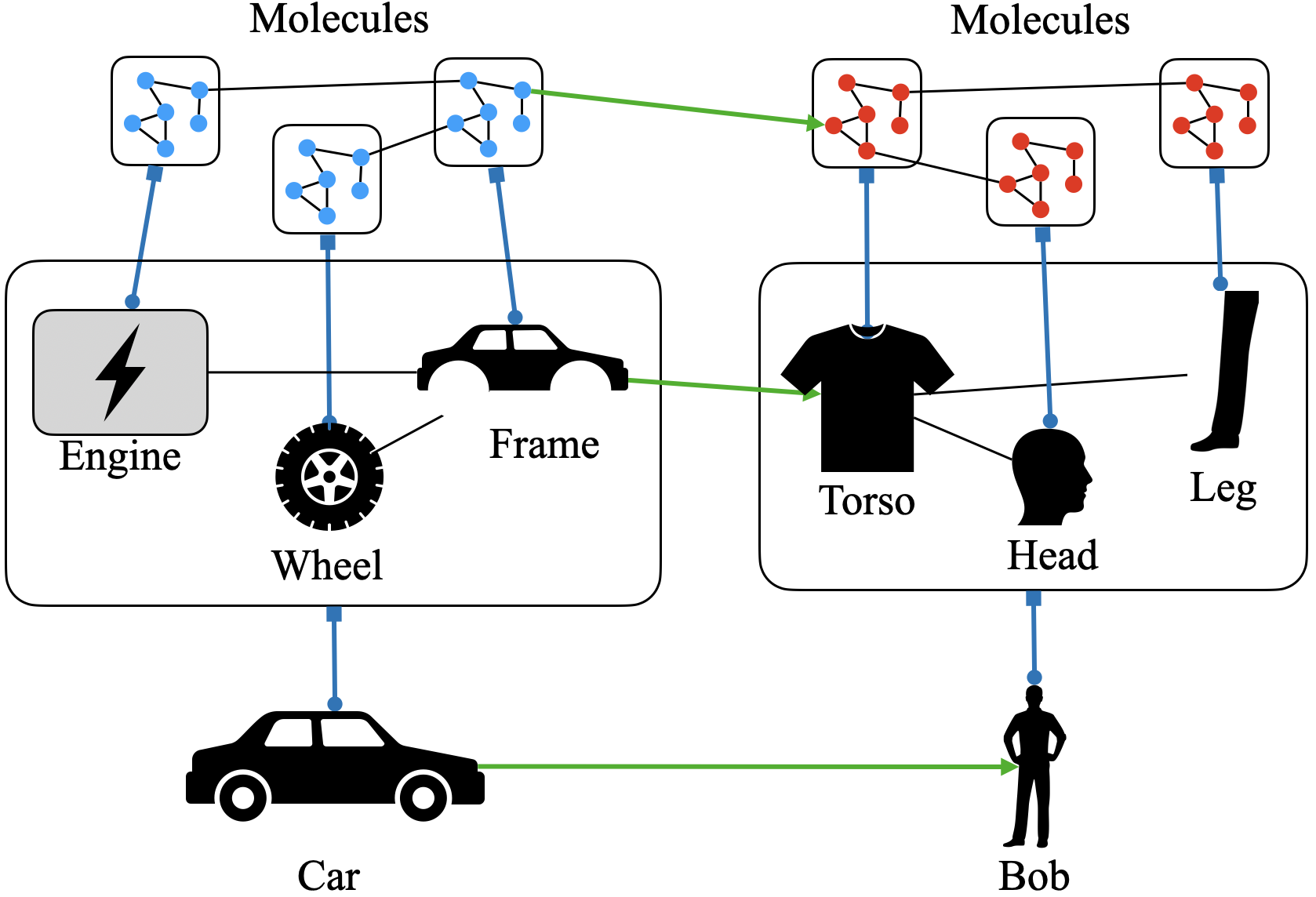}
    \caption{A visual example of the constitution hierarchy. Green edges indicate causal relations, blue edges indicate constitutional relations, and black edges are undefined, low-level relationships. The same objects are shown at three levels of granularity: molecules (top), parts (middle), and complete objects (bottom). A causal relationship might indicate that a car is the cause of Bob's injury, but it would be inaccurate to say that the molecules of a car ``cause'' the car. A particular study should focus on the variables from one specific level of granularity, and choosing the most appropriate one is the goal of the problem of choosing intervariable clusters.}
    \label{fig:cons-hierarchy}
\end{figure}

We note importantly that ``non-causal'' in this context does not refer to spurious effects, nor does it mean anti-causal in the sense that the direction of causality is reversed, as these are still well-defined from a causal perspective. For example, in a scenario modeled by Fig.~\ref{fig:spurious-causal}, the relationship between $X$ and $Y$ is not considered a non-causal relationship. Despite $Y$ not causing $X$, their relationship is still considered causal due to the cause from $X$ from $Y$. Further, although there is spurious correlation through the unobserved confounding, this is still considered a causal type of relationship since there is some causal effect from the unobserved confounder to $X$ and $Y$. ``Non-causal'' relationships refer to relationships between variables that cannot be defined in a causal manner. When considering non-causal relationships, the SCM framework is no longer compatible since it is not defined how to model such relationships. Nonetheless, the framework presented in this paper can still be used as long as C1 is satisfied.

The term ``non-causal'' is very broad and could encompass many different types of relationships. Studying specific types of non-causal relationships is out of the scope of this work, so we will treat these relationships in the same way. For every subset of variables (say $\*Z \subseteq \*V_L$) related through some non-causal relationship, we will assume that there is some function $f_{\*Z}$ that takes as input variables from $\*U_L$ and $\*V_L$ and maps it jointly to the space of $\*Z$. We note that this is a general encoding of such relationships that makes minimal assumptions. For example, perhaps the relationship between three variables, $X$, $Z$, and $Y$ can be described using a Markov random field like shown in Fig.~\ref{fig:noncausal-g}, where $X \indep Y \mid Z$. Our assumption states that the behavior of these three variables can be abstracted into one function $f_{\{X, Z, Y\}}$ which outputs values of $X$, $Z$, and $Y$, therefore losing the independence information. (It is possible, however, that these could be implicitly encoded through the intravariable clusters). These non-causal relationships can be expressed graphically.

\begin{definition}[Non-causal Graph]
    \label{def:non-causal-graph}
    Let $\mathbf{V}_L$ be a set of variables and $\overline{\cG} = \langle \mathbf{V}_L, \mathbf{E} \rangle$ be an undirected graph with nodes representing $\*V_L$. Then, an edge $(V_1, V_2)$ is in $\mathbf{E}$ if and only if there is some non-causal relationship between $V_1$ and $V_2$.
    \hfill $\blacksquare$
\end{definition}

We will assume that, in settings which intervariable clusters must be learned, we are given knowledge of the existence of non-causal relationships through a non-causal graph of the lowest level variables $\*V_L$. This information could be acquired simply as an assumption of the user, as they are typically only present at the lowest levels of abstraction (e.g.\ the relationship between pixels in an image or between atoms of an object). Under this assumption, distributions over $\*V_L$ can be treated as if variables in the same connected component in $\overline{\cG}$ share a common cause, and they can be factorized as such.

\smallskip

\noindent \textbf{Condition C2.}
By definition of admissibility, two variables should not be in the same intervariable cluster if doing so forms a cycle in the order of the functions or graph. This condition can be verified given causal diagram $\cG$ of the low level variables $\*V_L$.

\smallskip

\noindent \textbf{Condition C3.}
C3 and C4 depend on the user's needs based on the queries of interest, $\mathbb{Q}$, in downstream tasks. The ideal level of abstraction can be determined by the groupings of variables in these queries. C3 essentially enforces that the level of abstraction should be kept low enough to retain the ability to differentiate the nuances between individual variables of interest in the queries. For example, if one would like to study the causal effect of a drug $X$ on recovery rate $Y$, but $Y$ was clustered with another variable $Z$ representing blood pressure, then it no longer becomes possible to answer queries specifically about $Y$ without including $Z$. This is formally defined below.

\begin{definition}[Cluster Answerability]
    A counterfactual query $Q = P^{\mathcal{M}_L}(\*Y_{1[\*x_1]}, \*Y_{2[\*x_2]}, \dots)$ is answerable from intervariable clusters $\mathbb{C}$ if and only if for all $\*Y_i$ (and $\mathbf{X}_i$), there exists $\bbC_i \subseteq \bbC$ such that $\*Y_i = \bigcup_{\*C \in \bbC_i} \*C$.
    \hfill $\blacksquare$
\end{definition}

In other words, all queries should be written in terms of unions of clusters.

\smallskip

\noindent \textbf{Condition C4.}
Even if answerability is not violated, clustering variables together or projecting them out will always result in a loss of information. Such a loss may affect the identifiability of the queries. Ideally, one would not want to drop information that would change the status of an already identifiable query, and this is enforced by C4.

\smallskip

\noindent \textbf{Condition C5.}
Finally, C5 is usually desirable following the idea of Occam's razor, i.e., all else being equal, simplicity should be preferred. We define coarseness as follows.

\begin{definition}[Coarseness]
    Let $\mathbb{C}_1$ and $\mathbb{C}_2$ be two intervariable clusterings of $\*V_L$. $\mathbb{C}_1$ is said to be coarser than $\mathbb{C}_2$ (equivalently, $\mathbb{C}_2$ is finer than $\mathbb{C}_1$) if and only if for every $\*C_2 \in \bbC_2$, either $\*C_2 \cap \bigcup_{\*C_1 \in \bbC_1} \*C_1 = \emptyset$, or there exists $\*C_1 \in \bbC_1$ such that $\*C_2 \subseteq \*C_1$.
\end{definition}

With this definition, the goal is to find a maximally coarse clustering $\bbC$ that satisfies conditions 1-4. That is, there should not exist a coarser clustering $\bbC'$ that also satisfies the conditions, although the coarsest clustering is not necessarily unique.

It turns out that the set of all possible maximally course clusterings can be ``bounded'' in a sense using C1, C2, and C3. The ``minimally'' coarse clustering can be described with the following lemma.
\begin{lemma}
    \label{lem:min-clusters}
    Let $\overline{\cG}$ be a non-causal graph over $\*V_L$. Let $\bbC_{\min}$ be the intervariable clustering of $\*V_L$ composed of the connected components of $\overline{\cG}$. Then, any clustering $\bbC'$ violates condition C1 if and only if it is not equal to or coarser than $\bbC_{\min}$.
    \hfill $\blacksquare$
    \begin{proof}
        We first note that $\bbC_{\min}$ does not violate condition C1 because by construction, if there exists a noncausal edge between $V_1$ and $V_2$ in $\overline{\cG}$, then they must be in the same cluster. Any coarser clustering has the same property.
        
        If $\bbC'$ is not equal to or coarser than $\bbC_{\min}$, that means that there exists $\*C \in \bbC_{\min}$ such that $\*C \not \subseteq \*C'$ for all $\*C' \in \bbC'$. This implies that there exists $V_1, V_2 \in \*C$ such that belong to different clusters $\*C_1'$ and $\*C_2'$ in $\bbC'$. However, since $V_1$ and $V_2$ were in the same cluster in $\bbC_{\min}$, there must exist some non-causal path from $V_1$ to $V_2$ in $\overline{\cG}$. This means that there is a noncausal connection between clusters $\*C_1'$ and $\*C_2'$, so $\bbC'$ violates condition C1.
    \end{proof}
\end{lemma}
Intuitively, the minimally coarse clustering must at least cover the connected components of $\overline{\cG}$ to abstract away the non-causal relations and satisfy condition C1. Additionally, admissibility (C2) adds another constraint.
\begin{lemma}
    \label{lem:admissible-clusters}
    Let $\bbC$ be a set of intervariable clusters that are not admissible w.r.t.~$\cG$. Let $\bbC^*$ be the set of clusters that is coarser than $\bbC$ such that $\*C_1$ and $\*C_2$ in $\bbC$ are merged in $\bbC^*$ if and only if $\*C_1$ and $\*C_2$ are in a cycle in $\cG_{\bbC}$. Then, any set of clusters that are coarser than $\bbC$ and are admissible w.r.t.~$\cG$ are equal to or coarser than $\bbC^*$
    \hfill $\blacksquare$
    \begin{proof}
        First, we note that $\bbC^*$ is admissible w.r.t.~$\cG$ by construction, since all cycles have been merged. If $\bbC'$ is a clustering that is coarser than $\bbC$ but not $\bbC^*$, then there must exist $\*C^* \in \bbC^*$ such that $\*C^* \subsetneq \*C'$ for all $\*C' \in \bbC'$. Then, there must exist $V_1, V_2 \in \*C^*$ such that $V_1$ and $V_2$ are in different clusters in $\*C_1, \*C_2 \in \bbC'$. However, this implies that there is still a cycle between $\*C_1$ and $\*C_2$ in $\cG_{\bbC'}$, breaking admissibility.
    \end{proof}
\end{lemma}

\begin{figure}
    \centering
    \begin{tikzpicture}[xscale=1, yscale=1.0]
        \node[draw, circle] (X) at (-1, 0) {$X$};
        \node[draw, circle] (Y) at (1, 0) {$Y$};
        
        \path [-{Latex[length=2mm]}] (X) edge (Y);
        \path [{Latex[length=2mm]}-{Latex[length=2mm]}, dashed, bend left] (X) edge (Y);
    \end{tikzpicture}
    \caption{Example of spurious relationships.}
    \label{fig:spurious-causal}
\end{figure}
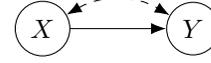

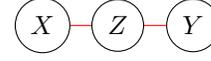
\begin{figure}
    \centering
    \begin{tikzpicture}[xscale=1, yscale=1.0]
        \node[draw, circle] (X) at (-1, 0) {$X$};
        \node[draw, circle] (Z) at (0, 0) {$Z$};
        \node[draw, circle] (Y) at (1, 0) {$Y$};
        
        \path [red, -] (X) edge (Z);
        \path [red, -] (Z) edge (Y);
    \end{tikzpicture}
    \caption{Example of non-causal graph.}
    \label{fig:noncausal-g}
\end{figure}

    \begin{figure*}
        \centering
        \begin{subfigure}{.3\linewidth}
            \centering
            \begin{tikzpicture}[xscale=1, yscale=1]
                \node[draw, circle] (A) at (-1, 2) {$A$};
                \node[draw, circle] (T) at (0.5, 2) {$T$};
                \node[draw, circle] (G) at (-2, 1) {$G$};
                \node[draw, circle] (H1) at (-0.5, 0.5) {$H_1$};
                \node[draw, circle] (H2) at (0.5, 0.5) {$H_2$};
                \node[draw, circle] (O1) at (1.5, 1) {$O_1$};
                \node[draw, circle] (O2) at (2.5, 1) {$O_2$};
                \node[draw, circle] (O3) at (1.5, 0) {$O_3$};
                \node[draw, circle] (O4) at (2.5, 0) {$O_4$};
                \node[draw, circle] (P1) at (-1, -1) {$P_1$};
                \node[draw, circle] (P2) at (0, -1) {$P_2$};
                \node[draw, circle] (P3) at (-1, -2) {$P_3$};
                \node[draw, circle] (P4) at (0, -2) {$P_4$};
    
                \node[draw,densely dotted,color=blue,rounded corners=2mm, fit=(O1) (O2) (O3) (O4), line width=1pt, inner sep=1.2mm] {};
                \node[draw,densely dotted,color=blue,rounded corners=2mm, fit=(P1) (P2) (P3) (P4), line width=1pt, inner sep=1.2mm] {};
                
                \path [-{Latex[length=2mm]}] (A) edge (T);
                \path [-{Latex[length=2mm]}] (G) edge (T);
                \path [-{Latex[length=2mm]}] (T) edge (H1);
                \path [-{Latex[length=2mm]}] (T) edge (H2);
                \path [-{Latex[length=2mm]}] (T) edge (O1);
                \path [-{Latex[length=2mm]}] (G) edge (P1);
                \path [-{Latex[length=2mm]}] (A) edge (P1);
                \path [-{Latex[length=2mm]}] (H1) edge (P1);
                \path [-{Latex[length=2mm]}] (H2) edge (P1);
                \path [-{Latex[length=2mm]}] (H2) edge (P2);
                \path [-{Latex[length=2mm]}] (O3) edge (P4);
                \path [-{Latex[length=2mm]}] (O4) edge (P4);
                \path [{Latex[length=2mm]}-{Latex[length=2mm]}, dashed, bend left] (G) edge (A);
                \path [red, -] (O1) edge (O2);
                \path [red, -] (O1) edge (O3);
                \path [red, -] (O1) edge (O4);
                \path [red, -] (O2) edge (O3);
                \path [red, -] (O2) edge (O4);
                \path [red, -] (O3) edge (O4);
                \path [red, -] (P1) edge (P2);
                \path [red, -] (P1) edge (P3);
                \path [red, -] (P2) edge (P4);
                \path [red, -] (P3) edge (P4);
            \end{tikzpicture}
            \subcaption{The full set of variables without clustering. Not a valid choice because condition C1 is violated (red noncausal relationships between $O$ and $P$ variables.}
            \label{fig:ex-learn-c-g1}
        \end{subfigure}
        \hfill
        \begin{subfigure}{.25\linewidth}
            \centering
            \begin{tikzpicture}[xscale=1, yscale=1]
                \node[draw, circle] (A) at (-1, 2) {$A$};
                \node[draw, circle] (T) at (0.5, 2) {$T$};
                \node[draw, circle] (G) at (-2, 1) {$G$};
                \node[draw, circle] (H1) at (-0.5, 0.5) {$H_1$};
                \node[draw, circle] (H2) at (0.5, 0.5) {$H_2$};
                \node[draw, circle] (H3) at (1.5, 0.5) {$H_3$};
                \node[draw, circle] (I) at (-1, -1) {$I$};
    
                \node[draw,densely dotted,color=blue,rounded corners=2mm, fit=(H1) (H2) (H3), line width=1pt, inner sep=1.2mm] {};
                
                \path [-{Latex[length=2mm]}] (A) edge (T);
                \path [-{Latex[length=2mm]}] (G) edge (T);
                \path [-{Latex[length=2mm]}] (T) edge (H1);
                \path [-{Latex[length=2mm]}] (T) edge (H2);
                \path [-{Latex[length=2mm]}] (T) edge (H3);
                \path [-{Latex[length=2mm]}] (A) edge (I);
                \path [-{Latex[length=2mm]}] (G) edge (I);
                \path [-{Latex[length=2mm]}] (H1) edge (I);
                \path [-{Latex[length=2mm]}] (H2) edge (I);
                \path [-{Latex[length=2mm]}] (H3) edge (I);
                \path [{Latex[length=2mm]}-{Latex[length=2mm]}, dashed, bend left] (G) edge (A);
            \end{tikzpicture}
            \subcaption{Graph with $O_1, O_2, O_3, O_4$ clustered into $H_3$ and $P_1, P_2, P_3, P_4$ clustered into $I$. This cluster is the $\bbC_{\min}$ obtained from Lem.~\ref{lem:min-clusters} and is a valid clustering that satisfies C1-4.}
            \label{fig:ex-learn-c-g2}
        \end{subfigure}
        \hfill
        \begin{subfigure}{.2\linewidth}
            \centering
            \begin{tikzpicture}[xscale=1, yscale=1]
                \node[draw, circle] (A) at (0, 1) {$A$};
                \node[draw, circle] (G) at (-1, 0) {$G$};
                \node[draw, circle] (M) at (1, 0) {$M$};
                \node[draw, circle] (I) at (0, -1.5) {$I$};
    
                \node[draw,densely dotted,color=blue,rounded corners=2mm, fit=(A) (G) (M), line width=1pt, inner sep=1.2mm] {};
                
                \path [-{Latex[length=2mm]}] (A) edge (I);
                \path [-{Latex[length=2mm]}] (G) edge (I);
                \path [-{Latex[length=2mm]}] (A) edge (M);
                \path [-{Latex[length=2mm]}] (G) edge (M);
                \path [-{Latex[length=2mm]}] (M) edge (I);
                \path [{Latex[length=2mm]}-{Latex[length=2mm]}, dashed, bend left] (G) edge (A);
            \end{tikzpicture}
            \subcaption{Graph with $H_1, H_2, H_3$ clustered into $M$. This cluster is also valid and is coarser than the graph to the left.}
            \label{fig:ex-learn-c-g3}
        \end{subfigure}
        \hfill
        \begin{subfigure}{.2\linewidth}
            \centering
            \begin{tikzpicture}[xscale=1, yscale=1]
                \node[draw, circle] (X) at (0, 1) {$X$};
                \node[draw, circle] (I) at (0, -1) {$I$};
                
                \path [-{Latex[length=2mm]}] (X) edge (I);
            \end{tikzpicture}
            \subcaption{Graph with $A, G, M$ clustered into $X$. This cluster is no longer valid because C3 is violated. The queries are no longer answerable.}
            \label{fig:ex-learn-c-g4}
        \end{subfigure}
        \caption{Graphs for Example \ref{ex:learn-intervariable-clusters}. Blue outline illustrates changes in clustering between each graph.}
        \label{fig:ex-learn-c-g}
    \end{figure*}
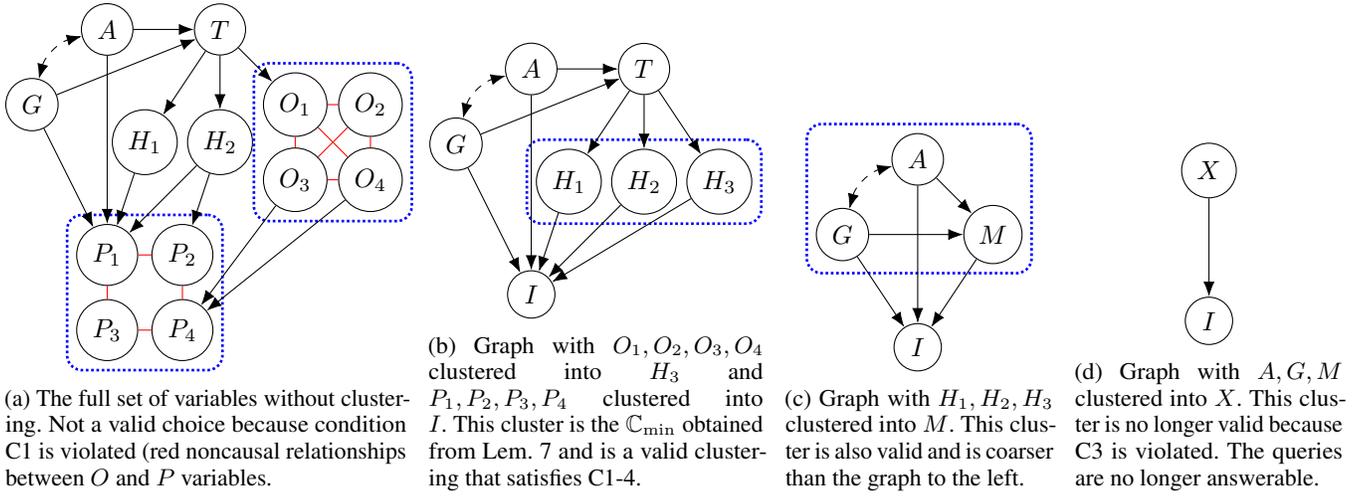

On the other hand, the maximally coarse clustering is related to the concept of maximally answerable clusters, defined below.

\begin{definition}[Maximally Answerable Clusters]
    Let $\bbQ$ be a set of counterfactual queries over variables $\*V_L$, and let $\bbV_{\bbQ}$ be the set of all subsets of $\*V_L$ that are used in a term in $\bbQ$ (i.e.\ for every counterfactual term $\*Y_{\*x}$ in any of the queries, both $\*Y$ and $\*X$ should be included in $\bbV_{\bbQ}$). Denote $\bbV^*_{\bbQ}$ as the extension of $\bbV_{\bbQ}$ closed under intersection, that is, $\bbV_{\bbQ} \subseteq \bbV^*_{\bbQ}$, and $V_1, V_2 \in \bbV^*_{\bbQ} \Rightarrow V_1 \cap V_2 \in \bbV^*_{\bbQ}$. Then, an intervariable clustering $\bbC_{\max}$ is called the maximally answerable cluster w.r.t.\ $\bbQ$ if and only if $\bbC_{\max}$ is the set of atoms of $\bbV^*_{\bbQ}$, or in other words, $\*V_i \in \bbC_{\max}$ if and only if $\*V_i \in \bbV^*_{\bbQ}$ and there exists no $\*V_j \in \bbV^*_{\bbQ}$ such that $\*V_j \subsetneq \*V_i$.
    \hfill $\blacksquare$
\end{definition}

The concept of maximally answerable clusters can be best shown through an example.
\begin{example}
    Suppose
    \begin{equation*}
        \*V_L = \{A, B, C, D, E, F, G, H, I, X, Z\}
    \end{equation*}
    and the queries of interest $\bbQ = \{Q_1, Q_2, Q_3\}$ are as follows:
    \begin{align*}
        Q_1 &= P(\{a, b, c, d, e\}_{x}) \\
        Q_2 &= P(\{c, d, e, f, g, h\}_{x, z}) \\
        Q_3 &= P(\{e, f, h, i, z\}_{x}).
    \end{align*}
    We first dissect the terms in the queries, including the subscripts, to obtain $\bbV_{\bbQ}$. Doing so yields
    \begin{align*}
        \bbV_{\bbQ} &= \{\{A, B, C, D, E\}, \{C, D, E, F, G, H\}, \\
        & \{E, F, H, I, Z\}, \{X\}, \{X, Z\}\}.
    \end{align*}
    These sets are shown at the top of Fig.~\ref{fig:ex-max-answer-cluster}. Note that many of the sets overlap. The set of nonoverlapping subregions becomes the maximally answerable clusters, as shown at the bottom of Fig.~\ref{fig:ex-max-answer-cluster}. In this case, this means
    \begin{align*}
        \bbC_{\max} &= \{\{A, B\}, \{C, D\}, \{E\}, \{F, H\},\\
        &\{G\}, \{I\}, \{X\}, \{Z\}\}.
    \end{align*}
    Note that no set in $\bbC_{\max}$ overlaps with each other, but all sets are subsets of some set in $\bbV_{\bbQ}$.
    \hfill $\blacksquare$
\end{example}

It turns out that the set of maximally answerable clusters is also the maximally coarse clustering that does not violate condition C3, as shown below.

\begin{lemma}
    \label{lem:max-clusters}
    Let $\bbC_{\max}$ be the intervariable clustering of $\*V_L$ that is the maximally answerable cluster with respect to a set of counterfactual queries $\bbQ$. Then, any clustering $\bbC'$ violates condition C3 if and only if it is not equal to or finer than $\bbC_{\max}$.
    \hfill $\blacksquare$
    \begin{proof}
        We first note that $\bbC_{\max}$ does not violate condition C3, since if there existed a query with a term set $\*Z$ such that $\*Z$ is not a union of clusters of $\bbC_{\max}$, then either there exists $Z \in \*Z$ such that $Z$ is not in any set of $\bbC_{\max}$, or $Z \in \*C$ for some $\*C \in \bbC_{\max}$ such that $\*C \subsetneq \*Z$. However, the first case is impossible since if $Z \in \*Z$, then $Z$ must be added to some set in $\bbV_{\bbQ}$. The second case is not possible either because then $\*Z \cap \*C \neq \emptyset$, which means that the clusters of $\bbC_{\max}$ could be divided further, contradicting its definition. If $\bbC_{\max}$ does not violate condition C3, then no finer clustering can violate it either because a union of clusters in $\bbC_{\max}$ can always translate to a union of clusters in a finer clustering.

        If $\bbC'$ is not equal to or finer than $\bbC_{\max}$, then there exists at least one cluster $\*C' \in \bbC'$ such that $\*C' \subsetneq \*C$ for all $\*C \in \bbC_{\max}$. This means that there exist $V_1, V_2 \in \*C'$ such that $V_1$ and $V_2$ are not in the same cluster in $\bbC'$. This implies that there exists at least one query such that one of its terms contains one of $V_1$ or $V_2$ but not the other, in which case, such a term cannot be described as a union of clusters from $\bbC'$, violating condition C3.
    \end{proof}
\end{lemma}

\begin{figure}
    \centering
    \includegraphics[width=0.42\linewidth]{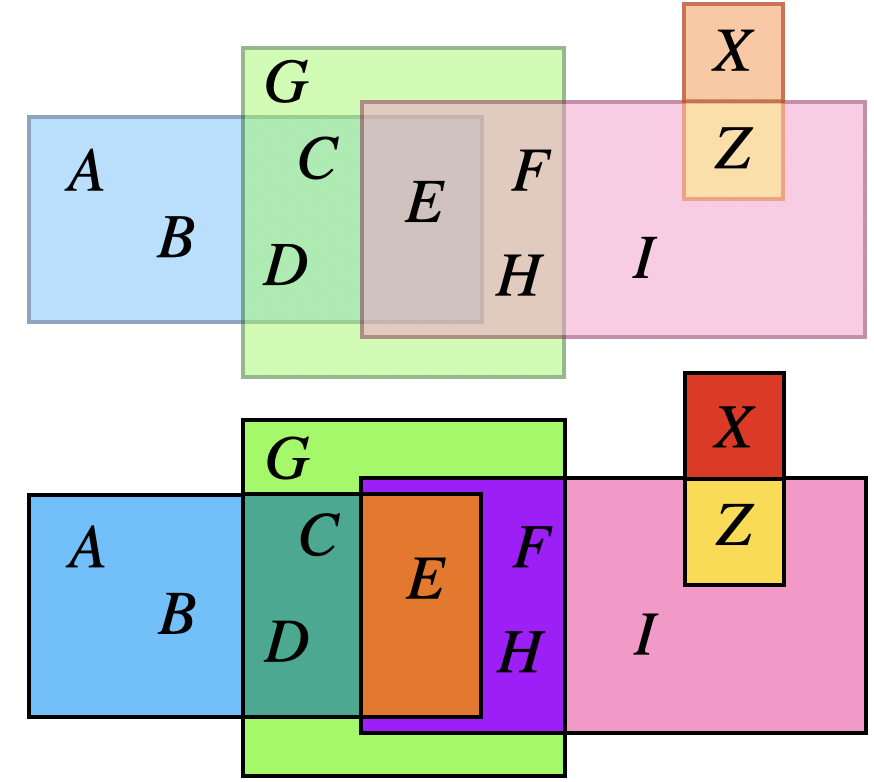}
    \caption{\textbf{Top:} A venn diagram of all of the variables and what set of $\bbV_{\bbQ}$ to which they belong. \textbf{Bottom:} The collection of nonoverlapping intersections of all sets that form the maximum answerable clusters.}
    \label{fig:ex-max-answer-cluster}
\end{figure}

Intuitively, we cannot simply cluster all variables together because we may want to answer queries that require separate consideration of different sets of variables. At best, we can only cluster together terms that always appear together in all queries. Additionally, we note that finding the maximally answerable clusters can be done in polynomial time. One simple approach is to add one set of $\bbV_{\bbQ}$ at a time, repeatedly computing the intersection with the existing clusters. At any point, there can be at most $n = |\*V_L|$ clusters.

With these points in mind, we provide Alg.~\ref{alg:choose-intervariable-clusters} for learning intervariable clusters. We note that Alg.~\ref{alg:choose-intervariable-clusters} returns a maximally course clustering satisfying all conditions.

\begin{proposition}
    Alg.~\ref{alg:choose-intervariable-clusters} returns a maximally course clustering that satisfies C1-4, or returns \texttt{FAIL} if one does not exist.
    \hfill $\blacksquare$
    \begin{proof}
        We first note that if \texttt{FAIL} is returned, then there does not exist a clustering that satisfies all four conditions. Notably, if $\bbC_{\max}$ is not coarser than $\bbC_{\min}$ at the beginning, then \texttt{FAIL} is returned, as a consequence of Lemmas \ref{lem:min-clusters}, \ref{lem:admissible-clusters}, and \ref{lem:max-clusters}. \texttt{FAIL} is also returned if $\bbC_{\min}$ does not satisfy identifiability of all queries, since further clustering will only remove edges and nodes and cannot make a non-ID query turn ID.

        If $\bbC$ is returned at the end, then $\bbC$ must satisfy all four conditions. It satisfies conditions C1 and C3 because it will always be coarser than $\bbC_{\min}$ but finer than $\bbC_{\max}$. It satisfies conditions C2 and C4 because $\bbC$ is never updated unless the ``valid'' subroutine returns \texttt{True}, which implies that C2 and C4 hold.

        Finally, $\bbC$ must be maximal. If a clustering violates C4, then there is no coarser clustering that does not violate C4 since clustering only further reduces information. If a clustering does not violate C2, and a coarser clustering exists that also does not violate C2, then there must exist a way to either merge pairs of clusters or remove variables (namely in reverse topological order) such that each intermediate clustering also does not violate C2. Then, the loop ensures that if a variable could be removed, or two clusters could be merged without violating the conditions, then it will be done.
    \end{proof}
\end{proposition}

Finally, we note that Alg.~\ref{alg:choose-intervariable-clusters} runs in polynomial time in terms of $n = |\*V_L|$ and $|\bbQ|$, as long as checking identification takes polynomial time. Finding $\bbC_{\min}$ from Lem.~\ref{lem:min-clusters} takes polynomial time because finding connected components in an undirected graph takes polynomial time. Checking admissibility simply requires checking cycles, which can also be done in polynomial time. Finding the maximally answerable clusters can also be done in polynomial time, as discussed earlier.

%\begin{wrapfigure}{r}{0.5\textwidth}
%\vspace{-0.2in}
%\hspace{-0.05cm}
%\IncMargin{1em}
%\begin{algorithm}[H]
\begin{algorithm}
    %\scriptsize
    %\setstretch{0.9}
    %\renewcommand{\thealgocf}{\arabic{algocf} (NeuralID)}
    %\renewcommand{\AlCapSty}[1]{\normalfont\scriptsize{\textbf{#1}}\unskip}
    \DontPrintSemicolon
    \SetKw{Not}{not}
    \SetKw{Or}{or}
    \SetKw{And}{and}
    \SetKwFunction{noncausal}{nonCausal}
    \SetKwFunction{merge}{merge}
    \SetKwFunction{remove}{remove}
    \SetKwFunction{admissible}{admissible}
    \SetKwFunction{answerable}{answerable}
    \SetKwFunction{id}{ID}
    \SetKwFunction{mincluster}{minCluster}
    \SetKwFunction{mergecycles}{mergeCycles}
    \SetKwFunction{maxcluster}{maxCluster}
    \SetKwFunction{coarserthan}{coarserThan}
    \SetKwFunction{cdag}{CDAG}
    \SetKwInOut{Input}{Input}
    \SetKwInOut{Output}{Output}

    \SetKwFunction{valid}{valid}
      \SetKwProg{Fn}{Function}{:}{}
    
    \Input{Variables $\*V_L$, causal diagram $\cG$, noncausal diagram $\overline{\cG}$, queries of interest $\bbQ$}
    \Output{Maximally coarse clustering $\mathbb{C}$ such that all criteria are satisfied, or \texttt{FAIL} if none exist.}
    \BlankLine
    \textls[-20]{
    \Fn{\valid{$\bbC$, $\cG$, $\bbQ$}}{
             \If{\Not $\admissible(\bbC, \cG)$}{
        \Return \texttt{False} \tcp*{C2 violated}
        }
        \For{$Q \in \bbQ$}{
            $\cG_{\bbC} \gets \cdag(\cG, \bbC)$ \tcp*{from Def.~\ref{def:cdag}}
            \If{\Not $\id(Q, \cG_{\bbC})$}{
                \Return \texttt{False} \tcp*{C4 violated}
            }
        }
        \Return \texttt{True}\;
    }
    %\tcc{start with all variables in separate clusters}
    $\mathbb{C}_{\min} \gets \mincluster(\*V_L, \overline{\cG})$ \tcp*{from Lem.\ \ref{lem:min-clusters}}
    $\mathbb{C}_{\min} \gets \mergecycles(\bbC_{\min}, \cG)$ \tcp*{from Lem.~\ref{lem:admissible-clusters}}
    $\mathbb{C}_{\max} \gets \maxcluster(\*V_L, \bbQ)$ \tcp*{from Lem.\ \ref{lem:max-clusters}}
    \If{\Not $\bbC_{\max}.\coarserthan(\bbC_{\min})$}{
        \Return \texttt{FAIL} \tcp*{either C1, C2, or C3 violated}
    }
    \If{\Not $\valid(\bbC, \cG, \bbQ)$}{
        \Return \texttt{FAIL} \tcp*{C4 cannot be satisfied}
    }
    $\bbC \gets \bbC_{\min}$\;
    \While{$\bbC$ keeps updating}{
        \For{$V \in \*V_L$}{
            \If{$V \not \in \*C$ for all $\*C \in \bbC_{\max}$}{
                $\bbC' \gets \bbC.\remove(V)$\;
                \If{$\valid(\bbC', \cG, \bbQ)$}{
                    $\bbC \gets \bbC'$\;
                }
            }
        }
        \For{$\*C_1, \*C_2 \in \bbC$}{
            \For{$\*C \in \bbC_{\max}$}{
                \If{$\*C_1, \*C_2 \subseteq \*C$}{
                    $\bbC' \gets \merge(\bbC, \*C_1, \*C_2)$\;
                    \If{$\valid(\bbC', \cG, \bbQ)$}{
                        $\bbC \gets \bbC'$\;
                    }
                }
            }
        }
    }
    \Return $\bbC$
    
    }
    \caption{Choosing intervariable clusters.}
    \label{alg:choose-intervariable-clusters}
\end{algorithm}
%\DecMargin{1em}
%\vspace{-0.3in}
%\end{wrapfigure}

Consider the following example for intuition on choosing the best set of clusters.
\begin{example}
    \label{ex:learn-intervariable-clusters}

    Consider a setting of annotated image data where
    \begin{equation*}
        \*V_L = \{A, T, G, H_1, H_2, O_1, O_2, O_3, O_4, P_1, P_2, P_3, P_4\},
    \end{equation*}
    and the causal diagram $\cG$ (and noncausal diagram $\overline{\cG}$) are given in Fig.~\ref{fig:ex-learn-c-g1}. The data is collected from people, and the variables are age ($A$), gender ($G$), testosterone level ($T$), mustache hairs ($H_1$, $H_2$), atoms of other mustache hairs ($O_1, O_2, O_3, O_4$), and pixels of an image of the person $(P_1, P_2, P_3, P_4)$. The queries of interest are $\bbQ = \{Q_1, Q_2, Q_3\}$, where
    \begin{align*}
        Q_1 &= P(\{p_1, p_2, p_3, p_4\}_{h_1, h_2, o_1, o_2, o_3, o_4}) \\
        Q_2 &= P(\{p_1, p_2, p_3, p_4\}_{h_1, h_2, o_1, o_2, o_3, o_4, g}) \\
        Q_3 &= P(\{p_1, p_2, p_3, p_4\}_{a}).
    \end{align*}

    Evidently, the variables depicted in Fig.~\ref{fig:ex-learn-c-g1} are too complex, and some intervariable clustering is needed. For starters, there are interactions between the atom variables ($O$) and the pixel variables ($P$) that are too low-level to be considered causally. They should be clustered together, as described in Lemma \ref{lem:min-clusters}. This provides the graph in Fig.~\ref{fig:ex-learn-c-g2}, in which all of the $O$ variables are clustered into another hair variable $H_3$, and the pixel variables $P$ are clustered into a single image variable $I$. Formally, this clustering is
    \begin{align*}
        \bbC_1 &= \{\{A\}, \{T\}, \{G\}, \{H_1\}, \\
        &\{H_2\}, H_3 = \{O_1, O_2, O_3, O_4\}, I = \{P_1, P_2, P_3, P_4\}\}.
    \end{align*}

    While this choice of $\bbC_1$ is valid in that it satisfies conditions C1-4, there exists a coarser clustering that satisfies the conditions, shown in Fig.~\ref{fig:ex-learn-c-g3}. This one is obtained by further clustering $H_1, H_2, H_3$ into a single ``mustache'' variable $M$. Formally, this clustering is
    \begin{align*}
        \bbC_2 &= \{\{A\}, \{T\}, \{G\}, M = \{H_1, H_2, O_1, O_2, O_3, O_4\}, \\
        &I = \{P_1, P_2, P_3, P_4\}\}.
    \end{align*}
    The corresponding queries of $\tau(\bbQ)$, under this choice of $\bbC$, would be
    \begin{align*}
        Q_1 &= P(\{p_1, p_2, p_3, p_4\}_{h_1, h_2, o_1, o_2, o_3, o_4}) = P(i_{m}) \\
        Q_2 &= P(\{p_1, p_2, p_3, p_4\}_{h_1, h_2, o_1, o_2, o_3, o_4, g}) = P(i_{m,g}) \\
        Q_3 &= P(\{p_1, p_2, p_3, p_4\}_{a}) = P(i_{a}).
    \end{align*}

    If for some reason we decide to try to cluster $\bbC_2$ further, such as by clustering $A, G, M$ into a single variable $X$ (Fig.~\ref{fig:ex-learn-c-g4}), we violate condition C3, the answerability of the queries. How could we distinguish between queries like $P(I_{M=m} = i)$ and $P(I_{A=a} = i)$ if $M$ and $A$ are grouped in the same cluster? For this reason, we cannot proceed any further than $\bbC_2$, and, in fact, $\bbC_2$ is actually the maximally answerable clusters from Lemma \ref{lem:max-clusters}.
    \hfill $\blacksquare$
\end{example}

\subsection{Discussion on the Abstract Invariance Condition}
\label{app:invariance-condition}

Recall the abstract invariance condition (AIC):
\aicdef*

The AIC intuitively ensures that there is no loss of information when clustering two values together in the intravariable clusters $\bbD$. Specifically, it states that clustering two values together is safe whenever these values are interchangeable with respect to the behavior of downstream functions. That is, the functions are \emph{invariant} to changes between values in the same intravariable cluster. This is illustrated in Fig.~\ref{fig:inv-cond}. Examples \ref{ex:cholesterol-inv-cond} and \ref{ex:inv-cond-satisfied} from Sec.~\ref{sec:abstract-ncm} illustrate the subtleties of the AIC both when it holds and when it does not hold.

If the AIC is violated, then two functionally different values were placed in the same cluster. Hence, assuming that the AIC holds is reasonable provided that the intravariable clusters were chosen reasonably. For example, in a study where a drug $X$ is reasonably expected to affect blood pressure levels $Y$, it would not make sense to cluster the values $X = 0$ (drug not taken) and $X = 1$ (drug taken) together since, by design, they have different downstream effects on $Y$ and cannot be treated similarly.

On the other hand, in an image recognition setting, perhaps it does not matter for the task if the image is scaled or rotated, so these invariances can be modeled in the intravariable clustering. For example, if one image is a scaled version of another but they are otherwise the same, the two images might be placed in the same intravariable clustering without violating the AIC. Sec.~\ref{sec:applications} and Appendix \ref{app:rep-learning} discuss this phenomenon in more detail.

If a data scientist finds the AIC too strict for her settings, one option is to revise the intravariable clusters, perhaps making them less coarse. Otherwise, the abstraction may be too strong and some important information is lost.

\begin{figure}
    \centering
    \begin{subfigure}[t]{0.49\linewidth}
        \centering
        \includegraphics[width=\linewidth]{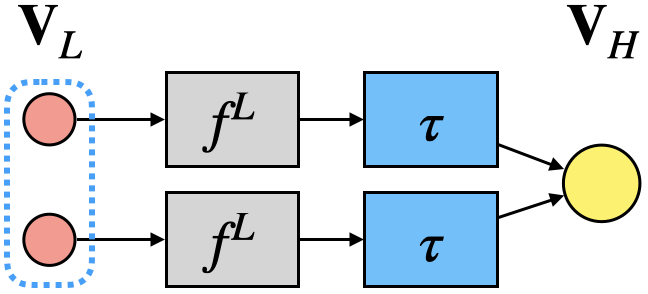}
        \caption{When the AIC holds, two values of $\*V_L$ that are in the same intravariable cluster are mapped to the same output in $\*V_H$ for any downstream function $f^L$ after applying $\tau$.}
    \end{subfigure}%
    \hfill
    \begin{subfigure}[t]{0.49\linewidth}
        \centering
        \includegraphics[width=\linewidth]{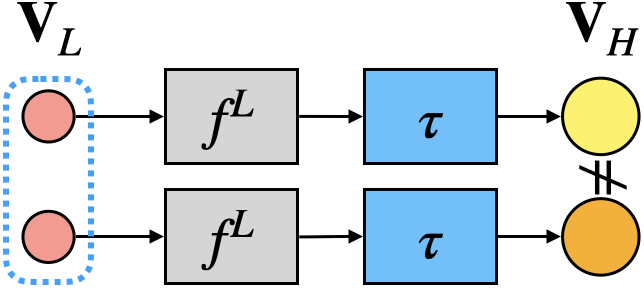}
        \caption{When the AIC does not hold, two values of $\*V_L$ that are in the same intravariable cluster may map to different outputs in $\*V_H$ for some downstream function $f^L$ after applying $\tau$.}
    \end{subfigure}
    \caption{Visualization of the AIC (Def.~\ref{def:invariance-condition}).}
    \label{fig:inv-cond}
\end{figure}

Still, there are alternatives if weaker assumptions are desired. The AIC ensures that an $\cL_3$-$\tau$ consistent model exists, as demonstrated by Alg.~\ref{alg:map-abstraction}, but it is not a necessary condition in applications that do not require full  $\cL_3$-$\tau$ consistency. For example, if one only works on the interventional level ($\cL_2$), one may not necessarily care if consistencies between $\cL_3$ quantities are lost.

In the strictest form, making as minimal assumptions as needed for the task, a query-specific AIC can be assumed. That is,
\begin{definition}
    \label{def:query-invariance-cond}
    Let $\bbQ$ be a set of queries of the form $P^{\cM_H}(\*y_{H,1[\*x_{H, 1}]}, \*y_{H,2[\*x_{H, 2}]}, \dots)$. For all $Q \in \bbQ$, we have
    \begin{align*}
        & \sum_{\forall i \*y_{L, i} \in \cD_{\*Y_{L, i}} : \tau(\*y_{L, i}) = \*y_{H, i}} \hspace{-1cm} P^{\cM_L}(\*y_{L,1[\*x_{L, 1}]}, \*y_{L,2[\*x_{L, 2}]}, \dots) \\
        &= \sum_{\forall i \*y_{L, i} \in \cD_{\*Y_{L, i}} : \tau(\*y_{L, i}) = \*y_{H, i}} \hspace{-1cm} P^{\cM_L}(\*y_{L,1[\*x'_{L, 1}]}, \*y_{L,2[\*x'_{L, 2}]}, \dots),
    \end{align*}
    for all $\*x_{L, i}, \*x'_{L, i}$ such that $\tau(\*x_{L, i}) = \tau(\*x'_{L, i}) = \*x_{H, i}$.
    \hfill $\blacksquare$
\end{definition}

In words, $\tau$ satisfies Def.~\ref{def:query-invariance-cond} if and only if there exists $\cM_H$ that is $Q$-$\tau$ consistent with $\cM_L$ for all $Q \in \bbQ$, since all such variations of $Q$ have the same probability. Still, it may be more useful to have a general criterion that is not query-dependent.

It turns out that the study of weaker forms of the AIC is not new and is sfirst discussed in \citet{10.5555/3020847.3020867}. Broadly, the paper studies a similar abstraction setting in which causal macro-variables are learned from observational data from the low-level micro-variables. Similar to the intravariable learning problem posed in Sec.~\ref{sec:applications}, the paper aims to learn intravariable clusters of a low-level high-dimensional variable $I$ (e.g., an image), which causes some other variable $T$, possibly with unobserved confounding. However, while the results of Sec.~\ref{sec:applications} focus on satisfying the AIC in Def.~\ref{def:invariance-condition}, \citet{10.5555/3020847.3020867} focuses on satisfying weaker conditions that are verifiable through data.

More specifically, consider the setting proposed in \citet{10.5555/3020847.3020867} illustrated by the causal diagram in Fig.~\ref{fig:chalupka-setting}. Note that Eq.~\ref{eq:aic} in Def.~\ref{def:invariance-condition} requires that
\begin{equation}
    \tau(f_T(i_1, \ui{T})) = \tau(f_T(i_2, \ui{T})),
\end{equation}
for any $i_1, i_2 \in \cD_{I}$ that are chosen to be clustered in the same intravariable cluster. This condition is a requirement on the SCM function $f_T$, which, as demonstrated by Lem.~\ref{prop:full-intra-cluster}, implies that one cannot cluster together any two values $i_1, i_2$ without further information about $f_T$, which may be difficult to acquire due to the unobserved nature of the SCM. Instead, the paper proposes two other properties. Notably, two values $i_1$ and $i_2$ are part of the same \emph{observational partition} \citep[Def.~1]{10.5555/3020847.3020867} if
\begin{equation}
    \label{eq:observational-partition}
    P(t \mid i_1) = P(t \mid i_2),
\end{equation}
and they are part of the same \emph{causal partition} \citep[Def.~3]{10.5555/3020847.3020867} if
\begin{equation}
    \label{eq:causal-partition}
    P(t_{i_1}) = P(t_{i_2}).
\end{equation}
Values of $\cD_I$ can easily be clustered according to Eq.~\ref{eq:observational-partition} using the observational data. However, Eq.~\ref{eq:causal-partition} may be a more desirable property if the goal of downstream tasks is to perform causal inferences from $I$ to $T$.

It turns out that a clustering learned according to Eq.~\ref{eq:observational-partition} typically also satisfies Eq.~\ref{eq:causal-partition} according to a major proven result of the paper, the \emph{Causal Coarsening Theorem} \citep[Thm.~5]{10.5555/3020847.3020867}, which states (in informal terms) that the maximally coarse clustering that satisfies Eq.~\ref{eq:causal-partition} is almost always a coarser clustering of the maximally coarse clustering that satisfies Eq.~\ref{eq:observational-partition}. This essentially implies that, under certain faithfulness conditions, one can achieve the property of Eq.~\ref{eq:causal-partition} simply as a consequence of achieving Eq.~\ref{eq:observational-partition}, which can be accomplished with observational data alone.

The concepts of the partitions described by Eqs.~\ref{eq:observational-partition} and \ref{eq:causal-partition} can be generalized to the general SCM/graph case. Suppose we are given the inter/intravariable clusters $\bbC$ and $\bbD$. Consider a variant of the AIC applied on conditional probabilties:
\begin{definition}[Conditional Abstract Invariance Condition]
    Let $P(\*V_L)$ be an observational distribution over $\*V_L$, and let $\tau: \cD_{\*V_L} \rightarrow \cD_{\*V_H}$ be a constructive abstraction function relative to $\bbC$ and $\bbD$. $P(\*V_L)$ is said to satisfy the conditional abstract invariance condition (conditional AIC, for short) w.r.t.~$\tau$ if, for all $\*X_L$ where $\*X_L$ is a union of clusters of $\bbC$, all $\*x_1, \*x_2 \in \cD_{\*X_L}$ such that $\tau(\*x_1) = \tau(\*x_2)$, and all $\*v_H \in \cD_{\*V_H}$, we have
    \begin{equation}
        \sum_{\*v_{L} \in \cD_{\*V_{L}} : \tau(\*v_{L}) = \*v_{H}} \hspace{-1cm} P(\*v_{L} \mid \*x_1) = \sum_{\*v_{L} \in \cD_{\*V_{L}} : \tau(\*v_{L}) = \*v_{H}} \hspace{-1cm} P(\*v_L \mid \*x_2).
    \end{equation}
    \hfill $\blacksquare$
\end{definition}

\begin{figure}
    \centering
    \begin{tikzpicture}[xscale=1, yscale=1.0]
        \node[draw, circle] (I) at (-1, 0) {$I$};
        \node[draw, circle] (T) at (1, 0) {$T$};
        
        \path [-{Latex[length=2mm]}] (I) edge (T);
        \path [{Latex[length=2mm]}-{Latex[length=2mm]}, dashed, bend left] (I) edge (T);
    \end{tikzpicture}
    \caption{Problem setting in \citet{10.5555/3020847.3020867}.}
    \label{fig:chalupka-setting}
\end{figure}
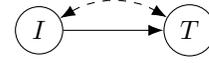

That is, two values can be clustered together as long as their conditional probabilities do not change in the space of $\*V_H$. Note that this condition is entirely focused on the observational distributions $P(\*V_L)$ and $P(\*V_H)$ and make no requirements over any distributions from the higher layers, $\cL_2$ and $\cL_3$.

Now consider an interventional variant:
\begin{definition}[Interventional Abstract Invariance Condition]
    Let $\cL_2(\cM_L)$ be a collection of interventional distributions over $\*V_L$, and let $\tau: \cD_{\*V_L} \rightarrow \cD_{\*V_H}$ be a constructive abstraction function relative to $\bbC$ and $\bbD$. We say that $\cL_2(\cM_L)$ satisfies the interventional abstract invariance condition (interventional AIC, for short) w.r.t.~$\tau$ if, for all $\*X_L$ where $\*X_L$ is a union of clusters of $\bbC$, all $\*x_1, \*x_2 \in \cD_{\*X_L}$ such that $\tau(\*x_1) = \tau(\*x_2)$, and all $\*v_H \in \cD_{\*V_H}$, we have
    \begin{equation}
        \sum_{\*v_{L} \in \cD_{\*V_{L}} : \tau(\*v_{L}) = \*v_{H}} \hspace{-1cm} P(\*v_{L[\*x_1]}) = \sum_{\*v_{L} \in \cD_{\*V_{L}} : \tau(\*v_{L}) = \*v_{H}} \hspace{-1cm} P(\*v_{L[\*x_2]}).
    \end{equation}
    \hfill $\blacksquare$
\end{definition}
This states that two values can be clustered together as long as the corresponding interventional distributions do not change in the space of $\*V_H$.

These two definitions may be useful when looking to find abstractions while working on lower layers of the PCH. When given the C-DAG $\cG_{\bbC}$, these can be refined to specific queries based on parent-child relationships in the graph. Although not shown formally here, it is likely that a more general variant of the Causal Coarsening Theorem could be proven for these two definitions. Nonetheless, using these definitions in place of the full AIC from Def.~\ref{def:invariance-condition} has its limitations; for example, it may no longer be possible to perform counterfactual inferences on the high-level model. We encourage further research on this topic, comparing the relationships between variations of the AIC.

\subsection{Representation Learning with Intravariable Clusters}
\label{app:rep-learning}

Recall the following proposition from Sec.~\ref{sec:applications}:
\fullintracluster*

When learning intravariable clusters, Prop.~\ref{prop:full-intra-cluster} states that without extra information, each value needs to be put in its own cluster. However, if extra information is available for use, this can be leveraged to improve the choice of $\bbD$. In practice, this is common, and is generally provided in the form of invariance assumptions. As a intuitive example, suppose that $X$ is an image of an animal, $Y$ is the corresponding label, and $f_Y(x, u_Y)$ describes the mechanism that humans use to classify $X$ (here, $u_Y$ is the exogenous noise). Suppose that $g(x, \phi)$ is a function that rotates $x$ by $\phi$ degrees. It is the case that the classification of animals is rotationally invariant, that is, $f_Y(g(x, \phi), u_Y) = f_Y(x, u_Y)$ for all $x$, $\phi$, $u_Y$. Then, in fact, $\bbD$ can be defined such that values of $x_1, x_2$ are clustered together if $g(x_1, \phi) = x_2$, as shown by the following result.

\begin{proposition}
    \label{prop:invariance-info}
    For each $\*C_i \in \bbC$, suppose that there exists function $g_{\*C_i}: \cD_{\*C_i} \times \cD_{\phi_{\*C_i}} \rightarrow \cD_{\*C_i}$ (with parameters $\phi_{\*C_i}$) such that
    \begin{equation}
    f^L_V(\*c_i, \*u_V) = f^L_V(g_{\*C_i}(\*c_i, \phi_{\*C_i}), \*u_V)
    \end{equation}
    for all $V$ that are children of $\*C_i$ and all $\phi_{\*C_i} \in \cD_{\phi_{\*C_i}}$. Then, $\cM_L$ satisfies the AIC w.r.t.~any constructive abstraction function with $\bbD$ such that $\*c_i^{(1)}, \*c_i^{(2)}$ being in the same intravariable cluster implies that $g_{\*C_i}(\*c_i^{(1)}, \phi_{\*C_i}) = \*c_i^{(2)}$ for some $\phi_{\*C_i}$.
    \hfill $\blacksquare$

    \begin{proof}
        Consider the proposed set of intravariable clusters $\bbD$. Let $\*v_1, \*v_2 \in \cD_{\*V_L}$ be two values such that $\tau(\*v_1) = \tau(\*v_2)$. The goal is to show that for all $\*u \in \cD_{\*U_L}$ and $\*C_k \in \bbC$, we have
        \begin{equation}
            \label{eq:inv-cond-inv-proof}
            \begin{split}
                & \tau \left( \left( f^L_V(\pai{V}^{(1)}, \*u_V): V \in \*C_k \right) \right) \\
                &= \tau \left( \left( f^L_V(\pai{V}^{(2)}, \*u_V): V \in \*C_k \right) \right),
            \end{split}
        \end{equation}
        where $\pai{V}^{(1)}$ and $\pai{V}^{(2)}$ are the values corresponding to $\*v_1$ and $\*v_2$ respectively. If $\tau(\*v_1) = \tau(\*v_2)$, that implies that for all $\*C_i \in \bbC$ (where $\*c_i^{(1)}$ and $\*c_i^{(2)}$ correspond to the values of $\*v_1$ and $\*v_2$ respectively), $\*c_i^{(1)}$ and $\*c_i^{(2)}$ must be in the same intravariable cluster in $\bbD_{\*C_i}$ (this includes the case where $\*c_i^{(1)} = \*c_i^{(2)}$). By construction of $\bbD$, this is only possible if $g_{\*C_i}(\*c_i^{(1)}, \phi_{\*C_i}) = \*c_i^{(2)}$ for some $\phi_{\*C_i}$.
        
        For $g_{V} = \{g_{\*C_i}: \*C_i \cap \Pai{V} \neq \emptyset\}$, denote $g_{V}(\pai{V}^{(1)}, \phi)$ as the collection of outputs of $g_{\*C_i}(\*c_i^{(1)}, \phi_{\*C_i})$ for the functions of $g_{\*C_i} \in g_{V}$, which, as established earlier, is equal to $\pai{V}^{(2)}$. Then we have
        \begin{align*}
            & \tau \left( \left( f^L_V(\pai{V}^{(1)}, \*u_V): V \in \*C_k \right) \right) \\
            &= \tau \left( \left( f^L_V(g_V(\pai{V}^{(1)}, \phi), \*u_V): V \in \*C_k \right) \right) \\
            &= \tau \left( \left( f^L_V(\pai{V}^{(2)}, \*u_V): V \in \*C_k \right) \right),
        \end{align*}
        completing the proof.
    \end{proof}
\end{proposition}

For intuition, consider the following simple example.

\begin{example}
    Consider a situation where $Y$ is some binary variable that is a (noisy) bitwise AND of two other binary variables $X_1$ and $X_2$. For example, perhaps $Y$ denotes whether a law is enacted, and $X_1$ and $X_2$ denotes the votes of the two branches of government, like in Ex.~\ref{ex:allowed-interventions}. Formally, let $\cM_L$ be defined as follows:
    \begin{align*}
        \*U_L &= \{U_{X_1}, U_{X_2}, U_Y\}, \text{ all binary} \\
        \*V_L &= \{X_1, X_2, Y\}, \text{ all binary} \\
        \cF_L &=
        \begin{cases}
            X_1 \gets f^L_{X_1}(u_{X_1}) &= u_{X_1} \\
            X_2 \gets f^L_{X_2}(u_{X_2}) &= u_{X_2} \\
            Y \gets f^L_Y(x_1, x_2, u_Y) &= (x_1 \wedge x_2) \oplus u_Y
        \end{cases} \\
        P(\*U_L) &=
        \begin{cases}
            P(U_{X_1} = 1) &= P(U_{X_2} = 1) = 0.5 \\
            P(U_Y = 1) &= 0.1
        \end{cases}
    \end{align*}
    Now suppose we want to create an abstraction of $\cM_L$ using the intervariable clusters $\bbC = \{\*C_1 = \{X_1, X_2\}, \*C_2 = \{Y\}\}$. The domain $\cD_{\*C_1}$ has four values as $X_1$ and $X_2$ can both be either 0 or 1. However, as emphasized by Prop.~\ref{prop:full-intra-cluster}, the only intravariable clustering we can choose without additional information is the one where each value is in its own cluster. In other words, $\bbD_{\*C_1} = \{\{(X_1 = 0, X_2 = 0)\}, \{(X_1 = 0, X_2 = 1)\}, \{(X_1 = 1, X_2 = 0)\}, \{(X_1 = 1, X_2 = 1)\}\}$.

    However, suppose we are given additional information that $f_Y$ is \emph{permutation invariant} to its endogenous inputs, i.e.,
    \begin{equation}
    f_Y(x_1, x_2, u_Y) = f_Y(x_2, x_1, u_Y).
    \end{equation}
    In the notation of Prop.~\ref{prop:invariance-info}, we can define $\phi_{\*C_1}$ to be a binary variable such that 0 means original order and 1 means reversed. Then we can define $g(x_1, x_2, \phi_{\*C_1}) = (x_1, x_2)$ if $\phi_{\*C_1} = 0$ or $(x_2, x_1)$ if $\phi_{\*C_1} = 1$. Then, this implies that $f_Y(x_1, x_2, u_Y) = f_Y(g(x_1, x_2, \phi_{\*C_1}), u_Y)$ for any choice of $\phi_{\*C_1}$.

    By Prop.~\ref{prop:invariance-info}, this implies that $(X_1 = 0, X_2 = 1)$ and $(X_1 = 1, X_2 = 0)$ can be placed in the same intravariable cluster without violating the AIC. Indeed, the function $f^L_Y(x_1, x_2, u_Y) = (x_1 \wedge x_2) \oplus u_Y$ does not change when $x_1$ and $x_2$ are swapped, so we have
    \begin{equation*}
        \begin{split}
            & \tau \left( \left( f^L_Y(X_1 = 0, X_2 = 1, \*u_V): V \in \*C_k \right) \right) \\
            &= \tau \left( \left( f^L_Y(X_1 = 1, X_2 = 0, \*u_V): V \in \*C_k \right) \right),
        \end{split}
    \end{equation*}
    confirming that the AIC still holds when $(X_1 = 0, X_2 = 1)$ and $(X_1 = 1, X_2 = 0)$ can be placed in the same intravariable cluster.
    \hfill $\blacksquare$
\end{example}

In practice, this invariance information can be incorporated in the process of learning the intravariable clusters while training $\widehat{\tau}$ in an RNCM, as described in Def.~\ref{def:rncm}. To take into account Prop.~\ref{prop:invariance-info}, the training objective could include a term to enforce the invariance specified by $g_{\*C_i}$ (e.g.~through a penalty on $\tau_{\*C_i}(\*c_i, \bm{\theta}_{\*C_i}) - \tau_{\*C_i}(g_{\*C_i}(\*c_i, \phi_{\*C_i}), \bm{\theta}_{\*C_i})$).

In fact, it turns out that the idea of intravariable clusters works in tandem with techniques in deep representation learning when it comes to incorporating invariances in the data. In the case of image data for example, many works in computer vision have already leveraged general patterns found in images (e.g.~rotations, crops, flips, etc.~do not affect classification) to achieve faster training with less data \citep{Shorten2019ASO}. Permutation invariance concepts \citep{10.5555/3294996.3295098, murphy2018janossy} have been used as pooling functions for convolutional neural networks \citep{10.5555/303568.303704}. Existing frameworks for representation learning such as through contrastive methods \citep{DBLP:conf/icml/ChenK0H20} can be used for learning $\tau$.

%\clearpage
\section{Additional Examples} \label{app:examples}

This section contains examples to improve the clarity of concepts in the paper.

\subsection{Sec.~\ref{sec:abstract-ncm} Examples}
We provide more examples of the concepts in Sec.~\ref{sec:abstract-ncm}, namely, abstractions constructed through inter/intravariable clusters. In addition to Example \ref{ex:bmi} within the section, we provide a more involved example below.

\begin{example}
    \label{ex:clusters}
    Suppose an economist is studying the effects of implementing a new type of government policy on recession prevention. The economist records data on several variables of interest: whether the policy is implemented ($X$); whether the policy is lobbied ($W$); economic spending in terms of consumption ($C$), investment ($I$), government spending ($G$), imports ($M$), and exports ($E$); and whether or not there is a recession ($Y$). Out of these variables, $Z$, $X$, and $Y$ are binary, and $C$, $I$, $G$, $M$, and $E$ are numerical values representing how much the spending has .changed relative to the previous year (in billions of dollars).
    
    Suppose the true SCM $\cM^* = \cM_L$ is defined as follows: 
    \begin{align*}
        \*U_L &= \{U_W, U_X, U_{XY}, U_C, U_I, U_G, U_M, U_E, U_{IG}, \\
        & U_{CM}, U_{EM}\} \\
        \*V_L &= \{W, X, C, I, G, M, E, Y\} \\
        \cF_L &= \{\\
        & W \gets f_W(u_W) = u_W \\
        & X \gets f_X(w, u_X, u_{XY}) = w \oplus u_X \oplus u_{XY} \\
        & C \gets f_C(x, g, u_C) = 5x - 0.2g + u_C + u_{CM} \\
        & I \gets f_I(x, u_I) = -5x + u_I + u_{IG} \\
        & G \gets f_G(x, u_G) = 10x + u_G + u_{IG} \\
        & E \gets f_E(u_E) = u_E + u_{EM} \\
        & M \gets f_M(c, u_M) = 0.2c + u_M + u_{CM} - u_{EM} \\
        & Y \gets f_Y(c, i, g, e, m, u_{XY}) = \\
        & \quad \mathbf{1}\{c + i + g + e - m \leq 0\} \oplus u_{XY}\\
        P(\*U_L) &=
        \begin{cases}
            U_W, U_X &\sim \bern(0.5) \\
            U_{XY} &\sim \bern(0.1) \\
            U_C, U_I, U_G, U_M, U_E &\sim N(0, 10) \\
            U_{IG}, U_{CM}, U_{EM} &\sim N(0, 2)
        \end{cases}
    \end{align*}
    
    To summarize, the policy has some impact on the consumption, investment, and government spending. Whether or not there is a recession depends on imports and exports in addition to all of these factors. Indeed, despite the fact that this is a toy example, the SCM $\cM^*$ is already quite complex to describe with this many variables and functions, and perhaps this level of detail is not needed to achieve the inference we desire. We will see how utilizing abstractions can help with this.
    
    We first note that perhaps $\*V_L$ contains too many micro-level variables that can be summarized with a smaller set of abstract higher-level variables. Further, perhaps using real-valued variables is overly complex, and the same phenomena can be described without loss of generality using a lower-dimensional space. For entertaining these considerations, we utilize the idea of intervariable and intravariable clusters.

	Suppose we are only interested in the causal effect of the policy $X$ on recession $Y$. We could study the same phenomenon under a simpler set of variables. To do so, we can cluster variables of $\*V_L$ to form a new set of macro-level variables. First, note that although we may have data on $C$, $I$, $G$, $E$, and $M$, it may make more sense to simply cluster them together and consider them as one variable (e.g. GDP). Further, perhaps we may decide that $W$ is not relevant to the analysis and exclude it from the study. We leave these clustering decisions at the discretion of the data scientist (e.g. the economist in this example).
	
	By the definition of intervariable clusters from Def.~\ref{def:var-clusterings}, we can choose clusters $\bbC = \{\*C_1 = \{X\}, \*C_2 = \{C, I, G, E, M\}, \*C_3 = \{Y\}\}$. We leave $X$ and $Y$ in their own clusters while grouping all of $C$, $I$, $G$, $E$, $M$ into one cluster $\*C_2$. $W$ is excluded from all of the clusters and is effectively projected out of the system. We can then treat $\*C_1$, $\*C_2$, and $\*C_3$ as our new variables. Let us relabel them $X$, $Z$, and $Y$ respectively, where $Z$. We can define our higher level variables as $\*V_H = \{X, Z, Y\}$.
	
	Further, it is not immediately clear what the domain of $Z$ is. Certainly, it could be left as a tuple $(C, I, G, E, M) \in \bbR^5$. However, this would be an overly complex representation of $Z$, and we do not need to retain all of the joint information of $C, I, G, E, M$. In other words, $\cD_{Z}$ does not have to be equal to $\cD_{\*C_2}$ and can be represented more compactly. Instead, we may choose to define $Z$ as the GDP, or $Z = C + I + G - E - M$. This can be described through the use of intravariable clusters.
	
	By the definition of intravariable clusters from Def.~\ref{def:var-clusterings},  we must choose $\bbD = \{\bbD_{\*C_1}, \bbD_{\*C_2}, \bbD_{\*C_3}\}$, where $\bbD_{\*C_i}$ is a partitioning of the domain $\cD_{\*C_i}$. Since $X$ and $Y$ are already binary variables, we will not be able to group their values together, so we simply define $\bbD_{\*C_1}$ and $\bbD_{\*C_2}$ as $\{\{0\}, \{1\}\}$. For $\*C_2$ however, we must cluster values of $(C, I, G, E, M)$ such that $(c_1, i_1, g_1, e_1, m_1), (c_2, i_2, g_2, e_2, m_2) \in \cD_{\*C_2}$ are in the same cluster if and only if $c_1 + i_1 + g_1 + e_1 - m_1 = c_2 + i_2 + g_2 + e_2 - m_2$. In other words, we can, for example, define $\bbD_{\*C_2} = \{\cD_{\*C_2}^j : (c, i, g, e, m) \in \cD_{\*C_2}, c + i + g + e - m = j\}$.
	
	Now, with the intravariable clusters $\bbD$ defined, we can choose the domains of $\*V_H$ as simply their corresponding clusters in $\bbD$. That is, each value of $Z$ corresponds to some $\cD_{\*C_2}^j$. In fact, we can simply set $Z = j$, where $j = C + I + G + E - M$, which intuitively corresponds to the idea that $Z$ represents the annual change in GDP. Note that the domain of $Z$, $\cD_{Z}$ becomes smaller than $\cD_{\*C_2}$ in some sense. First, it is lower dimensional ($\bbR$ instead of $\bbR^5$), and second, there are clearly values of $(C, I, G, E, M)$ that are mapped to the same value of $Z$. For example, $(C = 1, I = 2, G = 3, E = 4, M = 2$ maps to the same value as $(C = 3, I = 2, G = 1, E = 5, M = 3)$ because $1 + 2 + 3 + 4 - 2 = 3 + 2 + 1 + 5 - 3 = 8$.
	
    However, perhaps this level of dimensionality reduction is insufficient. In terms of cardinality, $\cD_Z$ is the same size as $\cD_{\*C_2}$ as the cardinality of $\bbR$ and $\bbR^5$ are the same. Perhaps the domain of $Z$ could be compressed further. Indeed, we could turn $Z$ into a binary variable by defining $\bbD_{\*C_2} = \{\cD^0_{\*C_2}, \cD^1_{\*C_2}\}$, where $(c, i, g, e, m) \in \cD^0_{\*C_2}$ if $c + i + g + e - m > 0$ or it is in $\cD^1_{\*C_2}$ otherwise. In other words, we drop all information about the annual change in GDP except whether it is positive or negative. We will use this clustering for the rest of the examples, and we will see in later examples why this clustering is allowed and makes sense. Sec.~\ref{sec:applications} expands on the general discussion of which intravariable cluster choices are allowed.
    
    With the idea of inter/intravariable clusters, the concept of constructive abstraction functions can be established as in Def.~\ref{def:tau}. From Example \ref{ex:clusters}, the function $\tau$ constructed from $\bbC$ and $\bbD$ is clear. For example, consider the value $\*v_L = (W=0, X=1, C=2.6, I=-1.2, G=10.2, E=0.4, M=1.2, Y=0)$. We can compute $\*v_H = \tau(\*v_L)$ and find that $\*v_H = (X = 1, Z = 0, Y = 0)$. In this case, $\tau = (\tau_{\*C_1}, \tau_{\*C_2}, \tau_{\*C_3})$, where $\tau_{\*C_1}$ and $\tau_{\*C_3}$ are the identity function on $X$ and $Y$ respectively, and $\tau_{\*C_{2}}(C=2.6,I=-1.2,G=10.2,E=0.4,M=1.2) = \mathbf{1}\{2.6 - 1.2 + 10.2 + 0.4 - 1.2 \leq 0\} = 0$. That is, for $\*C_2$, the intravariable clusters are defined such that values of $(c, i, g, e, m)$ are divided into two categories depending on whether $c + i + g + e - m \leq 0$, which we can arbitrarily choose as the binary values $0$ and $1$. This provides a mapping from every value of $\cD_{\*V_L}$ to some value of $\cD_{\*V_H}$.
\hfill $\blacksquare$
\end{example}

Note that $\tau$ alone does not define an abstraction. While $\tau$ provides a well-defined mapping from $\*V_L$ to $\*V_H$, not every SCM defined over $\*V_H$ can be considered an abstraction of $\cM^*$. Consider the following example.

\begin{example}
    \label{ex:bad-tau}
    Consider the SCM $\cM_H$ defined as follows:
    \[
    \cM_H = 
    \begin{cases}
    	\*U_H &= \{U_X, U_Z, U_Y\} \\
    	\*V_H &= \{X, Z, Y\} \\
    	\cF_H &=
    	\begin{cases}
    		X \gets f_X(u_X) &= u_X \\
                Z \gets f_Z(u_Z) &= u_Z \\
                Y \gets f_Y(u_Y) &= u_Y
    	\end{cases} \\
    	P(\*U_H) &= U_X, U_Z, U_Y \sim \bern(0.5)
    \end{cases}
    \]
    One could argue that $\cM_H$ is defined over $\*V_H$, constructed via $\tau$ from Example \ref{ex:clusters}. However, it is trivial to see that $\cM_H$ retains none of the meaning of the variables of $\*V_H$ intended by $\tau$ and is clearly oblivious of $\cM_L$. After all, there is not even any causal relationship between the variables defined in $\cM_H$. Intuitively, $\cM_H$ is clearly not an abstraction of $\cM^*$, and this is reinforced by the fact that $\cM_H$ does not match common definitions of abstractions, such as Def.~\ref{def:tau-abs}.

    However, assuming that we do not have access to $\cM_H$ or $\cM^*$, and can only observe them through their distributions of the PCH, how could we tell that $\cM_H$ is not an abstraction of $\cM^*$? Certainly, there are inconsistencies in the distribution too. For example, note that
    \begin{align*}
        & P^{\cM^*}(C + I + G + E - M > 0 \mid X = 1) \\
        &= P(5 - 0.2g + u_C + u_{CM} - 5 + u_I + u_{IG} + 10 + u_G \\
        &+ u_{IG} + u_E + u_{EM} - 0.2c - u_M - u_{CM} - u_{EM} > 0) \\
        &= P(5 - 0.2(10 + u_G + u_{IG}) + u_C + u_{CM} - 5 + u_I \\
        &+ u_{IG} + 10 + u_G + u_{IG} + u_E + u_{EM} \\
        &- 0.2(5 - 0.2(10 + u_G + u_{IG}) + u_C + u_{CM}) \\
        &- u_M - u_{CM} - u_{EM} > 0) \\
        &= P(7.4 + 0.84u_G + 1.8u_{IG} + 0.8u_C -0.2u_{CM} \\ 
        &+ 1.04u_I + u_E - u_M > 0) \\
        &= P(N(0, 7.056) + N(0, 6.48) + N(0, 6.4) + N(0, 0.08) \\
        &+ N(0, 10.816) + N(0, 10) + N(0, 10) > -7.4) \\
        &= P(N(0, 50.832) > -7.4) \approx 0.85.
    \end{align*}
    On the other hand, we see that $P^{\cM_H}(Z = 0 \mid X = 1) = P^{\cM_H}(Z=0) = P(U_Z = 0) = 0.5$. We would expect that if $\cM_H$ were an abstraction of $\cM^*$, then $P^{\cM_H}(Z = 0 \mid X = 1)$ should match $P^{\cM^*}(C + I + G + E - M > 0 \mid X=1)$, which is evidently not the case.
    
    The intuition behind this connection is that $\tau$ directly maps cases of $C + I + G + E - M > 0$ to cases of $Z = 0$, so their corresponding probabilities should be the same. If $Q = P^{\cM^*}(C + I + G + E - M > 0 \mid X=1)$, we would say that $\cM_H$ is not $Q$-$\tau$ consistent with $\cM^*$, via Def.~\ref{def:q-tau-consistency}.
    \hfill $\blacksquare$
\end{example}

Not all functions mapping between two spaces of variables are constructive abstraction functions. Some abstraction functions, while not being constructive, may still have qualities that are expected from abstractions. However, the lack of structure in these functions make certain features of the abstraction models difficult to define, such as the concept of $Q$-$\tau$ consistency. The following example describes a few such cases.

\begin{example}
    \label{ex:non-constructive-tau}
    To start, we note that when fixing the intervariable and intravariable clusters $\bbC$ and $\bbD$, the corresponding constructive abstraction function $\tau$ is deterministic, and the corresponding $\*V_H$ is fixed, up to a bijective mapping of its domain. For example, suppose $\bbC = \{\*A, \*B, \*C\}$, and $\bbD_{\*A} = \{\*a_1, \*a_2\}, \bbD_{\*B} = \{\*b_1, \*b_2\}, \bbD_{\*C} = \{\*c_1, \*c_2\}$. However, let $\*V_H = \{X, Y, Z, W\}$, where $\cD_{X} = \cD_Y = \cD_Z = \{0, 1\}$ and $\cD_{W} = \{0, 1, 2\}$. While a function could be constructed from $\bbD$ and $\bbC$, it could not be a valid mapping to $\*V_H$. First of all, there are four variables in $\*V_H$ but only three clusters in $\bbC$, so the mapping could not be one-to-one.
    
    Even if $X$ were removed, all of the intervariable clusters have values clustered into two sets, but $W$ is a ternary variable, so a one-to-one mapping is not possible with any of the intervariable clusters. However, if $W$ were removed from $\*V_H$, then $\tau$ could be a valid constructive abstraction mapping between $\*V_L$ and $\*V_H$, since the other three variables are binary. The number of variables of $\*V_H$ and their domain sizes match the number of clusters of $\bbC$ and the number within each cluster of $\bbD$, respectively. In fact, at least without considering the model distributions, any cluster of $\bbC$ could map to any of the variables of $\*V_H \setminus \{W\}$ as they are isomorphic.

    In other words, $\tau$ fails to be a constructive tau abstraction for any choice of $\*V_H$ that is incompatible with $\bbC$ and $\bbD$. Intuitively, this means that when given $\bbC$ and $\bbD$, the function $\tau$ and high level variables $\*V_H$ are already fixed, which is how the rest of this paper is framed. However, there can exist cases with some function $\tau$ mapping from $\*V_L$ to $\*V_H$ such that $\tau$ cannot be considered a constructive abstraction function for any choice of $\bbC$ and $\bbD$. Some cases may even appear to be valid abstractions intuitively.

    Consider an example of a company board setting. Suppose in a company, the board consists of the CEO, Alice ($A$), and two vice presidents, Bob ($B$) and Charlie ($C$). When voting on company policies, each board member can choose to vote for ($+1$), vote against ($-1$), or abstain ($0$). In other words, $A$, $B$, and $C$ are ternary variables with domain $\{-1, 0, +1\}$. While data can be collected on each of the members' voting behaviors ($V_L = \{A, B, C\}$), it may be more sensible to aggregate the votes into a more useful quantity. Suppose two high level variables are computed ($\*V_H = \{X, Z\}$) through some abstraction function $\tau = (\tau_X, \tau_Z)$, defined as follows.
    \begin{align*}
        \tau_X(a, b, c) &=
        \begin{cases}
            +1 & a + b + c > 0 \\
            0 & a + b + c = 0 \\
            -1 & a + b + c < 0
        \end{cases} \\
        \tau_Z(a, b, c) &=
        \begin{cases}
            a & a \neq 0 \\
            +1 & a = 0, b + c > 0 \\
            0 & a = 0, b + c = 0 \\
            -1 & a = 0, b + c < 0
        \end{cases}
    \end{align*}
    In words, $X$ is the aggregate vote that is simply the majority vote of all three members. On the other hand $Z$ is an aggregate vote that prioritizes Alice's vote, as she is CEO. If she abstains, then it is an aggregate of the votes of the vice presidents. For example, $\tau(A = +1, B = -1, C = -1) = (X = -1, Z = +1)$.

    Perhaps in company matters, it is more useful to use the variables $X$ and $Z$ over the individual votes of the board members. However, $\tau$ cannot be considered a constructive abstraction function for any choice of clusters $\bbC$ and $\bbD$. This is because both $X$ and $Z$ change values depending on all of the values of $A$, $B$, and $C$, so the variables of $\*V_L$ cannot be cleanly separated into two different clusters.

    Even so, this choice of $\tau$ seems like it could result in a valid abstraction. The difficulty lies with the analysis of causal quantities after $\tau$ is fixed. Without the notion of clusters, the definition of $Q$-$\tau$ consistency fails to work, so it is no longer clear what causal quantities correspond to what. For example, what would an intervention of $A = +1$ imply on the $\*V_H$ level? Or, what would an intervention on $Z = +1$ imply on the $\*V_L$ level? The answer might change depending on the setting, or there may not even be an answer that makes sense at all. This is why prior works like \citet{rubenstein:etal17-causalsem} define the mapping between interventions separately from $\tau$, which maps the variables.

    In this particular case, provided that queries do not require separation of $X$ and $Z$ (e.g.~queries like $P(X_{Z = z} = x)$ are not needed), then it may be sensible to cluster $A$, $B$, and $C$ together into one intervariable cluster and then simply have one variable in $\*V_H$ used for downstream tasks. This allows for the theory in this paper to be applied, significantly reducing the complexity of defining the abstractions.
    \hfill $\blacksquare$
\end{example}

See the following for a negative example of an abstraction of the drug example in Ex.~\ref{ex:drug-tau}.

\begin{example}
    \label{ex:drug-bad-abstraction}
    Continuing Example \ref{ex:drug-tau}, consider the following SCM $\cM_H$.
    \begin{align*}
        \*U_H &= \{U_X, U_Y\} \\
        \*V_H &= \{X, Y\} \\
        \cF_H &=
        \begin{cases}
            X \gets f^H_X(u_X) &= u_X \\
            Y \gets f^H_Y(u_Y) &= u_Y
        \end{cases} \\
        P(\*U_H) &=
        \begin{cases}
            P(U_X = 1) &= 0.5 \\
            P(U_Y = 1) &= 0.2
        \end{cases}
    \end{align*}

    Indeed, this choice of $\cM_H$ is defined over $X$ and $Y$. However, it does not seem like there is any connection between $\cM_H$ and $\cM_L$ from Example \ref{ex:drug-tau}, even if $\*V_H = \tau(\*V_L)$. To verify this, we can compare the distributions induced by the two models. Note that while $P^{\cM_L}(Y = 1 \mid A = 1, B = 1) \approx 0.853$, we see that $P^{\cM_H}(Y = 1 \mid X = 1) = P^{\cM_H}(Y = 1) = P^{\cM_H}(U_Y = 1) = 0.2$. It seems that these two quantities should be related, as $\tau$ maps $(A = 1, B = 1)$ to $X = 1$, yet they are clearly not equal in the two models. This is similar for the causal effect $P^{\cM_L}(Y_{A = 1, B = 1} = 1)$. Computing $P^{\cM_H}(Y_{X = 1} = 1) = P^{\cM_H}(Y = 1) = 0.2$ actually yields the same result, which is clearly incorrect. In fact, it even seems that the causal relations are incorrect, as $f^H_Y$ does not use $X$ as an input.
    
    \hfill $\blacksquare$
\end{example}

Example \ref{ex:drug-bad-abstraction} shows an example of $\cM_H$ that is a poor abstraction of $\cM_L$ despite the fact that it is defined over the space of $\*V_H$ mapped by $\tau$. From the example, intuition tells us that a proper abstraction of $\cM_L$ should match $\cM_L$ in certain quantities, including observational, interventional, and counterfactual quantities. Specifically, there are quantities induced by $\cM_L$ that appear to have matching counterparts in $\cM_H$ based on $\tau$. This notion is made concrete through the concept of $Q$-$\tau$ consistency (Def.~\ref{def:q-tau-consistency}).

\end{document}